\documentclass[12pt]{article}

\usepackage[T1]{fontenc}
\usepackage{microtype}
\usepackage{mathtools}
\usepackage{booktabs}
\usepackage{tabularx}
\usepackage{tikz}
\usepackage{algorithm}
\usepackage{algorithmic}
\usepackage{pifont}
\usepackage[letterpaper,margin=1in]{geometry}
\usepackage{hyperref}
\usepackage[round]{natbib}
\usepackage{amssymb}
\usepackage{amsthm}
\usepackage{setspace}
\usepackage{authblk}

\numberwithin{equation}{section}

\newtheorem{theorem}{Theorem}
\newtheorem{proposition}{Proposition}
\newtheorem{lemma}{Lemma}

\newcommand*{\email}[1]{\href{mailto:#1}{\nolinkurl{#1}}}

\hypersetup{colorlinks=true,citecolor=blue,urlcolor=blue}

\newcommand{\cmark}{\ding{51}}
\newcommand{\xmark}{\ding{55}}

\newcolumntype{Y}{>{\centering\arraybackslash}X}

\newcommand\keywords[1]{%
\begin{NoHyper}
\renewcommand\thefootnote{}\footnote{\emph{Keywords:} #1}%
\addtocounter{footnote}{-1}%
\end{NoHyper}
}

\def\ndot{n_{\scriptscriptstyle\bullet}}

\title{Scalable Subset Selection in Linear Mixed Models}
\author{Ryan Thompson\thanks{Corresponding author. Email: \email{ryan.thompson-1@uts.edu.au}}}
\author{Matt P. Wand}
\author{Joanna J. J. Wang}
\affil{School of Mathematical and Physical Sciences, University of Technology Sydney}

\begin{document}

\maketitle
\vspace{-0.2in}

\begin{abstract}
Linear mixed models (LMMs), which incorporate fixed and random effects, are key tools for analyzing heterogeneous data, such as in personalized medicine. Nowadays, this type of data is increasingly wide, sometimes containing thousands of candidate predictors, necessitating sparsity for prediction and interpretation. However, existing sparse learning methods for LMMs do not scale well beyond tens or hundreds of predictors, leaving a large gap compared with sparse methods for linear models, which ignore random effects. This paper closes the gap with a new $\ell_0$ regularized method for LMM subset selection that can run on datasets containing thousands of predictors in seconds to minutes. On the computational front, we develop a coordinate descent algorithm as our main workhorse and provide a guarantee of its convergence. We also develop a local search algorithm to help traverse the nonconvex optimization surface. Both algorithms readily extend to subset selection in generalized LMMs via a penalized quasi-likelihood approximation. On the statistical front, we provide a finite-sample bound on the Kullback--Leibler divergence of the new method. We then demonstrate its excellent performance in experiments involving synthetic and real datasets.
\end{abstract}


\section{Introduction}
\label{sec:introduction}

Linear mixed models (LMMs) have a long and rich history in statistics, dating back at least to the 1950s \citep{Henderson1959}, and have more recently drawn attention in machine learning. The appeal of LMMs lies in their capacity to flexibly model heterogeneous effects in data containing multiple observations on clusters of related units. Examples include clinical studies with repeated measurements on the same patients at different times to capture patient-specific effects, or educational evaluations with repeated measurements on the same schools across different classes to isolate school-specific effects. The prevalence of data exhibiting this type of structure has motivated an enormous body of research into LMMs. We refer to \citet{Jiang2021} and \citet{West2022} for recent and accessible treatments.

Given $n_i$ observations on $\mathbf{y}_i\in\mathbb{R}^{n_i}$, $\mathbf{X}_i\in\mathbb{R}^{n_i\times p}$, and $\mathbf{Z}_i\in\mathbb{R}^{n_i\times d}$, representing a response vector and predictor matrices for cluster $i=1,\dots,m$, an LMM is typically written as
\begin{equation*}
\mathbf{y}_i=\mathbf{X}_i\boldsymbol{\beta}+\mathbf{Z}_i\mathbf{u}_i+\boldsymbol{\varepsilon}_i,\quad i=1,\dots,m,
\end{equation*}
where $\boldsymbol{\varepsilon}_i\in\mathbb{R}^{n_i}$ is stochastic noise.\footnote{To simplify presentation, we exclude intercept terms.} The coefficients $\boldsymbol{\beta}\in\mathbb{R}^p$ are known as fixed effects, and the coefficients $\mathbf{u}_i\in\mathbb{R}^d$ as random effects. The fixed effects are so-called because $\boldsymbol{\beta}$ is constant across all clusters $i$, whereas the random effects $\mathbf{u}_i$ vary across clusters according to a specified distribution, typically a Gaussian distribution $\mathbf{u}_i\sim\mathrm{N}(\mathbf{0},\boldsymbol{\Gamma})$. In modern data-analytic regimes where the predictor dimensionalities $p$ and $d$ are large, it is common to operate on the principle that $\boldsymbol{\beta}$ and $\mathbf{u}_i$ are sparse. The goal is then to learn if a predictor in $\mathbf{X}_i$ or $\mathbf{Z}_i$ has a nonzero effect. More generally, we wish to learn if a predictor has a zero, fixed, or random effect. In that case, which is our paper's focus, one takes $\mathbf{Z}_i=\mathbf{X}_i$, giving
\begin{equation}
\label{eq:lmm}
\mathbf{y}_i=\mathbf{X}_i(\boldsymbol{\beta}+\mathbf{u}_i)+\boldsymbol{\varepsilon}_i,\quad i=1,\dots,m.
\end{equation}
Consequently, unless the number of nonzero random effects is small, it is necessary to structure the covariance matrix to prevent the number of estimable parameters from growing quadratically. A standard and widely adopted structure, particularly suitable for modern, predictor-rich settings, is the diagonal matrix $\boldsymbol{\Gamma}=\operatorname{diag}(\boldsymbol{\gamma})$, where $\boldsymbol{\gamma}\in\mathbb{R}_+^p$ is a vector of random effect variances. Hence, the random effect $u_{ik}=0$ whenever its variance $\gamma_k=0$.

Methods for fitting sparse LMMs broadly fit into two categories: (i) those that select either fixed effects or random effects but not both \citep{Schelldorfer2011,Groll2014,Pan2014,Schelldorfer2014,Kowal2023}, and (ii) those that select both types of effect \citep{Bondell2010,Ibrahim2011,Fan2012,Peng2012,Lin2013,Hui2017a,Hui2017b,Li2018,Heiling2023,Sholokhov2024}. While methods in the first category sometimes scale to thousands of predictors (specifically those that select fixed effects), they assume either the set of relevant fixed effects or set of relevant random effects is known a priori---an assumption rarely satisfied in contemporary predictor-rich settings. The second category of methods addresses this limitation by selecting both types of effects simultaneously, but typically do not scale nearly as well and often do not enforce the hierarchical selection principle, whereby a random effect can only be nonzero if the corresponding fixed effect is also nonzero \citep[an exception is][]{Hui2017a}. Moreover, apart from \citet{Kowal2023} and \citet{Sholokhov2024}, these approaches rely on continuous shrinkage regularizers such as the $\ell_1$ (lasso) penalty \citep{Tibshirani1996}, smoothly clipped absolute deviation penalty \citep{Fan2001}, or minimax concave penalty \citep{Zhang2010}, which have the side effect of biasing (shrinking) any estimated nonzero coefficients toward zero. Further still, many only support Gaussian likelihoods, restricting their application to continuous responses. Table~\ref{tab:capabilities} concisely summarizes these capabilities and limitations.
\begin{table}[ht]
\centering
\small
\begingroup
\setlength{\tabcolsep}{2.8pt}
\begin{tabularx}{\linewidth}{lYYYYYYX}
\toprule
& \shortstack{Fixed \\ selection} & \shortstack{Random \\ selection} & \shortstack{Hier. \\ principle} & \shortstack{$\ell_0$ subset \\ selection} & \shortstack{\vspace{0.07in}GLMMs} & \multicolumn{1}{c}{\shortstack{Unrestr. \\ covariance}} & \multicolumn{1}{c}{\shortstack{Max \\ predictors}} \\
\midrule
\citet{Bondell2010} & \cmark & \cmark & \xmark & \xmark & \xmark & \cmark & $\approx10^1$ \\
\citet{Ibrahim2011} & \cmark & \cmark & \xmark & \xmark & \cmark & \cmark & $\approx10^1$ \\
\citet{Schelldorfer2011} & \cmark & \xmark & - & \xmark & \xmark & \cmark & $\approx10^3$ \\
\citet{Fan2012} & \cmark & \cmark & \xmark & \xmark & \xmark & \xmark & $\approx10^2$ \\
\citet{Peng2012} & \cmark & \cmark & \xmark & \xmark & \xmark & \cmark & $\approx10^1$ \\
\citet{Lin2013} & \cmark & \cmark & \xmark & \xmark & \xmark & \cmark & $\approx10^1$ \\
\citet{Groll2014} & \cmark & \xmark & - & \xmark & \cmark & \cmark & $\approx10^2$ \\
\citet{Pan2014} & \xmark & \cmark & - & \xmark & \cmark & \cmark & $\approx10^1$ \\
\citet{Schelldorfer2014} & \cmark & \xmark & - & \xmark & \cmark & \cmark & $\approx10^3$ \\
\citet{Hui2017a} & \cmark & \cmark & \cmark & \xmark & \cmark & \cmark & $\approx10^1$ \\
\citet{Hui2017b} & \cmark & \cmark & \xmark & \xmark & \cmark & \cmark & $\approx10^1$ \\
\citet{Li2018} & \cmark & \cmark & \xmark & \xmark & \xmark & \cmark & $\approx10^2$ \\
\citet{Heiling2023} & \cmark & \cmark & \xmark & \xmark & \cmark & \cmark & $\approx10^1$ \\
\citet{Kowal2023} & \cmark & \xmark & - & \cmark & \xmark & \cmark & $\approx10^2$ \\
\citet{Sholokhov2024} & \cmark & \cmark & \xmark & \cmark & \xmark & \xmark & $\approx10^3$ \\
Our proposal & \cmark & \cmark & \cmark & \cmark & \cmark & \xmark & $\approx10^4$ \\
\bottomrule
\end{tabularx}
\endgroup
\caption{Capabilities and limitations of existing tools for fitting sparse LMMs. Hier. principle indicates whether the method can enforce that random effects are selected only when their corresponding fixed effects are selected. Unrestr. covariance indicates whether the method allows unrestricted random effect covariance structures. Max predictors is derived from the largest experiment reported in each paper, with the $\approx$ symbol meaning ``on the order of.''}
\label{tab:capabilities}
\end{table}

In parallel to the above line of work, research on sparse linear models (LMs) (i.e., models with fixed effects only) has yielded several remarkable advances over the past few years. In particular, tools now exist that can rapidly solve $\ell_0$ regularized regression---also known as best subset selection---on datasets with thousands of predictors \citep{Hazimeh2020,Dedieu2021} despite the problem being NP-hard \citep{Natarajan1995}. Unlike continuous shrinkage regularizers, $\ell_0$ regularization explicitly penalizes the number of selected predictors, yielding superior performance in selection, interpretation, and, provided the underlying signal is sufficiently strong, prediction \citep{Bertsimas2016,Hastie2020}. These computational breakthroughs rely on coordinate-wise methods, pathwise optimization, and heuristics inspired by early work in \citet{Friedman2007}. Concurrently, these developments have also motivated Bayesian adaptations of $\ell_0$ regularized regression \citep{Kowal2022}, which follow a different computational paradigm but similarly leverage discrete selection to yield sparse solutions. Recent work also explores structured variants of $\ell_0$ subset selection, including grouped and hierarchical selection \citep{Hazimeh2023,Thompson2024}.

In view of the preceding discussion, we introduce a new estimator that facilitates scalable $\ell_0$ subset selection in LMMs. The new estimator is defined as a solution to the problem
\begin{equation}
\label{eq:l0lmm}
\min_{\substack{\boldsymbol{\beta}\in\mathbb{R}^p,\boldsymbol{\gamma}\in\mathbb{R}_+^p \\ \beta_k=0\Rightarrow\gamma_k=0}}\;l(\boldsymbol{\beta},\boldsymbol{\gamma};\mathbf{y},\mathbf{X})+\lambda\alpha\|\boldsymbol{\beta}\|_0+\lambda(1-\alpha)\|\boldsymbol{\gamma}\|_0,
\end{equation}
where $\mathbf{y}=(\mathbf{y}_1^\top,\dots,\mathbf{y}_m^\top)^\top$ and $\mathbf{X}=(\mathbf{X}_1^\top,\dots,\mathbf{X}_m^\top)^\top$ are the full response vector and predictor matrix, respectively, and $l$ is a negative log-likelihood. The $\ell_0$ norm $\|\cdot\|_0$ counts the number of nonzeros and thus induces sparsity in $\boldsymbol{\beta}$ and $\boldsymbol{\gamma}$, with $\lambda\geq0$ and $\alpha\in(0,1]$ serving as tuning parameters. The constraint $\beta_k=0\Rightarrow\gamma_k=0$ means that if a fixed effect's coefficient $\beta_k$ is zero, the corresponding random effect's variance $\gamma_k$ must also be zero, thereby enforcing a hierarchical selection structure. The estimator \eqref{eq:l0lmm} possesses several favorable properties not simultaneously achieved by any existing alternatives. First, it selects fixed \emph{and} random effects so one need not specify any relevant effects a priori. Indeed, with datasets containing many predictors, it is usually unknown which predictors are relevant and, for those that are, whether their effect is fixed or random.\footnote{If prior information is available, it is straightforward to enforce certain effect types for specific predictors by introducing parameter-specific penalty factors.} Second, it enforces the hierarchy that a predictor's random effect can only be active if its fixed effect is also active. From a modeling perspective, this constraint encodes the principle that random effects are deviations from a baseline fixed effect. Third, it allows the fixed and random effects to enter the model with their full effect. In contrast, estimators that use continuous regularizers require post-processing to unwind shrinkage, whereas our approach directly produces unshrunk (unbiased) coefficients.

Inspired by the remarkable computational advances for subset selection in LMs, we develop new algorithmic machinery for subset selection in LMMs on datasets containing at least an order of magnitude more predictors than has previously been feasible. The core components of our algorithmic framework are a coordinate descent method and a local search method, the former for rapidly computing a solution and the latter for refining that solution towards a global minimum. These algorithms readily extend to generalized LMMs (GLMMs) via a penalized quasi-likelihood approximation. On the statistical side, we provide a guarantee for the new estimator in the form of a high-probability bound on the Kullback--Leibler divergence from the true model. The bound does not rely on asymptotic arguments and holds for finite samples. We then evaluate the new estimator on synthetic datasets, stress-testing its scalability across various sample sizes, predictor dimensions, and correlation levels. We further demonstrate its utility on real datasets that require methods capable of handling heterogeneity, enforcing sparsity, and producing interpretable results.

The remaining sections of this paper are arranged as follows. Section~\ref{sec:computation} introduces the algorithmic framework. Section~\ref{sec:generalized} describes the sparse GLMM extension. Section~\ref{sec:statistical} derives the statistical guarantee. Section~\ref{sec:synthetic} presents the synthetic data experiments. Section~\ref{sec:riboflavin} reports the real data experiments. Finally, Section~\ref{sec:concluding} closes the paper with final remarks.

\section{Algorithmic Framework}
\label{sec:computation}

\subsection{Objective Function}

We begin with the Gaussian likelihood and defer discussion of other likelihoods to Section~\ref{sec:generalized}. The Gaussian (negative) log-likelihood, ignoring constants, takes the form
\begin{equation}
\label{eq:llsigma2}
l(\boldsymbol{\beta},\boldsymbol{\gamma},\sigma^2;\mathbf{y},\mathbf{X})=\sum_{i=1}^m\log\det\{\sigma^2\mathbf{V}_i(\boldsymbol{\gamma})\}+\frac{1}{\sigma^2}\sum_{i=1}^m(\mathbf{y}_i-\mathbf{X}_i\boldsymbol{\beta})^\top\mathbf{V}_i^{-1}(\boldsymbol{\gamma})(\mathbf{y}_i-\mathbf{X}_i\boldsymbol{\beta}),
\end{equation}
where $\sigma^2>0$ is the noise variance and $\mathbf{V}_i(\boldsymbol{\gamma})$ is a covariance matrix that captures the correlation between observations in the $i$th cluster as induced by the random effects:
\begin{equation}
\label{eq:vmat}
\mathbf{V}_i(\boldsymbol{\gamma})=\mathbf{I}+\mathbf{X}_i\operatorname{diag}(\boldsymbol{\gamma})\mathbf{X}_i^\top.
\end{equation}
This restricted covariance structure is a key ingredient in the scalability of the proposed method because changing a single $\gamma_k$ induces only a rank-one change in $\mathbf{V}_i(\boldsymbol{\gamma})$, which, as we will see, enables efficient updates of $\mathbf{V}_i^{-1}(\boldsymbol{\gamma})$ and $\log\det\{\mathbf{V}_i(\boldsymbol{\gamma})\}$. The predictor matrix $\mathbf{X}$ is assumed to be standardized to have columns with unit $\ell_2$ norm. Whereas the noise variance in non-mixed LMs can be ignored when estimating regression coefficients, it cannot be ignored in the mixed case since it is not independent of the other parameters. Nonetheless, it is still possible to simplify computation by profiling out $\sigma^2$ with its maximum-likelihood estimator, which admits a closed-form solution given $\boldsymbol{\beta}$ and $\boldsymbol{\gamma}$:
\begin{equation}
\label{eq:sigma2}
\hat{\sigma}^2(\boldsymbol{\beta},\boldsymbol{\gamma})=\frac{1}{\ndot}\sum_{i=1}^m(\mathbf{y}_i-\mathbf{X}_i\boldsymbol{\beta})^\top\mathbf{V}_i^{-1}(\boldsymbol{\gamma})(\mathbf{y}_i-\mathbf{X}_i\boldsymbol{\beta}),
\end{equation}
where $\ndot=\sum_{i=1}^mn_i$. Substituting \eqref{eq:sigma2} into \eqref{eq:llsigma2} and removing constants that do not depend on $\boldsymbol{\beta}$ or $\boldsymbol{\gamma}$ yields
\begin{equation}
\label{eq:ll}
l(\boldsymbol{\beta},\boldsymbol{\gamma};\mathbf{y},\mathbf{X})=\sum_{i=1}^m\log\det\{\mathbf{V}_i(\boldsymbol{\gamma})\}+\ndot\log\left\{\sum_{i=1}^m(\mathbf{y}_i-\mathbf{X}_i\boldsymbol{\beta})^\top\mathbf{V}_i^{-1}(\boldsymbol{\gamma})(\mathbf{y}_i-\mathbf{X}_i\boldsymbol{\beta})\right\}.
\end{equation}
We omit the dependence on $\mathbf{y}$ and $\mathbf{X}$ herein and write the negative log-likelihood as $l(\boldsymbol{\beta},\boldsymbol{\gamma})$.

Given $\boldsymbol{\beta}$ and $\boldsymbol{\gamma}$, it is often of interest to predict the (unobservable) random effects $\mathbf{u}_1,\dots,\mathbf{u}_m$. The best linear unbiased predictor for the random effects in cluster $i$ is
\begin{equation}
\label{eq:blup}
\hat{\mathbf{u}}_i=\operatorname{diag}(\boldsymbol{\gamma})\mathbf{X}_i^\top\mathbf{V}_i^{-1}(\boldsymbol{\gamma})(\mathbf{y}_i-\mathbf{X}_i\boldsymbol{\beta}).
\end{equation}
Though not directly part of the optimization process, $\hat{\mathbf{u}}_i$ are useful for generating predictions.

Returning to the optimization problem, we regularize the negative log-likelihood \eqref{eq:ll} as
\begin{equation}
\label{eq:l0lmm2}
\min_{\substack{\boldsymbol{\beta}\in\mathbb{R}^p,\boldsymbol{\gamma}\in\mathbb{R}_+^p \\ \beta_k=0\Rightarrow\gamma_k=0}}\;f(\boldsymbol{\beta},\boldsymbol{\gamma}):=l(\boldsymbol{\beta},\boldsymbol{\gamma})+\lambda\alpha\|\boldsymbol{\beta}\|_0+\lambda(1-\alpha)\|\boldsymbol{\gamma}\|_0.
\end{equation}
The role of the parameter $\alpha$ is to control the composition of sparsity. In particular, having fixed $\lambda$, a larger value of $\alpha$ means that a random effect is more likely to be nonzero when its corresponding fixed effect is nonzero. In the limiting case when $\alpha=1$, there is no penalty for a component in $\boldsymbol{\gamma}$ being nonzero once the corresponding component in $\boldsymbol{\beta}$ is nonzero.

The objective function $f$ in \eqref{eq:l0lmm2} is nonsmooth due to the $\ell_0$ norms. It is also nonconvex with two sources of nonconvexity: the negative log-likelihood $l$ and the $\ell_0$ norms. These properties rule out the direct application of off-the-shelf optimization software and necessitate the development of tailored algorithms that scale to real-world problem instances.

\subsection{Coordinate Descent}

Coordinate descent algorithms iteratively minimize multi-dimensional objective functions by optimizing one or more coordinates at a time while keeping all others fixed. This process involves repeated cycles through the coordinates until convergence is achieved. In contrast to simultaneous updating of all coordinates (i.e., gradient descent), coordinate-wise updating is well-suited to sparse, high-dimensional problems where only a small subset of coordinates is important. We refer the interested reader to \citet{Wright2015} for an overview.

The coordinate descent algorithm we develop here optimizes a block of two coordinates simultaneously: $(\beta_k,\gamma_k)$, i.e., the parameters corresponding to predictor $k$. The motivation for this design is two-fold: (i) these parameters are not easily separable due to the hierarchy constraint $\beta_k=0\Rightarrow\gamma_k=0$, and (ii) both parameters need several of the same quantities for their update, so we can reduce overheads by updating both at the same time. A single cycle of our coordinate descent algorithm is thus a single pass over all blocks $(\beta_1,\gamma_1),\dots,(\beta_p,\gamma_p)$.

Suppose we are in cycle $t$ and thus far we have optimized the parameter blocks $1,\dots,k-1$. A vanilla coordinate descent approach would proceed by updating the $k$th block as
\begin{equation}
\label{eq:cwmin1}
(\beta_k^{(t)},\gamma_k^{(t)})\gets\underset{\substack{\beta_k\in\mathbb{R},\gamma_k\in\mathbb{R}_+ \\ \beta_k=0\Rightarrow\gamma_k=0}}{\arg\min}\;f(\beta_1^{(t)},\dots,\beta_k,\dots,\beta_p^{(t-1)},\gamma_1^{(t)},\dots,\gamma_k,\dots,\gamma_p^{(t-1)}).
\end{equation}
The difficulty with this particular update is that no closed-form solution is available to the optimization problem in \eqref{eq:cwmin1}. A separate iterative algorithm is required. This difficulty is in contrast to non-mixed LMs, where the coordinate-wise update admits an analytical solution. Hence, rather than fully optimizing in each coordinate block, we partially optimize in each block by taking a single gradient descent step towards the optimum for that block.

Our partial minimization scheme can be understood in the context of the block descent lemma of \citet{Beck2013}, with which we can bound the negative log-likelihood:
\begin{equation}
\label{eq:bdlemma}
\begin{split}
l(\boldsymbol{\beta},\boldsymbol{\gamma})\leq \bar{l}_{\bar{L}_k}(\boldsymbol{\beta},\boldsymbol{\gamma};\tilde{\boldsymbol{\beta}},\tilde{\boldsymbol{\gamma}})&:=l(\tilde{\boldsymbol{\beta}},\tilde{\boldsymbol{\gamma}})+\nabla_{\beta_k}l(\tilde{\boldsymbol{\beta}},\tilde{\boldsymbol{\gamma}})(\beta_k-\tilde{\beta}_k)+\nabla_{\gamma_k}l(\tilde{\boldsymbol{\beta}},\tilde{\boldsymbol{\gamma}})(\gamma_k-\tilde{\gamma}_k) \\
&\hspace{2in}+\frac{\bar{L}_k}{2}\left\{(\beta_k-\tilde{\beta}_k)^2+(\gamma_k-\tilde{\gamma}_k)^2\right\},
\end{split}
\end{equation}
which holds for any $(\boldsymbol{\beta},\boldsymbol{\gamma})$ and $(\tilde{\boldsymbol{\beta}},\tilde{\boldsymbol{\gamma}})$ differing only in coordinate block $k$ (for the coordinate descent update below, $(\tilde{\boldsymbol{\beta}},\tilde{\boldsymbol{\gamma}})$ is taken to be the current iterate before block $k$ is updated). The notation $\nabla_{\beta_k}$ and $\nabla_{\gamma_k}$ is shorthand for gradients with respect to $\beta_k$ and $\gamma_k$, and $\bar{L}_k\geq L_k$ where $L_k$ is the Lipschitz constant of $l$ as a function of $\beta_k$ and $\gamma_k$. We may now upper bound the original objective function $f(\boldsymbol{\beta},\boldsymbol{\gamma})$ by a new objective function $\bar{f}_{\bar{L}_k}(\boldsymbol{\beta},\boldsymbol{\gamma};\tilde{\boldsymbol{\beta}},\tilde{\boldsymbol{\gamma}})$, given by
\begin{equation}
\label{eq:obj}
f(\boldsymbol{\beta},\boldsymbol{\gamma})\leq\bar{f}_{\bar{L}_k}(\boldsymbol{\beta},\boldsymbol{\gamma};\tilde{\boldsymbol{\beta}},\tilde{\boldsymbol{\gamma}}):=\bar{l}_{\bar{L}_k}(\boldsymbol{\beta},\boldsymbol{\gamma};\tilde{\boldsymbol{\beta}},\tilde{\boldsymbol{\gamma}})+\lambda\alpha\|\boldsymbol{\beta}\|_0+\lambda(1-\alpha)\|\boldsymbol{\gamma}\|_0.
\end{equation}
Taking $\bar{f}$ as our new objective function, the update to the $k$th block of coordinates becomes
\begin{equation}
\label{eq:cwmin2}
(\beta_k^{(t)},\gamma_k^{(t)})\gets\underset{\substack{\beta_k\in\mathbb{R},\gamma_k\in\mathbb{R}_+ \\ \beta_k=0\Rightarrow\gamma_k=0}}{\arg\min}\;\bar{f}_{\bar{L}_k}(\beta_1^{(t)},\dots\beta_k,\dots,\beta_p^{(t-1)},\gamma_1^{(t)},\dots,\gamma_k,\dots,\gamma_p^{(t-1)};\boldsymbol{\beta}^{(t)},\boldsymbol{\gamma}^{(t)}),
\end{equation}
where the arguments $(\boldsymbol{\beta}^{(t)},\boldsymbol{\gamma}^{(t)})$ after the semicolon denote the current iterate used to evaluate the bound before block $k$ is updated, i.e., $(\tilde{\boldsymbol{\beta}},\tilde{\boldsymbol{\gamma}})=(\boldsymbol{\beta}^{(t)},\boldsymbol{\gamma}^{(t)})$ in \eqref{eq:obj}, with coordinates $1,\ldots,k-1$ already updated in cycle $t$ and coordinates $k,\ldots,p$ still at their cycle $t-1$ values. In contrast to \eqref{eq:cwmin1}, the update \eqref{eq:cwmin2} admits a closed-form solution, described by Proposition~\ref{prop:update}.
\begin{proposition}
\label{prop:update}
Define the thresholding function
\begin{equation}
\label{eq:threshold}
T_{\bar{L}}(\beta,\gamma;\lambda,\alpha)=
\begin{dcases}
(0,0) & \text{if }\beta^2<\frac{2\lambda\alpha}{\bar{L}}\text{ and }\beta^2+\gamma_+^2<\frac{2\lambda}{\bar{L}} \\
(\beta,0) & \text{if }\gamma_+^2<\frac{2\lambda(1-\alpha)}{\bar{L}} \\
(\beta,\gamma_+) & \text{otherwise},
\end{dcases}
\end{equation}
where $\gamma_+=\max(\gamma,0)$. Furthermore, define the gradients
\begin{equation}
\label{eq:gradbeta}
\nabla_{\beta_k}l(\boldsymbol{\beta},\boldsymbol{\gamma})=-2\ndot\frac{\sum_{i=1}^m\mathbf{x}_{ik}^\top\mathbf{V}_i^{-1}(\boldsymbol{\gamma})(\mathbf{y}_i-\mathbf{X}_i\boldsymbol{\beta})}{\sum_{i=1}^m(\mathbf{y}_i-\mathbf{X}_i\boldsymbol{\beta})^\top\mathbf{V}_i^{-1}(\boldsymbol{\gamma})(\mathbf{y}_i-\mathbf{X}_i\boldsymbol{\beta})}
\end{equation}
and
\begin{equation}
\label{eq:gradgamma}
\nabla_{\gamma_k}l(\boldsymbol{\beta},\boldsymbol{\gamma})=\sum_{i=1}^m\mathbf{x}_{ik}^\top\mathbf{V}_i^{-1}(\boldsymbol{\gamma})\mathbf{x}_{ik}-\ndot\frac{\sum_{i=1}^m\left\{\mathbf{x}_{ik}^\top\mathbf{V}_i^{-1}(\boldsymbol{\gamma})(\mathbf{y}_i-\mathbf{X}_i\boldsymbol{\beta})\right\}^2}{\sum_{i=1}^m(\mathbf{y}_i-\mathbf{X}_i\boldsymbol{\beta})^\top\mathbf{V}_i^{-1}(\boldsymbol{\gamma})(\mathbf{y}_i-\mathbf{X}_i\boldsymbol{\beta})}.
\end{equation}
Then the optimization problem \eqref{eq:cwmin2} is solved by
\begin{equation*}
(\beta_k^{(t)},\gamma_k^{(t)})=T_{\bar{L}_k}\left(\beta_k^{(t-1)}-\frac{1}{\bar{L}_k}\nabla_{\beta_k}l(\boldsymbol{\beta}^{(t)},\boldsymbol{\gamma}^{(t)}),\gamma_k^{(t-1)}-\frac{1}{\bar{L}_k}\nabla_{\gamma_k}l(\boldsymbol{\beta}^{(t)},\boldsymbol{\gamma}^{(t)});\lambda,\alpha\right).
\end{equation*}
\end{proposition}
\begin{proof}
See Appendix~\ref{app:update}.
\end{proof}
Proposition~\ref{prop:update} gives the update as a gradient descent step with step size equal to the inverse of the Lipschitz constant, followed by a threshold step. The particular form of the threshold that appears here follows from the $\ell_0$ regularizers and the hierarchy constraint. In practice, the step size scalar $\bar{L}_k$ is determined via backtracking line search by starting with a large value and iteratively decreasing the estimated constant until inequality \eqref{eq:bdlemma} is satisfied.

We now present Algorithm~\ref{alg:cd}, which summarizes the coordinate descent routine.
\begin{algorithm}[ht]
\caption{Coordinate Descent}
\label{alg:cd}
\begin{algorithmic}
\REQUIRE $(\boldsymbol{\beta}^{(0)},\boldsymbol{\gamma}^{(0)})\in\mathbb{R}^p\times\mathbb{R}_+^p$ and $T\in\mathbb{N}$
\FOR{$t=1,\dots,T$}
\STATE $(\boldsymbol{\beta}^{(t)},\boldsymbol{\gamma}^{(t)})\gets(\boldsymbol{\beta}^{(t-1)},\boldsymbol{\gamma}^{(t-1)})$
\FOR{$k=1,\dots,p$}
\STATE $(\beta_k^{(t)},\gamma_k^{(t)})\gets\underset{\substack{\beta_k\in\mathbb{R},\gamma_k\in\mathbb{R}_+ \\ \beta_k=0\Rightarrow\gamma_k=0}}{\arg\min}\;\bar{f}_{\bar{L}_k}(\beta_1^{(t)},\dots\beta_k,\dots,\beta_p^{(t-1)},\gamma_1^{(t)},\dots,\gamma_k,\dots,\gamma_p^{(t-1)};\boldsymbol{\beta}^{(t)},\boldsymbol{\gamma}^{(t)})$
\ENDFOR
\IF{converged}
\STATE \textbf{break}
\ENDIF
\ENDFOR
\RETURN $(\boldsymbol{\beta}^{(t)},\boldsymbol{\gamma}^{(t)})$
\end{algorithmic}
\end{algorithm}
Our implementation of the algorithm adopts several heuristics including gradient screening, gradient sorting, and active set updates. These heuristics are described in Appendix~\ref{app:heuristics}.

Each iteration of Algorithm~\ref{alg:cd} requires computation of the gradients \eqref{eq:gradbeta} and \eqref{eq:gradgamma}. These gradients, in turn, involve the inverses of the matrices $\mathbf{V}_1(\boldsymbol{\gamma}),\dots,\mathbf{V}_m(\boldsymbol{\gamma})$, each of dimension $n_i\times n_i$. Computing these inverses naïvely takes $O(n_i^3)$ operations, which is computationally prohibitive if $n_i$ is large or $p$ is large (since the inverses need to be computed $p$ times to do one full cycle over the coordinates). Fortunately, since each coordinate descent iteration only updates a single coordinate of $\boldsymbol{\gamma}$, we can perform a rank-one update to get the new inverse. Let $\tilde{\mathbf{V}}_i^{-1}=\mathbf{V}_i^{-1}(\gamma_1^{(t)},\dots,\gamma_k^{(t-1)},\dots,\gamma_p^{(t-1)})$ be the inverse from the current iteration and let $\hat{\mathbf{V}}_i^{-1}=\mathbf{V}_i^{-1}(\gamma_1^{(t)},\dots,\gamma_k^{(t)},\gamma_{k+1}^{(t-1)},\dots,\gamma_p^{(t-1)})$ be the inverse for the next iteration. Then, since $\boldsymbol{\gamma}$ has changed in the $k$th coordinate only, the Sherman-Morrison formula gives
\begin{equation*}
\hat{\mathbf{V}}_i^{-1}=\tilde{\mathbf{V}}_i^{-1}-\frac{\tilde{\mathbf{V}}_i^{-1}\mathbf{x}_{ik}\mathbf{x}_{ik}^\top\tilde{\mathbf{V}}_i^{-1}}{(\gamma_k^{(t)}-\gamma_k^{(t-1)})^{-1}+\mathbf{x}_{ik}^\top\tilde{\mathbf{V}}_i^{-1}\mathbf{x}_{ik}}.
\end{equation*}
This update requires only $O(n_i^2)$ operations in contrast to $O(n_i^3)$ for the naïve update.

Though the log determinant of $\mathbf{V}_i(\boldsymbol{\gamma})$ is not required for the gradients and thus not needed by Algorithm~\ref{alg:cd}, it can be useful for other purposes, e.g., monitoring convergence. Similar to the rank-one update strategy for the inverse, we can perform a rank-one update for the log determinant via the matrix determinant lemma \citep{Harville1997}, which gives
\begin{equation*}
\log\,\det(\hat{\mathbf{V}}_i)=\log\,\det(\tilde{\mathbf{V}}_i)+\log\left(1+(\gamma_k^{(t)}-\gamma_k^{(t-1)})\mathbf{x}_{ik}^\top\tilde{\mathbf{V}}_i^{-1}\mathbf{x}_{ik}\right).
\end{equation*}
As with the inverse, this update takes $O(n_i^2)$ operations versus the naïve update of $O(n_i^3)$.

We now establish the convergence of our coordinate descent routine in Theorem~\ref{thrm:converge}.
\begin{theorem}
\label{thrm:converge}
Suppose $\mathbf{y}$ does not lie in the column space of $\mathbf{X}$, and that $\boldsymbol{\beta}\in\mathbb{R}^p$ and $\boldsymbol{\gamma}\in\mathbb{R}_+^p$ with $\|\boldsymbol{\beta}\|_\infty\leq\bar{\beta}\in\mathbb{R}_+$ and $\|\boldsymbol{\gamma}\|_\infty\leq\bar{\gamma}\in\mathbb{R}_+$. Let $\{(\boldsymbol{\beta}^{(t)},\boldsymbol{\gamma}^{(t)})\}_{t\in\mathbb{N}}$ be a sequence of iterates generated by Algorithm~\ref{alg:cd}. Then, for any $\bar{L}_k\geq L_k$ for $k=1,\dots,p$, the sequence of objective values $\{f(\boldsymbol{\beta}^{(t)},\boldsymbol{\gamma}^{(t)})\}_{t\in\mathbb{N}}$ is decreasing, convergent, and satisfies the inequality
\begin{equation*}
f(\boldsymbol{\beta}^{(t)},\boldsymbol{\gamma}^{(t)})-f(\boldsymbol{\beta}^{(t+1)},\boldsymbol{\gamma}^{(t+1)})\geq\sum_{k=1}^p\left(\frac{\bar{L}_k-L_k}{2}(\beta_k^{(t)}-\beta_k^{(t+1)})^2+\frac{\bar{L}_k-L_k}{2}(\gamma_k^{(t)}-\gamma_k^{(t+1)})^2\right).
\end{equation*}
Furthermore, for any $\bar{L}_k>L_k$, the sequence of active sets $\{\mathcal{A}(\boldsymbol{\beta}^{(t)},\boldsymbol{\gamma}^{(t)})\}_{t\in\mathbb{N}}$, i.e., the set of predictors with nonzero effects, converges to a fixed set in finitely many iterations.
\end{theorem}
\begin{proof}
See Appendix~\ref{app:converge}.
\end{proof}
Theorem~\ref{thrm:converge} ensures the sequence of objective values generated by Algorithm~\ref{alg:cd} is monotonically decreasing and convergent. Moreover, it guarantees the sequence of predictors with nonzero effects (the active set) converges to a fixed set in a finite number of iterations. These results are not direct applications of existing convergence results but rely on proving boundedness and Lipschitz continuity of the objective function, which hold provided $\boldsymbol{\beta}$ and $\boldsymbol{\gamma}$ lie in a compact set. While it is straightforward to ensure that Algorithm~\ref{alg:cd} satisfies this condition by introducing the box constraints $\|\boldsymbol{\beta}\|_\infty\leq\bar{\beta}\in\mathbb{R}_+$ and $\|\boldsymbol{\gamma}\|_\infty\leq\bar{\gamma}\in\mathbb{R}_+$, where $\|\cdot\|_\infty$ is the maximum absolute value, it appears unnecessary for convergence in practice.

\subsection{Local Search}

While coordinate descent algorithms can guarantee global minimizers for convex objective functions, the nonconvex nature of our objective function means Algorithm~\ref{alg:cd} cannot provide such a guarantee. This limitation is exacerbated when predictors are highly correlated, since several competing active sets may attain similar objective values, making coordinate descent more susceptible to converging to a suboptimal solution. To address this issue, we devise a local search routine to refine the solution from coordinate descent. Inspired by similar strategies proposed in \citet{Beck2013a}, \citet{Hazimeh2020}, and \citet{Dedieu2021}, our routine systematically perturbs the active set of predictors by swapping previously excluded predictors into the model while simultaneously removing active ones. The process is iterated with a configurable neighborhood size $s\in\mathbb{N}$, which governs the number of predictors swapped, until no further iterations improve the objective value.

The local search problem described above can be formalized mathematically as
\begin{equation}
\label{eq:ls}
\min_{\substack{\boldsymbol{\xi}\in\mathbb{R}^p,\boldsymbol{\zeta}\in\mathbb{R}_+^p,\xi_k=0\Rightarrow\zeta_k=0 \\ \mathbf{z}_1\in\{0,1\}^p,\mathbf{z}_2\in\{0,1\}^p \\ \operatorname{Supp}(\mathbf{z}_1)\subseteq\mathcal{A},\,\operatorname{Supp}(\mathbf{z}_2)\subseteq\mathcal{A}^c \\ \|\mathbf{z}_1\|_0\leq s,\,\|\mathbf{z}_2\|_0\leq s}}\;f(\boldsymbol{\beta}-\mathbf{z}_1\odot\boldsymbol{\beta}+\mathbf{z}_2\odot\boldsymbol{\xi},\boldsymbol{\gamma}-\mathbf{z}_1\odot\boldsymbol{\gamma}+\mathbf{z}_2\odot\boldsymbol{\zeta}),
\end{equation}
where $\operatorname{Supp}(\mathbf{z}):=\{j:z_j\neq0\}$. In words, given an incumbent solution $(\boldsymbol{\beta},\boldsymbol{\gamma})$, local search finds a subset of size $s$ to remove from the current active set $\mathcal{A}$ and replace with a subset of size $s$ from the inactive set $\mathcal{A}^c$ such that the objective function $f(\boldsymbol{\beta},\boldsymbol{\gamma})$ is minimized. Here, $\mathbf{z}_1$ and $\mathbf{z}_2$ are binary indicator vectors specifying the predictors to be removed and added, respectively, while $\boldsymbol{\xi}$ and $\boldsymbol{\zeta}$ contain the parameter values for the newly included predictors. When $s=1$, the problem simplifies to evaluating the objective for all possible single swaps, making it computationally efficient whilst having the capability to escape poor local optima.

After solving the local search problem \eqref{eq:ls}, the updated solution serves as the initialization point for a subsequent run of coordinate descent. This iterative interplay between the two methods---described in Algorithm~\ref{alg:ls}---combines the fast, fine-grained updates from coordinate descent with the broader, exploratory improvements from local search.
\begin{algorithm}[ht]
\caption{Coordinate Descent with Local Search}
\label{alg:ls}
\begin{algorithmic}
\REQUIRE $(\hat{\boldsymbol{\beta}}^{(0)},\hat{\boldsymbol{\gamma}}^{(0)})\in\mathbb{R}^p\times\mathbb{R}_+^p$ and $T\in\mathbb{N}$
\FOR{$t=1,\dots,T$}
\STATE Take $(\boldsymbol{\beta}^{(t)},\boldsymbol{\gamma}^{(t)})$ as the output of Algorithm~\ref{alg:cd} initialized at $(\hat{\boldsymbol{\beta}}^{(t-1)},\hat{\boldsymbol{\gamma}}^{(t-1)})$
\STATE Take $(\hat{\boldsymbol{\beta}}^{(t)},\hat{\boldsymbol{\gamma}}^{(t)})$ as the solution to the local search problem \eqref{eq:ls} at $(\boldsymbol{\beta}^{(t)},\boldsymbol{\gamma}^{(t)})$
\IF{$f(\hat{\boldsymbol{\beta}}^{(t)},\hat{\boldsymbol{\gamma}}^{(t)})=f(\boldsymbol{\beta}^{(t)},\boldsymbol{\gamma}^{(t)})$}
\STATE \textbf{break}
\ENDIF
\ENDFOR
\RETURN $(\boldsymbol{\beta}^{(t)},\boldsymbol{\gamma}^{(t)})$
\end{algorithmic}
\end{algorithm}
The algorithm terminates when neither coordinate descent nor local search achieves further improvement. Our implementation of Algorithm~\ref{alg:ls} solves the local search problem with $s=1$.

\subsection{Pathwise Optimization}

Typical practice in sparse modeling is to estimate multiple models with different values of $\lambda$ so different sparsity levels can be compared (e.g., using cross-validation). It is usually computationally expedient to compute these models in a pathwise fashion by sequentially initializing the optimizer at the previous solution. In our case, the pathwise optimization approach has the benefit of (i) reducing overall computation time since the solution at $\lambda=\lambda^{(r+1)}$ is likely near the solution at $\lambda=\lambda^{(r)}$, and (ii) helping the optimizer navigate the nonconvex optimization surface, thereby improving the quality of the models.

We construct the regularization sequence $\{\lambda^{(r)}\}_{r=1}^R$ so the model path evolves in a sparse-to-dense manner, beginning with a large amount of regularization and ending with a small amount, i.e., $\lambda^{(r)}>\lambda^{(r+1)}$. Moreover, we configure the sequence such that the model from the solve at $\lambda=\lambda^{(r+1)}$ has a different set of active predictors to the model from the solve with $\lambda=\lambda^{(r)}$. Proposition~\ref{prop:lambda} provides a formula for choosing $\lambda^{(r+1)}$ to guarantee this behavior.
\begin{proposition}
\label{prop:lambda}
Let $(\boldsymbol{\beta}^{(r)},\boldsymbol{\gamma}^{(r)})$ be a coordinate descent solution from $\lambda=\lambda^{(r)}$. Denote its active set by $\mathcal{A}^{(r)}$. Then running coordinate descent initialized at $(\boldsymbol{\beta}^{(r)},\boldsymbol{\gamma}^{(r)})$ with $\lambda$ set to
\begin{equation*}
\begin{split}
&\lambda^{(r+1)}= \\
&\hspace{0.4in}c\max_{k\not\in\mathcal{A}^{(r)}}\;\max\left(\frac{\{\nabla_{\beta_k}l(\boldsymbol{\beta}^{(r)},\boldsymbol{\gamma}^{(r)})\}^2}{2\alpha\bar{L}_k},\frac{\{\nabla_{\beta_k}l(\boldsymbol{\beta}^{(r)},\boldsymbol{\gamma}^{(r)})\}^2+\max(-\nabla_{\gamma_k}l(\boldsymbol{\beta}^{(r)},\boldsymbol{\gamma}^{(r)}),0)^2}{2\bar{L}_k}\right)
\end{split}
\end{equation*}
yields a solution $(\boldsymbol{\beta}^{(r+1)},\boldsymbol{\gamma}^{(r+1)})$ whose active set $\mathcal{A}^{(r+1)}$ differs from $\mathcal{A}^{(r)}$ for $c\in[0,1)$.
\end{proposition}
\begin{proof}
See Appendix~\ref{app:lambda}.
\end{proof}

\subsection{Computational Complexity}

Assuming the number of clusters $m$ is fixed, by leveraging low-rank matrix inverse updates, a single cycle of coordinate descent across all $p$ coordinates requires $O(p\sum_{i=1}^m n_i^2)$ operations. If instead $m$ varies and the cluster size is fixed (say $n_i=n$ for all $i=1,\dots,m$), the number of operations required becomes $O(\ndot p)$, where $\ndot=mn$. Running a single round of local search after coordinate descent does not alter the overall complexity since we need only perform $p$ swaps for each active predictor, with the number of active predictors fixed.

To verify the complexity analysis and obtain a sense of run times, we measure the time to fit a full regularization path over a sequence of 100 values of $\lambda$. Figure~\ref{fig:timings} reports these times as a function of both $\ndot$ and $p$.
\begin{figure}[t]
\centering
\input{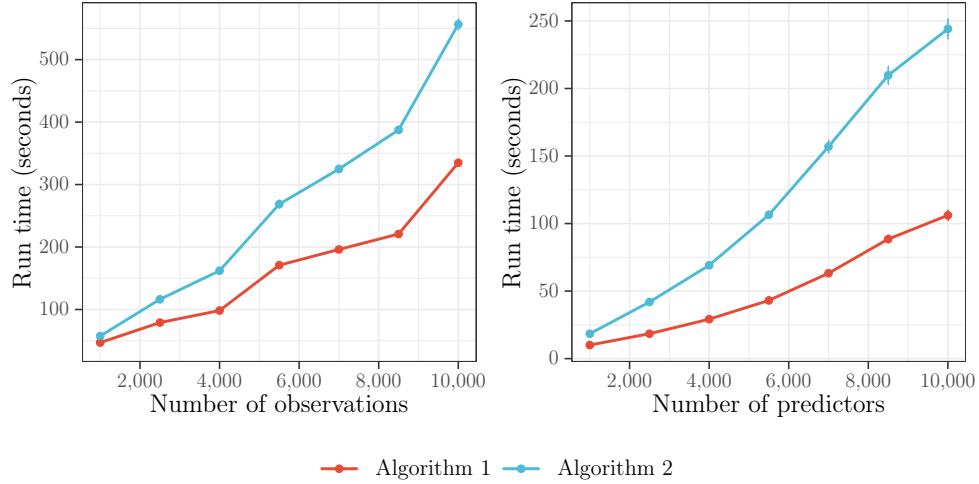}
\caption{Run time in seconds for Algorithm~\ref{alg:cd} (coordinate descent) and Algorithm~\ref{alg:ls} (coordinate descent with local search). The number of predictors $p=100$ in the left plot and the number of observations $\ndot=100$ in the right plot. The number of clusters $m=\lceil\ndot/10\rceil$. The averages (points) and standard errors (error bars) are measured over 100 synthetic datasets.}
\label{fig:timings}
\end{figure}
As predicted, the curves are mostly linear. Naturally, local search introduces additional computational costs, with the exact cost depending on $\ndot$ and $p$. Even so, one can fit a path of 100 models with 10,000 predictors in less than 5 minutes using coordinate descent with local search and 2 minutes using coordinate descent alone.

\section{Sparse GLMMs}
\label{sec:generalized}

\subsection{Non-Gaussian Likelihoods}

The preceding section focuses on $\ell_0$ regularized LMMs where the response is continuous and the likelihood is Gaussian. GLMMs are flexible generalizations of LMMs that handle a broad range of response types by extending the likelihood to include the entire exponential family \citep[see, e.g.,][]{Stroup2024}. For instance, similar to non-mixed generalized LMs (GLMs), GLMMs can accommodate binary response data via a Bernoulli likelihood and count response data via a Poisson likelihood. The conditional expectation of a GLMM satisfies
\begin{equation*}
g(\operatorname{E}[y_{ij}\mid\mathbf{u}_i])=\eta_{ij},\quad\eta_{ij}=\mathbf{x}_{ij}^\top(\boldsymbol{\beta}+\mathbf{u}_i),
\end{equation*}
where $g$ is a so-called link function (e.g., logit for Bernoulli or log for Poisson). Unfortunately, unlike their non-mixed GLM cousins, GLMMs are not nearly as computationally tractable as the likelihoods involved do not admit analytic forms except in the Gaussian case. As a consequence, it is not possible to fit an $\ell_0$ regularized GLMM by direct application of Algorithm~\ref{alg:cd} or Algorithm~\ref{alg:ls} since these rely on the availability of closed-form gradients.

\subsection{Penalized Quasi-Likelihood Approximation}

To address intractability, we employ a penalized quasi-likelihood (PQL) approximation \citep{Breslow1993} to fit $\ell_0$ regularized GLMMs in the non-Gaussian case. PQL approximations are among the most scalable techniques for fitting GLMMs since they effectively transform the likelihood into a tractable Gaussian form.\footnote{Although computationally efficient, PQL approximations can lead to biased parameter estimates; see \citet{Capanu2013} for a discussion on the relative merits of various GLMM estimation methods.} Briefly, the approximation replaces the likelihood with its quasi-likelihood and then iteratively linearizes the mean function around the current estimate $\hat{\mathbf{u}}_i$ of the random effects $\mathbf{u}_i$. At each iteration, the Gaussian likelihood is optimized with working responses $\mathbf{z}_i=(z_{i1},\dots,z_{in_i})^\top$ and weights $\mathbf{w}_i=(w_{i1},\dots,w_{in_i})^\top$, whose elements are given by
\begin{equation}
\label{eq:workresponse}
z_{ij}=\eta_{ij}+\frac{y_{ij}-\operatorname{E}[y_{ij}\mid\hat{\mathbf{u}}_i]}{\nabla g^{-1}(\eta_{ij})}
\end{equation}
and
\begin{equation}
\label{eq:workweights}
w_{ij}=\frac{\{\nabla g^{-1}(\eta_{ij})\}^2}{\operatorname{Var}[y_{ij}\mid\hat{\mathbf{u}}_i]}.
\end{equation}
Here, $\nabla g^{-1}$ is the gradient of the inverse link function, and $\operatorname{E}[y_{ij}\mid\hat{\mathbf{u}}_i]$ and $\operatorname{Var}[y_{ij}\mid\hat{\mathbf{u}}_i]$ are the conditional mean and conditional variance of $y_{ij}$ as prescribed under the distributional family assumed on the response, both evaluated at the current estimate of the random effects. Recall that these estimates are computable using the best linear unbiased predictor in \eqref{eq:blup}.

Algorithm~\ref{alg:pql} implements the above weighting scheme with the Gaussian likelihood optimization performed by coordinate descent and (optionally) local search.
\begin{algorithm}[ht]
\caption{PQL Approximation}
\label{alg:pql}
\begin{algorithmic}[1]
\REQUIRE $(\boldsymbol{\beta}^{(0)},\boldsymbol{\gamma}^{(0)})\in\mathbb{R}^p\times\mathbb{R}_+^p$ and $T\in\mathbb{N}$
\FOR{$t=1,\dots,T$}
\STATE Plug $\boldsymbol{\beta}^{(t-1)}$ and $\boldsymbol{\gamma}^{(t-1)}$ into \eqref{eq:blup} to get $\hat{\mathbf{u}}_1,\dots,\hat{\mathbf{u}}_m$
\STATE Plug $\boldsymbol{\beta}^{(t-1)}$ and $\hat{\mathbf{u}}_1,\dots,\hat{\mathbf{u}}_m$ into \eqref{eq:workresponse} and \eqref{eq:workweights} to get $\mathbf{z}_1,\dots,\mathbf{z}_m$ and $\mathbf{w}_1,\dots,\mathbf{w}_m$
\STATE Run Algorithm~\ref{alg:cd} or \ref{alg:ls} with $(\boldsymbol{\beta}^{(t-1)},\boldsymbol{\gamma}^{(t-1)})$, $\mathbf{y}_i=\mathbf{z}_i$, and $\mathbf{w}_i$, $i=1,\dots,m$, to get $(\boldsymbol{\beta}^{(t)},\boldsymbol{\gamma}^{(t)})$
\IF{converged}
\STATE \textbf{break}
\ENDIF
\ENDFOR
\RETURN $(\boldsymbol{\beta}^{(t)},\boldsymbol{\gamma}^{(t)})$
\end{algorithmic}
\end{algorithm}
To incorporate weights into either of these algorithms one need only replace the covariance matrix \eqref{eq:vmat} with
\begin{equation*}
\mathbf{V}_i(\boldsymbol{\gamma})=\operatorname{diag}(\mathbf{w}_i^{-1})+\mathbf{X}_i\operatorname{diag}(\boldsymbol{\gamma})\mathbf{X}_i^\top,
\end{equation*}
where $\mathbf{w}_i^{-1}$ is the elementwise reciprocal. The calls to coordinate descent or local search inside Algorithm~\ref{alg:pql} are initialized at the previous iterates, making the algorithm significantly faster than multiple individual runs of coordinate descent or local search. We remark that the algorithm resembles iteratively weighted least squares, which is used in the well-known and highly scalable \texttt{glmnet} package to fit $\ell_1$ regularized GLMs \citep{Tay2023}.

\subsection{Toolkit}

We implement our method for learning sparse GLMMs in the \texttt{R} package \texttt{glmmsel} (\underline{g}eneralized \underline{l}inear \underline{m}ixed \underline{m}odel \underline{sel}ection). At the time of writing, \texttt{glmmsel} supports Gaussian and Bernoulli likelihoods for handling regression and classification tasks. The package's core components are written in \texttt{C++} to facilitate fast computation, with the \texttt{R} code serving as a user-friendly interface. This modular design also makes it straightforward to develop interfaces for other languages (e.g., \texttt{Python} or \texttt{Julia}) since the algorithms need not be reimplemented. \texttt{glmmsel} is available open-source on the \texttt{R} repository \texttt{CRAN} and \texttt{GitHub} at
\begin{center}
\url{https://github.com/ryan-thompson/glmmsel}.
\end{center}

\section{Statistical Guarantee}
\label{sec:statistical}

\subsection{Setup}

We now develop a statistical guarantee on the finite-sample performance of our estimator. Towards this end, we assume the data-generating process is an LMM of the form
\begin{equation}
\label{eq:lmm2}
\mathbf{y}_i=\mathbf{X}_i\boldsymbol{\beta}^0+\boldsymbol{\nu}_i,\quad\boldsymbol{\nu}_i\sim\mathrm{N}\left(\mathbf{0},\sigma^2\mathbf{V}_i(\boldsymbol{\gamma}^0)\right),\quad i=1,\dots,m,
\end{equation}
where the noise variance $\sigma^2>0$ and the covariance matrix $\mathbf{V}_i(\boldsymbol{\gamma})=\mathbf{I}+\mathbf{X}_i\operatorname{diag}(\boldsymbol{\gamma})\mathbf{X}_i^\top$. We remark that model \eqref{eq:lmm2} is a standard and equivalent way of writing model \eqref{eq:lmm}. The predictor matrix $\mathbf{X}$ is assumed to be deterministic (i.e., non-random) and have columns with unit $\ell_2$ norm. The fixed effect coefficients $\boldsymbol{\beta}^0\in\mathbb{R}^p$ and random effect variances $\boldsymbol{\gamma}^0\in\mathbb{R}_+^p$ are assumed to be $s$-sparse with nested support, i.e., $\|\boldsymbol{\beta}^0\|_0\leq s$ and $\operatorname{Supp}(\boldsymbol{\gamma}^0)\subseteq\operatorname{Supp}(\boldsymbol{\beta}^0)$.

Given a sample of data $(\mathbf{y},\mathbf{X})$, we fit model \eqref{eq:lmm2} using the $\ell_0$ regularized estimator
\begin{equation}
\label{eq:boundestimator}
(\hat{\boldsymbol{\beta}},\hat{\boldsymbol{\gamma}})\in\underset{(\boldsymbol{\beta},\boldsymbol{\gamma})\in\mathcal{C}(s)}{\arg\min}\;l(\boldsymbol{\beta},\boldsymbol{\gamma}),
\end{equation}
where the negative log-likelihood 
\begin{equation*}
l(\boldsymbol{\beta},\boldsymbol{\gamma})=\sum_{i=1}^m\log\det\{\sigma^2\mathbf{V}_i(\boldsymbol{\gamma})\}+\frac{1}{\sigma^2}\sum_{i=1}^m(\mathbf{y}_i-\mathbf{X}_i\boldsymbol{\beta})^\top\mathbf{V}_i^{-1}(\boldsymbol{\gamma})(\mathbf{y}_i-\mathbf{X}_i\boldsymbol{\beta})
\end{equation*}
and the constraint set
\begin{equation*}
\mathcal{C}(s)=\left\{(\boldsymbol{\beta},\boldsymbol{\gamma})\in\mathbb{R}^p\times\mathbb{R}_+^p:\|\boldsymbol{\beta}\|_0\leq s,\|\boldsymbol{\beta}\|_\infty\leq\bar{\beta},\|\boldsymbol{\gamma}\|_\infty\leq\bar{\gamma},\operatorname{Supp}(\boldsymbol{\gamma})\subseteq\operatorname{Supp}(\boldsymbol{\beta})\right\}.
\end{equation*}
The estimator \eqref{eq:boundestimator} encodes the $\ell_0$ regularizer as a hard constraint and is hence recognizable as the constrained version of the penalized estimator \eqref{eq:l0lmm2} with $\alpha=1$ (chosen here to reduce the complexity of the analysis). The correspondence between these two versions of the estimator can be understood by the fact that for any value of the regularization parameter $\lambda$ there exists a sparsity level $s$ such that the two estimators produce identical estimates. The box constraints $\|\boldsymbol{\beta}\|_\infty\leq\bar{\beta}$ and $\|\boldsymbol{\gamma}\|_\infty\leq\bar{\gamma}$ are included in the constraint set to facilitate the technical analysis. Here, the upper bound parameters $\bar{\beta}$ and $\bar{\gamma}$ are taken to satisfy $\bar{\beta}\geq\|\boldsymbol{\beta}^0\|_\infty$ and $\bar{\gamma}\geq\|\boldsymbol{\gamma}^0\|_\infty$. Such bounds are commonly assumed when deriving statistical guarantees on sparse estimators \citep[see, e.g.,][]{Dedieu2021,Nguyen2024}.

\subsection{Result}

The object of our analysis is the Kullback--Leibler (KL) divergence between the true distribution $\mathrm{N}(\mathbf{X}\boldsymbol{\beta}^0,\sigma^2\mathbf{V}(\boldsymbol{\gamma}^0))$ and the fitted distribution $\mathrm{N}(\mathbf{X}\hat{\boldsymbol{\beta}},\sigma^2\mathbf{V}(\hat{\boldsymbol{\gamma}}))$ produced by the estimator \eqref{eq:boundestimator}, where $\mathbf{V}(\boldsymbol{\gamma})=\operatorname{bdiag}\left\{\mathbf{V}_1(\boldsymbol{\gamma}),\dots,\mathbf{V}_m(\boldsymbol{\gamma})\right\}$ is the covariance matrix over the full sample. Theorem~\ref{thrm:bound} gives a finite-sample (non-asymptotic) probabilistic upper bound on this divergence, where the notation $\lesssim$ is used to indicate that the bound holds up to a positive universal multiplicative constant independent of the parameters and data.
\begin{theorem}
\label{thrm:bound}
Fix $\delta\in(0,1]$. Let $L(\delta)>0$ and $L>0$ be the upper bounds on the Lipschitz constants of $l(\boldsymbol{\beta},\boldsymbol{\gamma})$ and $\operatorname{E}[l(\boldsymbol{\beta},\boldsymbol{\gamma})]$ guaranteed by Lemmas~\ref{lemma:lipschitz2} and \ref{lemma:lipschitz3} in Appendix~\ref{app:bound}, respectively. Then, with probability at least $1-\delta$, the estimator \eqref{eq:boundestimator} satisfies
\begin{equation}
\label{eq:bound}
\begin{split}
&\operatorname{KL}\left[\mathrm{N}\left(\mathbf{X}\boldsymbol{\beta}^0,\sigma^2\mathbf{V}(\boldsymbol{\gamma}^0)\right)\parallel\mathrm{N}\left(\mathbf{X}\hat{\boldsymbol{\beta}},\sigma^2\mathbf{V}(\hat{\boldsymbol{\gamma}})\right)\right] \\
&\hspace{2in}\lesssim\max\left\{(1+\bar{\gamma}s)\max\left(\sqrt{\ndot\Pi},\Pi\right),\frac{s\bar{\beta}}{\sigma}\sqrt{(1+\bar{\gamma}s)\Pi}\right\},
\end{split}
\end{equation}
where
\begin{equation*}
\Pi=s\log\left(\frac{p}{s}\right)+\max\left(L(\delta),L\right)s\sqrt{s}(\bar{\beta}+\bar{\gamma})+\log\left(\frac{1}{\delta}\right).
\end{equation*}
\end{theorem}
\begin{proof}
See Appendix~\ref{app:bound}.
\end{proof}
The quantities appearing on the right-hand side not previously described can be understood as follows. First, the term $s\log(p/s)$ represents the cost of searching for the best subset of predictors, indicating that the bound grows logarithmically with the number of candidate predictors. Meanwhile, the term $\max(L(\delta),L)$ is from the maximal Lipschitz constants $L(\delta)$ and $L$ of the negative log-likelihood and its expectation, respectively (exact forms of these terms are detailed in Appendix~\ref{app:bound}). Finally, the term $\log(1/\delta)$ relates the tightness of the bound to the probability it is satisfied, which is guaranteed to be at least $1-\delta$.

To our knowledge, no comparable guarantees are available for existing sparse LMM estimators, whether $\ell_0$ regularized or otherwise. \citet{Schelldorfer2011} provided a finite-sample bound on the KL divergence of an $\ell_1$ regularized LMM estimator, but their estimator does not select random effects. \citet{Raskutti2011} derived a now classic result for the $\ell_0$ regularized least squares estimator in the non-mixed LM setting. Rather than control the KL divergence, their bounds control the estimation and prediction errors of the estimator but nonetheless feature the same logarithmic scaling factor $s\log(p/s)$ that is in our bound.

Against this backdrop, Theorem~\ref{thrm:bound} provides a statistical justification for the hierarchical estimator proposed in Section~\ref{sec:computation}, at least in its constrained $\alpha=1$ form, which retains the same nested support condition. The bound shows that, under hierarchical sparsity, the dependence on the total number of candidate predictors is only logarithmic through the term $s\log(p/s)$. Our estimator should therefore be most effective in high-dimensional settings with a genuinely sparse hierarchical signal. The experiments in the next section examine performance in finite samples and assess robustness beyond the exact assumptions of the theorem.

\section{Synthetic Datasets}
\label{sec:synthetic}

\subsection{Simulation Design}

We generate synthetic datasets with the response distributed as Gaussian or Bernoulli. For the Gaussian case, the response is simulated according to the data-generating process
\begin{equation*}
y_{ij}\sim \mathrm{N}(\mathbf{x}_{ij}^\top(\boldsymbol{\beta}+\mathbf{u}_i),\sigma^2),\quad j=1,\dots,n_i,\quad i=1,\dots,m,
\end{equation*}
where the noise variance $\sigma^2>0$ is set such that the signal-to-noise ratio is one. Similarly, for the Bernoulli case, the response is simulated according to the data-generating process
\begin{equation*}
y_{ij}\sim\mathrm{Bern}\left(\frac{1}{1+\exp\{-\mathbf{x}_{ij}^\top(\boldsymbol{\beta}+\mathbf{u}_i)\}}\right),\quad j=1,\dots,n_i,\quad i=1,\dots,m.
\end{equation*}
The total sample size $\ndot=\sum_{i=1}^mn_i$ is swept over a grid of seven values equidistant between $\ndot=30$ and $\ndot=1,000$ on the logarithmic scale. Each sample is randomly assigned to one of $m=\lceil\ndot/10\rceil$ clusters so the cluster sample sizes $n_1,\dots,n_m$ vary with an average size of 10.

For both types of response, the predictors $\mathbf{x}_{ij}\in\mathbb{R}^p$ are generated as $\mathrm{N}(\mathbf{0},\boldsymbol{\Sigma})$, where $\boldsymbol{\Sigma}$ has elements $\Sigma_{ij}=\rho^{|i-j|}$ for $i,j=1,\dots,p$. Since our interest lies in the sparse regime, we only specify a subset of these predictors as relevant. Specifically, for the fixed effects, we randomly sample a set $\mathcal{A}_1\subset\{1,\dots,p\}$ of relevant indices. We then take the fixed effects coefficients corresponding to $\mathcal{A}_1$ to be one and all other coefficients to be zero, i.e., $\boldsymbol{\beta}_{\mathcal{A}_1}=\mathbf{1}$ and $\boldsymbol{\beta}_{\mathcal{A}_1^c}=\mathbf{0}$. Following the principle that fixed effects precede random effects, we sample a second subset $\mathcal{A}_2\subset\mathcal{A}_1$ indexing the predictors with random effects. The random effects coefficient vector $\mathbf{u}_i\sim\mathrm{N}(\mathbf{0},\operatorname{diag}(\boldsymbol{\gamma}))$ then has its variances configured as $\boldsymbol{\gamma}_{\mathcal{A}_2}=\mathbf{1}$ and $\boldsymbol{\gamma}_{\mathcal{A}_2^c}=\mathbf{0}$.

\subsection{Baselines and Metrics}

Our implementation \texttt{glmmsel} is benchmarked against several state-of-the-art baselines: \texttt{pysr3} \citep{Sholokhov2024}, \texttt{L0Learn} \citep{Hazimeh2020}, \texttt{glmnet} \citep{Friedman2010}, and \texttt{ncvreg} \citep{Breheny2011}. \texttt{L0Learn}, \texttt{glmnet}, and \texttt{ncvreg} are the best available toolkits for learning non-mixed sparse LMs. Meanwhile, \texttt{pysr3} is the only existing toolkit for learning sparse LMMs that is capable of scaling beyond tens of predictors. For completeness, Appendix~\ref{app:small} includes \texttt{glmmPen} \citep{Heiling2023} and \texttt{rpql} \citep{Hui2017b} as additional baselines in small-scale experiments. All sparsity parameters are tuned using a separate validation set generated independently and identically to the training set. Appendix~\ref{app:implementation} details the regularizer of each method and the candidate tuning parameter values considered.

We evaluate the competing methods according to several metrics that characterize different qualities of the estimated model. As a measure of prediction performance, we report the squared error of the fitted linear predictor relative to the null model:
\begin{equation*}
\text{Prediction error}:=\frac{\sum_{i=1}^m\|\mathbf{X}_i(\boldsymbol{\beta}+\mathbf{u}_i)-\mathbf{X}_i(\hat{\boldsymbol{\beta}}+\hat{\mathbf{u}}_i)\|_2^2}{\sum_{i=1}^m\|\mathbf{X}_i(\boldsymbol{\beta}+\mathbf{u}_i)\|_2^2},
\end{equation*}
where $\hat{\boldsymbol{\beta}}$ and $\hat{\mathbf{u}}_1,\dots,\hat{\mathbf{u}}_m$ are estimates of the fixed and random effect coefficients. Here, the null model is the model with all fixed and random effects equal to zero. Thus, a value of zero indicates perfect prediction accuracy, and a value of one indicates no improvement over the null model. This same linear predictor-based definition is used for both Gaussian and Bernoulli responses. As a measure of interpretability, we report the sparsity level (i.e., the number of predictors selected to have nonzero fixed effects or nonzero random effects):
\begin{equation*}
\text{Sparsity}:=\|\hat{\boldsymbol{\beta}}+\hat{\mathbf{u}}_1\|_0.
\end{equation*}
Lower values are desirable from an interpretation standpoint, with a value of five reflecting the correct sparsity level. As a measure of selection performance, we report the F1 score:
\begin{equation*}
\text{F1 score}:=\frac{2\mathrm{TP}}{2\mathrm{TP}+\mathrm{FP}+\mathrm{FN}},
\end{equation*}
where $\mathrm{TP}$ is the number of true positives, $\mathrm{FP}$ is the number of false positives, and $\mathrm{FN}$ is the number of false negatives. We consider two different types of F1 score. The first is an F1 score of the effect types, under which a predictor is counted as correct only if it is selected with the correct role, fixed or random. The second is an F1 score of the nonzeros, which ignores this distinction and counts a predictor as correct whenever it is selected at all. Hence, the effect type F1 score is the more stringent measure and is better suited to assessing whether a method correctly distinguishes fixed effects from random effects.

\subsection{Results}

Figures~\ref{fig:gaussian-1000-5} and \ref{fig:bernoulli-1000-5} report results for $p=1,000$ predictors with Gaussian and Bernoulli responses, respectively, in a sparse setting with 5 nonzero fixed effects and 3 nonzero random effects.
\begin{figure}[t]
\centering
\input{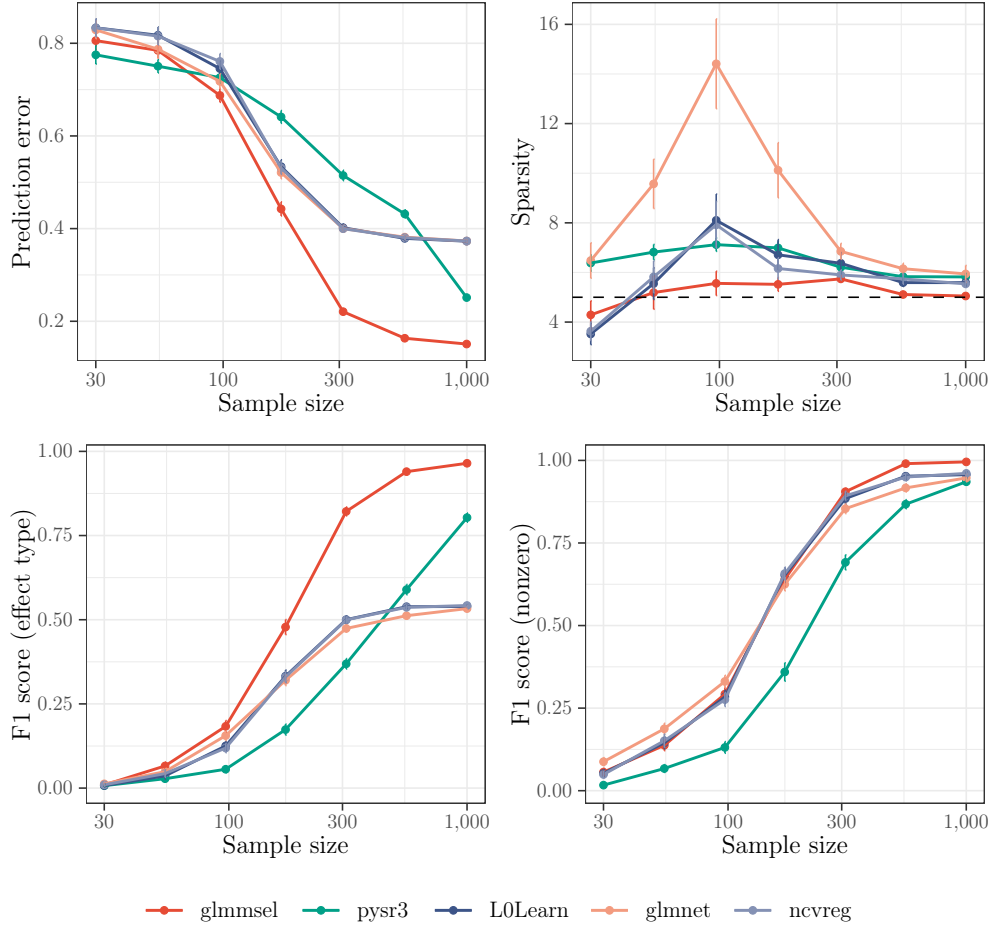}
\caption{Comparisons on synthetic data with Gaussian response. The number of predictors $p=1,000$ with 5 nonzero fixed effects and 3 nonzero random effects. The correlation level $\rho=0.5$. The averages (points) and standard errors (error bars) are measured over 100 datasets. In the upper right panel, the dashed horizontal line indicates the true number of predictors with nonzero effects.}
\label{fig:gaussian-1000-5}
\end{figure}
\begin{figure}[t]
\centering
\input{Figures/bernoulli-1000-5.tex}
\caption{Comparisons on synthetic data with Bernoulli response. The number of predictors $p=1,000$ with 5 nonzero fixed effects and 3 nonzero random effects. The correlation level $\rho=0.5$. The averages (points) and standard errors (error bars) are measured over 100 datasets. In the upper right panel, the dashed horizontal line indicates the true number of predictors with nonzero effects.}
\label{fig:bernoulli-1000-5}
\end{figure}
In both cases, \texttt{glmmsel} demonstrates excellent performance in prediction, interpretation, and selection. At smaller sample sizes, where accurately estimating random effects is inherently challenging, \texttt{glmmsel}'s prediction error is comparable to that of the sparse LM methods \texttt{L0Learn}, \texttt{glmnet}, and \texttt{ncvreg}. However, as the sample size increases, \texttt{glmmsel} delivers substantially better predictions. Moreover, it accurately recovers the true predictors and, unlike the sparse LM methods, correctly distinguishes between fixed and random effects. In the Gaussian case, when $\ndot\approx1,000$, \texttt{pysr3} emerges as the next best performer, though a substantial gap remains in both its prediction and selection performance compared to \texttt{glmmsel}. Notably, \texttt{pysr3} is incompatible with non-Gaussian responses, making \texttt{glmmsel} the first tool capable of learning sparse GLMMs on data of this dimensionality.

Next, we consider a particularly challenging, ultra-high-dimensional setting with $p=10,000$ predictors, an order of magnitude more than in the previous experiments. Figure~\ref{fig:gaussian-10000-5} reports the results.
\begin{figure}[t]
\centering
\input{Figures/gaussian-10000-5.tex}
\caption{Comparisons on synthetic data with Gaussian response. The number of predictors $p=10,000$ with 5 nonzero fixed effects and 3 nonzero random effects. The correlation level $\rho=0.5$. The averages (points) and standard errors (error bars) are measured over 100 datasets. In the upper right panel, the dashed horizontal line indicates the true number of predictors with nonzero effects.}
\label{fig:gaussian-10000-5}
\end{figure}
Here, \texttt{glmmsel} is the only sparse LMM tool that can practically scale to this many predictors. \texttt{pysr3}, which was able to handle $p=1,000$, is unable to scale to this experiment. Otherwise, the story mirrors that of the $p=1,000$ case: \texttt{glmmsel} delivers strong performance and eventually dominates all competing methods as the sample size increases.

As a second challenging setting, we consider predictors with high correlation levels of $\rho=0.9$. To illustrate the advantage of local search in such settings, we add \texttt{glmmsel}'s Algorithm~\ref{alg:ls} to the experiments. We also run \texttt{L0Learn} with its local search algorithm. Figure~\ref{fig:gaussian-high-corr-1000-5} reports the results.
\begin{figure}[t]
\centering
\input{Figures/gaussian-high-corr-1000-5.tex}
\caption{Comparisons on synthetic data with Gaussian response. The number of predictors $p=1,000$ with 5 nonzero fixed effects and 3 nonzero random effects. The correlation level $\rho=0.9$. The averages (points) and standard errors (error bars) are measured over 100 datasets. In the upper right panel, the dashed horizontal line indicates the true number of predictors with nonzero effects.}
\label{fig:gaussian-high-corr-1000-5}
\end{figure}
Under Algorithm~\ref{alg:ls}, \texttt{glmmsel} demonstrates a clear edge: it maintains superior predictive accuracy relative to all baselines and excels at identifying the true effects, despite the complexity of navigating the highly correlated predictor space.

Additional results are provided in Appendices~\ref{app:denser} and \ref{app:sensitivity}. Appendix~\ref{app:denser} shows that the main conclusions remain stable in denser settings with more active predictors, while Appendix~\ref{app:sensitivity} shows that additional tuning of the concavity/shrinkage parameter $\gamma$ for \texttt{glmnet} and \texttt{ncvreg} can improve prediction error but yields much denser models and weaker selection performance.

\section{Real Datasets}
\label{sec:riboflavin}

The \texttt{riboflavin} dataset from \citet{Schelldorfer2011} contains data on riboflavin production by Bacillus subtilis, a bacterium used in industrial fermentation. Besides a continuous response variable measuring the log-transformed production rate in bacterium samples, the dataset comprises $p=4,088$ predictor variables in the form of bacterium gene expression levels. The task is to discover the genes relevant to riboflavin production, with the ultimate goal of genetically modifying the bacterium to enhance output in industrial settings. To support this analysis, the dataset includes observations on $m=28$ bacterium samples, with each sample measured $n_i\in\{2,\dots,6\}$ times, resulting in a total of $\ndot=111$ observations.

We randomly split the dataset into training, validation, and testing sets in 0.70-0.15-0.15 proportions. Table~\ref{tab:riboflavin} reports the results from 100 such splits.
\begin{table}[ht]
\centering
\small

\begin{tabular}{lrrrr}
\toprule
 &  & \multicolumn{3}{c}{Sparsity} \\ 
\cmidrule(lr){3-5}
 & Prediction error & Total & Fixed & Random \\ 
\midrule
\texttt{glmmsel} & 0.304 (0.016) & 11.7 (1.1) & 10.8 (1.1) & 0.9 (0.1) \\ 
\texttt{L0Learn} & 0.359 (0.015) & 7.6 (1.2) & 7.6 (1.2) & 0.0 (0.0) \\ 
\texttt{glmnet} & 0.309 (0.015) & 23.6 (2.0) & 23.6 (2.0) & 0.0 (0.0) \\ 
\texttt{ncvreg} & 0.401 (0.018) & 5.3 (0.5) & 5.3 (0.5) & 0.0 (0.0) \\ 
\bottomrule
\end{tabular}

\caption{Comparisons on the \texttt{riboflavin} dataset. The averages and standard errors (parentheses) are measured over 100 splits of the data. Prediction error is the squared error loss of the fitted linear predictor relative to the fixed intercept-only model.}
\label{tab:riboflavin}
\end{table}
The LMMs fit by \texttt{glmmsel} achieve the lowest prediction loss among all methods. \texttt{L0Learn} yields sparser LMs, but because they cannot capture heterogeneity across samples, their prediction loss is notably higher. The LMs fit by \texttt{glmnet} exhibit statistically comparable prediction loss to the LMMs fit by \texttt{glmmsel}, but contain more than twice as many genes on average. Meanwhile, \texttt{ncvreg} achieves the highest sparsity, selecting fewer genes on average than all other methods considered, at the cost of substantially worse prediction performance. Taken together, these results show that \texttt{glmmsel} achieves the most favorable prediction--sparsity trade-off, attaining the lowest prediction loss using 10.8 fixed effects and only 0.9 random effects on average. The fact that so few random effects are selected indicates that heterogeneity in the \texttt{riboflavin} data is present but sparse, with most of the signal explained by genes with shared effects across samples and only a very small number requiring sample-specific deviations. Ultimately, \texttt{glmmsel} is the only method capable of capturing heterogeneity while scaling effectively to datasets of this dimension.

As a second real dataset, Appendix~\ref{app:mashable} reports the performance of \texttt{glmmsel} on the classification of popularity of news articles posted to the online news platform Mashable.

\section{Concluding Remarks}
\label{sec:concluding}

LMMs are key tools for modeling heterogeneous patterns encountered in contemporary data-analytic problems, such as those that arise in personalized medicine. These problems increasingly involve datasets with a high dimensionality of predictors, making sparsity a desirable (even necessary) quality from the standpoints of prediction and interpretation. This paper introduces a new $\ell_0$ regularized estimator for subset selection in LMMs. We demonstrate that coordinate descent, coupled with a low-complexity local search routine, can scale our estimator to wide datasets well beyond the reach of existing methods. These algorithms also readily extend to subset selection in GLMMs via a PQL approximation. Besides scalability, the new estimator is shown to have excellent statistical properties in theory, simulation, and practice. Our \texttt{R} package \texttt{glmmsel} is available on \texttt{CRAN} and \texttt{GitHub}.

\section*{Acknowledgments}

The authors acknowledge financial support from the Australian Research Council under Discovery Project DP230101179.

\appendix

\section{Proof of Proposition~\ref{prop:update}}
\label{app:update}

The proof is delivered in three parts. First, we derive the form of the thresholding function. Second, we derive the gradient for $\beta_k$. Third, we derive the gradient for $\gamma_k$.

\paragraph{Thresholding Function}

Observe that \eqref{eq:bdlemma}, the upper bound on the negative log-likelihood, treated as a function of block $k$, satisfies
\begin{equation*}
\begin{split}
\bar{l}_{\bar{L}_k}(\boldsymbol{\beta},\boldsymbol{\gamma};\tilde{\boldsymbol{\beta}},\tilde{\boldsymbol{\gamma}})&=l(\tilde{\boldsymbol{\beta}},\tilde{\boldsymbol{\gamma}})+\nabla l(\tilde{\beta}_k)^\top(\beta_k-\tilde{\beta}_k)+\nabla l(\tilde{\gamma}_k)^\top(\gamma_k-\tilde{\gamma}_k) \\
&\hspace{2.5in}+\frac{\bar{L}_k}{2}\left\{(\beta_k-\tilde{\beta}_k)^2+(\gamma_k-\tilde{\gamma}_k)^2\right\} \\
&=\frac{\bar{L}_k}{2}\biggl\{\beta_k-\underbrace{\left(\tilde{\beta}_k-\frac{1}{\bar{L}_k}\nabla l(\tilde{\beta}_k)\right)}_{\hat{\beta}_k}\biggr\}^2+\frac{\bar{L}_k}{2}\biggl\{\gamma_k-\underbrace{\left(\tilde{\gamma}_k-\frac{1}{\bar{L}_k}\nabla l(\tilde{\gamma}_k)\right)}_{\hat{\gamma}_k}\biggr\}^2+\mathrm{const.} \\
&=\frac{\bar{L}_k}{2}(\beta_k-\hat{\beta}_k)^2+\frac{\bar{L}_k}{2}(\gamma_k-\hat{\gamma}_k)^2+\mathrm{const.},
\end{split}
\end{equation*}
where $\mathrm{const.}$ represents terms independent of $\beta_k$ and $\gamma_k$. It immediately follows that \eqref{eq:obj}, the upper bound on the objective function, satisfies
\begin{equation*}
\bar{f}_{\bar{L}_k}(\boldsymbol{\beta},\boldsymbol{\gamma};\tilde{\boldsymbol{\beta}},\tilde{\boldsymbol{\gamma}})=\frac{\bar{L}_k}{2}(\beta_k-\hat{\beta}_k)^2+\frac{\bar{L}_k}{2}(\gamma_k-\hat{\gamma}_k)^2+\lambda\alpha1\{\beta_k\neq0\}+\lambda(1-\alpha)1\{\gamma_k\neq0\}+\mathrm{const.}
\end{equation*}
Hence, to minimize \eqref{eq:obj} as a function of block $k$, we need only minimize the right-hand side of the above function:
\begin{equation*}
\min_{\substack{\beta\in\mathbb{R},\gamma\in\mathbb{R}_+ \\ \beta=0\Rightarrow\gamma=0}}\;\frac{\bar{L}}{2}(\beta-\hat{\beta})^2+\frac{\bar{L}}{2}(\gamma-\hat{\gamma})^2+\lambda\alpha1\{\beta\neq0\}+\lambda(1-\alpha)1\{\gamma\neq0\}
\end{equation*}
or equivalently
\begin{equation*}
\min_{\substack{\beta\in\mathbb{R},\gamma\in\mathbb{R} \\ \beta=0\Rightarrow\gamma=0}}\;\frac{\bar{L}}{2}(\beta-\hat{\beta})^2+\frac{\bar{L}}{2}(\gamma-\hat{\gamma}_+)^2+\lambda\alpha1\{\beta\neq0\}+\lambda(1-\alpha)1\{\gamma\neq0\},
\end{equation*}
where $\hat{\gamma}_+=\max(\hat{\gamma},0)$ and we drop the subscript $k$ to simplify notation. The hierarchy constraint gives rise to three possible cases for the optimal solutions $\beta^\star$ and $\gamma^\star$:
\begin{enumerate}
\item $\beta^\star=0$ and $\gamma^\star=0$, so the optimal objective value is $f_1^\star=\bar{L}/2\hat{\beta}^2+\bar{L}/2\hat{\gamma}_+^2$;
\item $\beta^\star=\hat{\beta}$ and $\gamma^\star=0$, so the optimal objective value is $f_2^\star=\bar{L}/2\hat{\gamma}_+^2+\lambda\alpha$; and
\item $\beta^\star=\hat{\beta}$ and $\gamma^\star=\hat{\gamma}_+$, so the optimal objective value is $f_3^\star=\lambda$.
\end{enumerate}
It is straightforward to check that
\begin{enumerate}
\item $f_1^\star$ is less than $f_2^\star$ and $f_3^\star$ if $\hat{\beta}^2<2\lambda\alpha/\bar{L}$ and $\hat{\beta}^2+\hat{\gamma}_+^2<2\lambda/\bar{L}$;
\item $f_2^\star$ is less than $f_1^\star$ and $f_3^\star$ if $\hat{\beta}^2>2\lambda\alpha/\bar{L}$ and $\hat{\gamma}_+^2<2\lambda(1-\alpha)/\bar{L}$; and
\item $f_3^\star$ is less than $f_1^\star$ and $f_2^\star$ if $\hat{\beta}^2+\hat{\gamma}_+^2>2\lambda/\bar{L}$ and $\hat{\gamma}_+^2>2\lambda(1-\alpha)/\bar{L}$.
\end{enumerate}
Combining these conditions, we have
\begin{equation*}
(\beta^\star,\gamma^\star)=
\begin{dcases}
(0,0)&\text{if }\hat{\beta}^2<\frac{2\lambda\alpha}{\bar{L}}\text{ and }\hat{\beta}^2+\hat{\gamma}_+^2<\frac{2\lambda}{\bar{L}} \\
(\hat{\beta},0)&\text{if }\hat{\gamma}_+^2<\frac{2\lambda(1-\alpha)}{\bar{L}} \\
(\hat{\beta},\hat{\gamma}_+)&\text{otherwise}.
\end{dcases}
\end{equation*}

\paragraph{Gradient for $\beta_k$}

Consider the gradient of $l$ with respect to $\beta_k$. Since $\mathbf{V}_i(\boldsymbol{\gamma})$ does not depend on $\beta_k$, we have
\begin{equation*}
\nabla_{\beta_k}l(\boldsymbol{\beta},\boldsymbol{\gamma})=\ndot\nabla_{\beta_k}\log\{S(\boldsymbol{\beta},\boldsymbol{\gamma})\},
\end{equation*}
where
\begin{equation*}
S(\boldsymbol{\beta},\boldsymbol{\gamma})=\sum_{i=1}^m(\mathbf{y}_i-\mathbf{X}_i\boldsymbol{\beta})^\top\mathbf{V}_i^{-1}(\boldsymbol{\gamma})(\mathbf{y}_i-\mathbf{X}_i\boldsymbol{\beta}).
\end{equation*}
The chain rule gives
\begin{equation*}
\nabla_{\beta_k}\log\{S(\boldsymbol{\beta},\boldsymbol{\gamma})\}=\frac{\nabla_{\beta_k}S(\boldsymbol{\beta},\boldsymbol{\gamma})}{S(\boldsymbol{\beta},\boldsymbol{\gamma})},
\end{equation*}
with the gradient on the right-hand side being
\begin{equation*}
\nabla_{\beta_k}S(\boldsymbol{\beta},\boldsymbol{\gamma})=-2\sum_{i=1}^m\mathbf{x}_{ik}^\top\mathbf{V}_i^{-1}(\boldsymbol{\gamma})(\mathbf{y}_i-\mathbf{X}_i\boldsymbol{\beta}).
\end{equation*}
The gradient claimed by the proposition immediately follows as
\begin{equation*}
\nabla_{\beta_k}l(\boldsymbol{\beta},\boldsymbol{\gamma})=-2\ndot\frac{\sum_{i=1}^m\mathbf{x}_{ik}^\top\mathbf{V}_i^{-1}(\boldsymbol{\gamma})(\mathbf{y}_i-\mathbf{X}_i\boldsymbol{\beta})}{\sum_{i=1}^m(\mathbf{y}_i-\mathbf{X}_i\boldsymbol{\beta})^\top\mathbf{V}_i^{-1}(\boldsymbol{\gamma})(\mathbf{y}_i-\mathbf{X}_i\boldsymbol{\beta})}.
\end{equation*}

\paragraph{Gradient for $\gamma_k$}

Consider now the gradient of $l$ with respect to $\gamma_k$. It holds
\begin{equation}
\label{eq:gammaderiv}
\nabla_{\gamma_k}l(\boldsymbol{\beta},\boldsymbol{\gamma})=\sum_{i=1}^m\nabla_{\gamma_k}\log\det\{\mathbf{V}_i(\boldsymbol{\gamma})\}+\ndot\nabla_{\gamma_k}\log\{S(\boldsymbol{\beta},\boldsymbol{\gamma})\}.
\end{equation}
Using the standard matrix calculus result that $\nabla\log\det\{\mathbf{A}(x)\}=\operatorname{tr}\{\mathbf{A}^{-1}(x)\nabla\mathbf{A}(x)\}$, the first term on the right-hand side of \eqref{eq:gammaderiv} evaluates to
\begin{equation}
\label{eq:gammaderiv1}
\nabla_{\gamma_k}\log\det\{\mathbf{V}_i(\boldsymbol{\gamma})\}=\operatorname{tr}\left(\mathbf{V}_i^{-1}(\boldsymbol{\gamma})\mathbf{x}_{ik}\mathbf{x}_{ik}^\top\right)=\operatorname{tr}\left(\mathbf{x}_{ik}^\top\mathbf{V}_i^{-1}(\boldsymbol{\gamma})\mathbf{x}_{ik}\right)=\mathbf{x}_{ik}^\top\mathbf{V}_i^{-1}(\boldsymbol{\gamma})\mathbf{x}_{ik},
\end{equation}
which follows from the cyclic property of the trace function and the final term being a scalar. For the second term on the right-hand side of \eqref{eq:gammaderiv}, we have
\begin{equation*}
\nabla_{\gamma_k}\log\{S(\boldsymbol{\beta},\boldsymbol{\gamma})\}=\frac{\nabla_{\gamma_k}S(\boldsymbol{\beta},\boldsymbol{\gamma})}{S(\boldsymbol{\beta},\boldsymbol{\gamma})}.
\end{equation*}
Applying the standard matrix calculus results that $\nabla(\mathbf{z}^\top\mathbf{A}^{-1}(x)\mathbf{z})=\mathbf{z}^\top\nabla\mathbf{A}^{-1}(x)\mathbf{z}$ and $\nabla\mathbf{A}^{-1}(x)=-\mathbf{A}^{-1}(x)\nabla\mathbf{A}(x)\mathbf{A}^{-1}(x)$ gives
\begin{equation*}
\begin{split}
\nabla_{\gamma_k} S(\boldsymbol{\beta},\boldsymbol{\gamma})&=-\sum_{i=1}^m(\mathbf{y}_i-\mathbf{X}_i\boldsymbol{\beta})^\top\mathbf{V}_i^{-1}(\boldsymbol{\gamma})\nabla_{\gamma_k}\mathbf{V}_i(\boldsymbol{\gamma})\mathbf{V}_i^{-1}(\boldsymbol{\gamma})(\mathbf{y}_i-\mathbf{X}_i\boldsymbol{\beta}) \\
&=-\sum_{i=1}^m(\mathbf{y}_i-\mathbf{X}_i\boldsymbol{\beta})^\top\mathbf{V}_i^{-1}(\boldsymbol{\gamma})\mathbf{x}_{ik}\mathbf{x}_{ik}^\top\mathbf{V}_i^{-1}(\boldsymbol{\gamma})(\mathbf{y}_i-\mathbf{X}_i\boldsymbol{\beta}) \\
&=-\sum_{i=1}^m\left\{\mathbf{x}_{ik}^\top\mathbf{V}_i^{-1}(\boldsymbol{\gamma})(\mathbf{y}_i-\mathbf{X}_i\boldsymbol{\beta})\right\}^2,
\end{split}
\end{equation*}
where the final equality follows from the fact that $\mathbf{z}^\top(\mathbf{u}\mathbf{u}^\top)\mathbf{z}=(\mathbf{u}^\top\mathbf{z})^2$. We now have
\begin{equation}
\label{eq:gammaderiv2}
\nabla_{\gamma_k}\log\{S(\boldsymbol{\beta},\boldsymbol{\gamma})\}=-\frac{\sum_{i=1}^m\left\{\mathbf{x}_{ik}^\top\mathbf{V}_i^{-1}(\boldsymbol{\gamma})(\mathbf{y}_i-\mathbf{X}_i\boldsymbol{\beta})\right\}^2}{\sum_{i=1}^m(\mathbf{y}_i-\mathbf{X}_i\boldsymbol{\beta})^\top\mathbf{V}_i^{-1}(\boldsymbol{\gamma})(\mathbf{y}_i-\mathbf{X}_i\boldsymbol{\beta})}.
\end{equation}
Finally, substituting \eqref{eq:gammaderiv2} and \eqref{eq:gammaderiv1} into \eqref{eq:gammaderiv} yields the claimed gradient as
\begin{equation*}
\nabla_{\gamma_k}l(\boldsymbol{\beta},\boldsymbol{\gamma})=\sum_{i=1}^m\mathbf{x}_{ik}^\top\mathbf{V}_i^{-1}(\boldsymbol{\gamma})\mathbf{x}_{ik}-\ndot\frac{\sum_{i=1}^m\left\{\mathbf{x}_{ik}^\top\mathbf{V}_i^{-1}(\boldsymbol{\gamma})(\mathbf{y}_i-\mathbf{X}_i\boldsymbol{\beta})\right\}^2}{\sum_{i=1}^m(\mathbf{y}_i-\mathbf{X}_i\boldsymbol{\beta})^\top\mathbf{V}_i^{-1}(\boldsymbol{\gamma})(\mathbf{y}_i-\mathbf{X}_i\boldsymbol{\beta})}.
\end{equation*}

\section{Coordinate Descent Heuristics}
\label{app:heuristics}

Our implementation of Algorithm~\ref{alg:cd} adopts three heuristics that accelerate its convergence.

\paragraph{Gradient Screening}

Prior to running coordinate descent, the predictors are designated strong or weak. Strong predictors are those whose gradients are among (say) the 100 largest in magnitude. The remaining predictors are weak predictors. Coordinate descent is then run on the strong set until convergence. At that point, we do one cycle over the weak predictors. If the active set is unchanged, we stop. Otherwise, if any weak predictors become active, they are shifted from the weak set to the strong set, and coordinate descent is run again.

\paragraph{Gradient Sorting}

Before running coordinate descent, the predictors are sorted in descending order according to the magnitude of their gradients. This sorting ensures that predictors with the greatest potential impact are considered first, usually leading to faster convergence.

\paragraph{Active Set Updates}

After several iterations of coordinate descent, the active set usually stabilizes, with subsequent iterations simply polishing the estimates of the nonzero parameters. For this reason, we restrict the coordinate descent updates solely to the active set once it has stabilized (e.g., after three cycles with no changes). After convergence on the active set, one additional cycle through the inactive predictors is performed to confirm overall convergence.

\section{Proof of Theorem~\ref{thrm:converge}}
\label{app:converge}

\subsection{Preliminary Lemmas}

The proof requires two technical lemmas which we state and prove in turn. The first lemma ensures the negative log-likelihood \eqref{eq:ll} is bounded below. The second lemma ensures the gradients \eqref{eq:gradbeta} and \eqref{eq:gradgamma} of the negative log-likelihood are Lipschitz continuous.

\begin{lemma}
\label{lemma:bounded}
Suppose $\mathbf{y}$ does not lie in the column space of $\mathbf{X}$. Then there exists $c\in\mathbb{R}$ such that the negative log-likelihood $l(\boldsymbol{\beta},\boldsymbol{\gamma})$ given in \eqref{eq:ll} satisfies $l(\boldsymbol{\beta},\boldsymbol{\gamma})\geq c$ for all $\boldsymbol{\beta}\in\mathbb{R}^p$ and $\boldsymbol{\gamma}\in\mathbb{R}_+^p$ such that $\|\boldsymbol{\gamma}\|_\infty\leq\bar{\gamma}$.
\end{lemma}
\begin{proof}
Recall that the negative log-likelihood is
\begin{equation}
\label{eq:ll2}
l(\boldsymbol{\beta},\boldsymbol{\gamma})=\sum_{i=1}^m\log\det\{\mathbf{V}_i(\boldsymbol{\gamma})\}+\ndot\log\left\{\sum_{i=1}^m(\mathbf{y}_i-\mathbf{X}_i\boldsymbol{\beta})^\top\mathbf{V}_i^{-1}(\boldsymbol{\gamma})(\mathbf{y}_i-\mathbf{X}_i\boldsymbol{\beta})\right\},
\end{equation}
where
\begin{equation*}
\mathbf{V}_i(\boldsymbol{\gamma})=\mathbf{I}+\mathbf{X}_i\operatorname{diag}(\boldsymbol{\gamma})\mathbf{X}_i^\top.
\end{equation*}
Begin by choosing any $\boldsymbol{\beta}\in\mathbb{R}^p$ and $\boldsymbol{\gamma}\in\mathbb{R}_+^p$ such that $\|\boldsymbol{\gamma}\|_\infty\leq\bar{\gamma}$. Consider the first term $\sum_{i=1}^m\log\det\{\mathbf{V}_i(\boldsymbol{\gamma})\}$ on the right-hand side of \eqref{eq:ll2}. Since $\operatorname{diag}(\boldsymbol{\gamma})$ is positive semi-definite (because $\gamma_k\geq0$ for all $k$), the product $\mathbf{X}_i\operatorname{diag}(\boldsymbol{\gamma})\mathbf{X}_i^\top$ must also be positive semi-definite. Hence, $\mathbf{V}_i(\boldsymbol{\gamma})$ must be positive definite and its minimal eigenvalue $\lambda_\mathrm{min}$ must satisfy
\begin{equation*}
\lambda_\mathrm{min}\{\mathbf{V}_i(\boldsymbol{\gamma})\}=\lambda_\mathrm{min}(\mathbf{I}+\mathbf{X}_i\operatorname{diag}(\boldsymbol{\gamma})\mathbf{X}_i^\top)\geq\lambda_\mathrm{min}(\mathbf{I})=1.
\end{equation*}
It follows immediately from the property of determinants that $\det\{\mathbf{V}_i(\boldsymbol{\gamma})\}\geq1$ and hence
\begin{equation}
\label{eq:lb1}
\sum_{i=1}^m\log\det\{\mathbf{V}_i(\boldsymbol{\gamma})\}\geq0.
\end{equation}
Consider now the second term $\ndot\log\{\sum_{i=1}^m(\mathbf{y}_i-\mathbf{X}_i\boldsymbol{\beta})^\top\mathbf{V}_i^{-1}(\boldsymbol{\gamma})(\mathbf{y}_i-\mathbf{X}_i\boldsymbol{\beta})\}$ on the right-hand side of \eqref{eq:ll2}. The maximal eigenvalue $\lambda_\mathrm{max}$ of $\mathbf{V}_i(\boldsymbol{\gamma})$ satisfies
\begin{equation*}
\lambda_\mathrm{max}\{\mathbf{V}_i(\boldsymbol{\gamma})\}=\lambda_\mathrm{max}(\mathbf{I}+\mathbf{X}_i\operatorname{diag}(\boldsymbol{\gamma})\mathbf{X}_i^\top)\leq1+\bar{\gamma}\lambda_\mathrm{max}(\mathbf{X}_i\mathbf{X}_i^\top)\leq c_1:=1+\bar{\gamma}\lambda_\mathrm{max}(\mathbf{X}\mathbf{X}^\top)>0,
\end{equation*}
and hence the minimal eigenvalue of $\mathbf{V}_i^{-1}(\boldsymbol{\gamma})$ satisfies
\begin{equation*}
\lambda_\mathrm{min}\{\mathbf{V}_i^{-1}(\boldsymbol{\gamma})\}=\frac{1}{\lambda_\mathrm{max}\{\mathbf{V}_i(\boldsymbol{\gamma})\}}\geq\frac{1}{c_1}.
\end{equation*}
It follows from the property of eigenvalues that
\begin{equation*}
\sum_{i=1}^m(\mathbf{y}_i-\mathbf{X}_i\boldsymbol{\beta})^\top\mathbf{V}_i^{-1}(\boldsymbol{\gamma})(\mathbf{y}_i-\mathbf{X}_i\boldsymbol{\beta})\geq \sum_{i=1}^m\frac{1}{c_1}\|\mathbf{y}_i-\mathbf{X}_i\boldsymbol{\beta}\|_2^2=\frac{1}{c_1}\|\mathbf{y}-\mathbf{X}\boldsymbol{\beta}\|_2^2.
\end{equation*}
By assumption, $\mathbf{y}$ is not in the column space spanned by $\mathbf{X}$, so there is a $c_2>0$ such that
\begin{equation*}
\|\mathbf{y}-\mathbf{X}\boldsymbol{\beta}\|_2^2\geq c_2.
\end{equation*}
Thus, we have
\begin{equation*}
\sum_{i=1}^m(\mathbf{y}_i-\mathbf{X}_i\boldsymbol{\beta})^\top\mathbf{V}_i^{-1}(\boldsymbol{\gamma})(\mathbf{y}_i-\mathbf{X}_i\boldsymbol{\beta})\geq\frac{c_2}{c_1},
\end{equation*}
and hence
\begin{equation}
\label{eq:lb2}
\ndot\log\left\{\sum_{i=1}^m(\mathbf{y}_i-\mathbf{X}_i\boldsymbol{\beta})^\top\mathbf{V}_i^{-1}(\boldsymbol{\gamma})(\mathbf{y}_i-\mathbf{X}_i\boldsymbol{\beta})\right\}\geq c:=\ndot\log\left(\frac{c_2}{c_1}\right).
\end{equation}
Combining \eqref{eq:lb1} and \eqref{eq:lb2} yields the result of the lemma.
\end{proof}

\begin{lemma}
\label{lemma:lipschitz}
Suppose $\mathbf{y}$ does not lie in the column space of $\mathbf{X}$ and that $\boldsymbol{\beta}\in\mathbb{R}^p$ and $\boldsymbol{\gamma}\in\mathbb{R}_+^p$ satisfy $\|\boldsymbol{\beta}\|_\infty\leq\bar{\beta}$ and $\|\boldsymbol{\gamma}\|_\infty\leq\bar{\gamma}$. Then the negative log-likelihood gradients \eqref{eq:gradbeta} and \eqref{eq:gradgamma} are Lipschitz continuous in the $k$th block of coordinates.
\end{lemma}
\begin{proof}
To facilitate exposition, define the functions
\begin{equation*}
A(\boldsymbol{\beta},\boldsymbol{\gamma})=\sum_{i=1}^m\mathbf{x}_{ik}^\top\mathbf{V}_i^{-1}(\boldsymbol{\gamma})(\mathbf{y}_i-\mathbf{X}_i\boldsymbol{\beta})
\end{equation*}
and
\begin{equation*}
B(\boldsymbol{\beta},\boldsymbol{\gamma})=\sum_{i=1}^m(\mathbf{y}_i-\mathbf{X}_i\boldsymbol{\beta})^\top\mathbf{V}_i^{-1}(\boldsymbol{\gamma})(\mathbf{y}_i-\mathbf{X}_i\boldsymbol{\beta}).
\end{equation*}
Consider first the gradient \eqref{eq:gradbeta} of the negative log-likelihood with respect to $\beta_k$, which can be written in terms of $A(\boldsymbol{\beta},\boldsymbol{\gamma})$ and $B(\boldsymbol{\beta},\boldsymbol{\gamma})$ as
\begin{equation*}
\nabla_{\beta_k}l(\boldsymbol{\beta},\boldsymbol{\gamma})=-2\ndot\frac{A(\boldsymbol{\beta},\boldsymbol{\gamma})}{B(\boldsymbol{\beta},\boldsymbol{\gamma})}.
\end{equation*}
Our task is to show that $\nabla_{\beta_k}l(\boldsymbol{\beta},\boldsymbol{\gamma})$ is Lipschitz continuous in block $k$, i.e., that there exists a constant $L_{\beta_k}\geq0$ such that
\begin{equation*}
|\nabla_{\beta_k}l(\boldsymbol{\beta}^{(1)},\boldsymbol{\gamma}^{(1)})-\nabla_{\beta_k}l(\boldsymbol{\beta}^{(2)},\boldsymbol{\gamma}^{(2)})|\leq L_{\beta_k}\left(|\beta_k^{(1)}-\beta_k^{(2)}|+|\gamma_k^{(1)}-\gamma_k^{(2)}|\right),
\end{equation*}
for any feasible $\boldsymbol{\beta}^{(1)}$ and $\boldsymbol{\beta}^{(2)}$ and any feasible $\boldsymbol{\gamma}^{(1)}$ and $\boldsymbol{\gamma}^{(2)}$ differing only in coordinate $k$. Hereafter, we use the shorthand $A^{(1)}$ and $B^{(1)}$ for $A(\boldsymbol{\beta}^{(1)},\boldsymbol{\gamma}^{(1)})$ and $B(\boldsymbol{\beta}^{(1)},\boldsymbol{\gamma}^{(1)})$, and similarly for $A^{(2)}$ and $B^{(2)}$. First, observe that
\begin{equation*}
\begin{split}
|\nabla_{\beta_k}l(\boldsymbol{\beta}^{(1)},\boldsymbol{\gamma}^{(1)})-\nabla_{\beta_k}l(\boldsymbol{\beta}^{(2)},\boldsymbol{\gamma}^{(2)})|&=2\ndot\left|\frac{A^{(1)}}{B^{(1)}}-\frac{A^{(2)}}{B^{(2)}}\right| \\
&=2\ndot\left|\frac{A^{(1)}B^{(2)}-A^{(2)}B^{(1)}}{B^{(1)}B^{(2)}}\right| \\
&=2\ndot\left|\frac{B^{(1)}(A^{(1)}-A^{(2)})-A^{(1)}(B^{(1)}-B^{(2)})}{B^{(1)}B^{(2)}}\right| \\
&\leq\frac{2\ndot}{|B^{(1)}B^{(2)}|}\left(|B^{(1)}||A^{(1)}-A^{(2)}|+|A^{(1)}||B^{(1)}-B^{(2)}|\right).
\end{split}
\end{equation*}
By the same arguments used in the proof of Lemma~\ref{lemma:bounded}, there exists a positive constant (call it $c_1$) independent of $\beta_k$ and $\gamma_k$ that lower bounds $B^{(1)}$ and $B^{(2)}$, and hence $1/|B^{(1)}B^{(2)}|\leq 1/c_1^2$. To upper bound $|A^{(1)}|$ and $|B^{(1)}|$, we use that
\begin{equation*}
\|\mathbf{V}_i^{-1}(\boldsymbol{\gamma})\|_2=\lambda_\mathrm{max}\{\mathbf{V}_i^{-1}(\boldsymbol{\gamma})\}=\frac{1}{\lambda_\mathrm{min}\{\mathbf{V}_i(\boldsymbol{\gamma})\}}\leq\frac{1}{1}=1
\end{equation*}
and
\begin{equation*}
\|\mathbf{y}_i-\mathbf{X}_i\boldsymbol{\beta}\|_2\leq\|\mathbf{y}-\mathbf{X}\boldsymbol{\beta}\|_2\leq\|\mathbf{y}\|_2+\|\mathbf{X}\boldsymbol{\beta}\|_2\leq\|\mathbf{y}\|_2+\bar{\beta}\sum_{j=1}^p\|\mathbf{x}_j\|_2\leq c_2:=\|\mathbf{y}\|_2+\bar{\beta}p,
\end{equation*}
since the columns of $\mathbf{X}$ have unit $\ell_2$ norm. It follows from the above two chains of inequalities that
\begin{equation*}
|A^{(1)}|\leq\sum_{i=1}^m\|\mathbf{x}_{ik}\|_2\|\mathbf{V}_i^{-1}(\boldsymbol{\gamma}^{(1)})\|_2\|\mathbf{y}_i-\mathbf{X}_i\boldsymbol{\beta}^{(1)}\|_2\leq c_2\sum_{i=1}^m\|\mathbf{x}_{ik}\|_2\leq c_2\sqrt{m}
\end{equation*}
and
\begin{equation*}
|B^{(1)}|\leq\sum_{i=1}^m\|\mathbf{y}_i-\mathbf{X}_i\boldsymbol{\beta}^{(1)}\|_2\|\mathbf{V}_i^{-1}(\boldsymbol{\gamma}^{(1)})\|_2\|\mathbf{y}_i-\mathbf{X}_i\boldsymbol{\beta}^{(1)}\|_2\leq c_2^2m.
\end{equation*}
Hence, there exists a constant $c_3\geq0$ such that
\begin{equation*}
|\nabla_{\beta_k}l(\boldsymbol{\beta}^{(1)},\boldsymbol{\gamma}^{(1)})-\nabla_{\beta_k}l(\boldsymbol{\beta}^{(2)},\boldsymbol{\gamma}^{(2)})|\leq c_3\left(|A^{(1)}-A^{(2)}|+|B^{(1)}-B^{(2)}|\right).
\end{equation*}
We now bound $|A^{(1)}-A^{(2)}|$ and $|B^{(1)}-B^{(2)}|$ in terms of $|\beta_k^{(1)}-\beta_k^{(2)}|+|\gamma_k^{(1)}-\gamma_k^{(2)}|$. Towards this end, we use that
\begin{equation*}
\begin{split}
\|\mathbf{V}_i^{-1}(\boldsymbol{\gamma}^{(1)})-\mathbf{V}_i^{-1}(\boldsymbol{\gamma}^{(2)})\|_2&=\left\|\mathbf{V}_i^{-1}(\boldsymbol{\gamma}^{(1)})\left[\mathbf{V}_i(\boldsymbol{\gamma}^{(2)})-\mathbf{V}_i(\boldsymbol{\gamma}^{(1)})\right]\mathbf{V}_i^{-1}(\boldsymbol{\gamma}^{(2)})\right\|_2 \\
&\leq\|\mathbf{V}_i^{-1}(\boldsymbol{\gamma}^{(1)})\|_2\|\mathbf{V}_i(\boldsymbol{\gamma}^{(1)})-\mathbf{V}_i(\boldsymbol{\gamma}^{(2)})\|_2\|\mathbf{V}_i^{-1}(\boldsymbol{\gamma}^{(2)})\|_2 \\
&\leq\|(\mathbf{I}+\mathbf{X}_i\operatorname{diag}(\boldsymbol{\gamma}^{(1)})\mathbf{X}_i^\top)-(\mathbf{I}+\mathbf{X}_i\operatorname{diag}(\boldsymbol{\gamma}^{(2)})\mathbf{X}_i^\top)\|_2 \\
&=\|\mathbf{x}_{ik}\|_2^2|\gamma_k^{(1)}-\gamma_k^{(2)}| \\
&\leq|\gamma_k^{(1)}-\gamma_k^{(2)}|,
\end{split}
\end{equation*}
since $\|\mathbf{x}_{ik}\|_2\leq 1$, and
\begin{equation*}
\|(\mathbf{y}_i-\mathbf{X}_i\boldsymbol{\beta}^{(1)})-(\mathbf{y}_i-\mathbf{X}_i\boldsymbol{\beta}^{(2)})\|_2=\|\mathbf{X}_i(\boldsymbol{\beta}^{(1)}-\boldsymbol{\beta}^{(2)})\|_2=\|\mathbf{x}_{ik}\|_2|\beta_k^{(1)}-\beta_k^{(2)}|\leq|\beta_k^{(1)}-\beta_k^{(2)}|.
\end{equation*}
Using the above two inequalities, we upper bound $|A^{(1)}-A^{(2)}|$ as
\begin{equation*}
\begin{split}
|A^{(1)}-A^{(2)}|&=\left|\sum_{i=1}^m\mathbf{x}_{ik}^\top\mathbf{V}_i^{-1}(\boldsymbol{\gamma}^{(1)})(\mathbf{y}_i-\mathbf{X}_i\boldsymbol{\beta}^{(1)})-\sum_{i=1}^m\mathbf{x}_{ik}^\top\mathbf{V}_i^{-1}(\boldsymbol{\gamma}^{(2)})(\mathbf{y}_i-\mathbf{X}_i\boldsymbol{\beta}^{(2)})\right| \\
&\leq\sum_{i=1}^m\|\mathbf{x}_{ik}\|_2\left\|\mathbf{V}_i^{-1}(\boldsymbol{\gamma}^{(1)})(\mathbf{y}_i-\mathbf{X}_i\boldsymbol{\beta}^{(1)})-\mathbf{V}_i^{-1}(\boldsymbol{\gamma}^{(2)})(\mathbf{y}_i-\mathbf{X}_i\boldsymbol{\beta}^{(2)})\right\|_2 \\
&=\sum_{i=1}^m\|\mathbf{x}_{ik}\|_2\left\|\{\mathbf{V}_i^{-1}(\boldsymbol{\gamma}^{(1)})-\mathbf{V}_i^{-1}(\boldsymbol{\gamma}^{(2)})\}(\mathbf{y}_i-\mathbf{X}_i\boldsymbol{\beta}^{(1)})\right. \\
&\hspace{1.5in}\left.+\mathbf{V}_i^{-1}(\boldsymbol{\gamma}^{(2)})\{(\mathbf{y}_i-\mathbf{X}_i\boldsymbol{\beta}^{(1)})-(\mathbf{y}_i-\mathbf{X}_i\boldsymbol{\beta}^{(2)})\}\right\|_2 \\
&\leq\sum_{i=1}^m\|\mathbf{x}_{ik}\|_2\|\mathbf{V}_i^{-1}(\boldsymbol{\gamma}^{(1)})-\mathbf{V}_i^{-1}(\boldsymbol{\gamma}^{(2)})\|_2\|\mathbf{y}_i-\mathbf{X}_i\boldsymbol{\beta}^{(1)}\|_2 \\
&\hspace{1in}+\sum_{i=1}^m\|\mathbf{x}_{ik}\|_2\|\mathbf{V}_i^{-1}(\boldsymbol{\gamma}^{(2)})\|_2\|(\mathbf{y}_i-\mathbf{X}_i\boldsymbol{\beta}^{(1)})-(\mathbf{y}_i-\mathbf{X}_i\boldsymbol{\beta}^{(2)})\|_2 \\
&\leq c_2\sum_{i=1}^m\|\mathbf{x}_{ik}\|_2|\gamma_k^{(1)}-\gamma_k^{(2)}|+\sum_{i=1}^m\|\mathbf{x}_{ik}\|_2|\beta_k^{(1)}-\beta_k^{(2)}| \\
&\leq c_2\sqrt{m}|\gamma_k^{(1)}-\gamma_k^{(2)}|+\sqrt{m}|\beta_k^{(1)}-\beta_k^{(2)}| \\
&\leq c_4(|\beta_k^{(1)}-\beta_k^{(2)}|+|\gamma_k^{(1)}-\gamma_k^{(2)}|),
\end{split}
\end{equation*}
where $c_4:=(1+c_2)\sqrt{m}$. One can upper bound $|B^{(1)}-B^{(2)}|$ by a similar line of working. Call the constant in that upper bound $c_5\geq0$. It follows that
\begin{equation*}
\begin{split}
|\nabla_{\beta_k}l(\boldsymbol{\beta}^{(1)},\boldsymbol{\gamma}^{(1)})-\nabla_{\beta_k}l(\boldsymbol{\beta}^{(2)},\boldsymbol{\gamma}^{(2)})|&\leq c_3\left(|A^{(1)}-A^{(2)}|+|B^{(1)}-B^{(2)}|\right) \\
&\leq c_3\max(c_4,c_5)\left(|\beta_k^{(1)}-\beta_k^{(2)}|+|\gamma_k^{(1)}-\gamma_k^{(2)}|\right).
\end{split}
\end{equation*}
Hence, $\nabla_{\beta_k}l(\boldsymbol{\beta},\boldsymbol{\gamma})$ is Lipschitz with constant $L_{\beta_k}=c_3\max(c_4,c_5)$. The same line of argument, omitted here for brevity, establishes that $\nabla_{\gamma_k}l(\boldsymbol{\beta},\boldsymbol{\gamma})$ is also Lipschitz with constant $L_{\gamma_k}\geq0$. Finally, the proof is completed by taking $L_k=\max(L_{\beta_k},L_{\gamma_k})$.
\end{proof}

\subsection{Proof of Main Results}

With the preliminary lemmas in hand, we can now prove the main results. The proof is presented in two parts: (i) a proof of convergence of the objective value, and (ii) a proof of convergence of the active set. The steps in our proof follow those used in the proof of Theorem 1 in \citet{Dedieu2021}.

\paragraph{Convergence of Objective Value}

Let $(\boldsymbol{\beta}^\star,\boldsymbol{\gamma}^\star)$ be the result of applying the thresholding function \eqref{eq:threshold} to $(\tilde{\boldsymbol{\beta}},\tilde{\boldsymbol{\gamma}})$. Since $l(\boldsymbol{\beta},\boldsymbol{\gamma})$ is Lipschitz on a compact set (Lemma~\ref{lemma:lipschitz}), we can invoke inequality \eqref{eq:bdlemma}. Substituting $\boldsymbol{\beta}=\boldsymbol{\beta}^\star$ and $\boldsymbol{\gamma}=\boldsymbol{\gamma}^\star$ into \eqref{eq:obj} gives
\begin{equation*}
\begin{split}
f(\boldsymbol{\beta}^\star,\boldsymbol{\gamma}^\star)&\leq\bar{f}_{L_k}(\boldsymbol{\beta}^\star,\boldsymbol{\gamma}^\star;\tilde{\boldsymbol{\beta}},\tilde{\boldsymbol{\gamma}})=l(\tilde{\boldsymbol{\beta}},\tilde{\boldsymbol{\gamma}})+\nabla l(\tilde{\beta}_k)(\beta_k^\star-\tilde{\beta}_k)+\nabla l(\tilde{\gamma}_k)(\gamma_k^\star-\tilde{\gamma}_k) \\
&\hspace{1.2in}+\frac{L_k}{2}\left\{(\beta_k^\star-\tilde{\beta}_k)^2+(\gamma_k^\star-\tilde{\gamma}_k)^2\right\}+\lambda\alpha\|\boldsymbol{\beta}^\star\|_0+\lambda(1-\alpha)\|\boldsymbol{\gamma}^\star\|_0.
\end{split}
\end{equation*}
Next, we add and subtract $\bar{L}_k/2\{(\beta_k^\star-\tilde{\beta}_k)^2+(\gamma_k^\star-\tilde{\gamma}_k)^2\}$ on the right-hand side to get
\begin{equation*}
\begin{split}
f(\boldsymbol{\beta}^\star,\boldsymbol{\gamma}^\star)&\leq l(\tilde{\boldsymbol{\beta}},\tilde{\boldsymbol{\gamma}})+\nabla l(\tilde{\beta}_k)(\beta_k^\star-\tilde{\beta}_k)+\nabla l(\tilde{\gamma}_k)(\gamma_k^\star-\tilde{\gamma}_k) \\
&\hspace{0.5in}+\frac{\bar{L}_k}{2}\left\{(\beta_k^\star-\tilde{\beta}_k)^2+(\gamma_k^\star-\tilde{\gamma}_k)^2\right\}+\frac{L_k-\bar{L}_k}{2}\left\{(\beta_k^\star-\tilde{\beta}_k)^2+(\gamma_k^\star-\tilde{\gamma}_k)^2\right\} \\
&\hspace{1in}+\lambda\alpha\|\boldsymbol{\beta}^\star\|_0+\lambda(1-\alpha)\|\boldsymbol{\gamma}^\star\|_0 \\
&=\bar{f}_{\bar{L}_k}(\boldsymbol{\beta}^\star,\boldsymbol{\gamma}^\star;\tilde{\boldsymbol{\beta}},\tilde{\boldsymbol{\gamma}})+\frac{L_k-\bar{L}_k}{2}\left\{(\beta_k^\star-\tilde{\beta}_k)^2+(\gamma_k^\star-\tilde{\gamma}_k)^2\right\}.
\end{split}
\end{equation*}
Using that $\bar{f}_{\bar{L}_k}(\boldsymbol{\beta}^\star,\boldsymbol{\gamma}^\star;\tilde{\boldsymbol{\beta}},\tilde{\boldsymbol{\gamma}})\leq\bar{f}_{\bar{L}_k}(\tilde{\boldsymbol{\beta}},\tilde{\boldsymbol{\gamma}};\tilde{\boldsymbol{\beta}},\tilde{\boldsymbol{\gamma}})=f(\tilde{\boldsymbol{\beta}},\tilde{\boldsymbol{\gamma}})$ and rearranging terms, we get
\begin{equation*}
f(\tilde{\boldsymbol{\beta}},\tilde{\boldsymbol{\gamma}})-f(\boldsymbol{\beta}^\star,\boldsymbol{\gamma}^\star)\geq\frac{\bar{L}_k-L_k}{2}\left\{(\tilde{\beta}_k-\beta_k^\star)^2+(\tilde{\gamma}_k-\gamma_k^\star)^2\right\}.
\end{equation*}
Finally, taking $\boldsymbol{\beta}^\star=\boldsymbol{\beta}^{(t+1)}$, $\tilde{\boldsymbol{\beta}}=\boldsymbol{\beta}^{(t)}$, $\boldsymbol{\gamma}^\star=\boldsymbol{\gamma}^{(t+1)}$, and $\tilde{\boldsymbol{\gamma}}=\boldsymbol{\gamma}^{(t)}$, and recalling that all $p$ coordinates are updated in a single cycle of coordinate descent (i.e., $p$ block coordinate updates from iteration $t$ to $t+1$), we have
\begin{equation*}
f(\boldsymbol{\beta}^{(t)},\boldsymbol{\gamma}^{(t)})-f(\boldsymbol{\beta}^{(t+1)},\boldsymbol{\gamma}^{(t+1)})\geq\sum_{k=1}^p\frac{\bar{L}_k-L_k}{2}\left\{(\beta_k^{(t)}-\beta_k^{(t+1)})^2+(\gamma_k^{(t)}-\gamma_k^{(t+1)})^2\right\}.
\end{equation*}
The right-hand side is nonnegative for $\bar{L}_k\geq L_k$. Hence, the sequence $\{f(\boldsymbol{\beta}^{(t)},\boldsymbol{\gamma}^{(t)})\}_{t\in\mathbb{N}}$ is decreasing and, because it is bounded below on a compact set (Lemma~\ref{lemma:bounded}), convergent.

\paragraph{Convergence of Active Set}

We proceed using proof by contradiction. Suppose the active set does not converge in finitely many iterations. Choose $t$ such that $\mathcal{A}^{(t)}\neq\mathcal{A}^{(t+1)}$. Then for some $k=1,\dots,p$ either (i) $\beta_k^{(t)}=0$ and $\beta_k^{(t+1)}\neq0$, or (ii) $\beta_k^{(t)}\neq0$ and $\beta_k^{(t+1)}=0$. Consider case (i). By the previous result we have
\begin{equation*}
f(\boldsymbol{\beta}^{(t)},\boldsymbol{\gamma}^{(t)})-f(\boldsymbol{\beta}^{(t+1)},\boldsymbol{\gamma}^{(t+1)})\geq\frac{\bar{L}_k-L_k}{2}(\beta_k^{(t+1)})^2.
\end{equation*}
Because $\beta_k^{(t+1)}$ is the output of the thresholding function, it holds $|\beta_k^{(t+1)}|\geq\sqrt{2\lambda\alpha/\bar{L}_k}$, giving
\begin{equation*}
f(\boldsymbol{\beta}^{(t)},\boldsymbol{\gamma}^{(t)})-f(\boldsymbol{\beta}^{(t+1)},\boldsymbol{\gamma}^{(t+1)})\geq\frac{\bar{L}_k-L_k}{\bar{L}_k}\lambda\alpha,
\end{equation*}
which is positive for $\bar{L}_k>L_k$. The same line of working establishes a similar inequality for case (ii). The quantity on the right-hand side is positive, indicating the objective value strictly decreases. If the active set does not converge then the objective value must decrease indefinitely, which contradicts the fact that the objective value is bounded below on a compact set (Lemma~\ref{lemma:bounded}). Therefore, the active set must converge in finitely many iterations.

\section{Proof of Proposition~\ref{prop:lambda}}
\label{app:lambda}

Let $\lambda^{(r)}$ be the current value of the sparsity parameter. Under the conditions of Theorem~\ref{thrm:converge} the active set $\mathcal{A}^{(r)}$ converges in finitely many iterations, meaning that no active predictors are removed from the model by the thresholding function \eqref{eq:threshold} and no inactive predictors are added. Hence, by the first condition on the inactive predictors in \eqref{eq:threshold}, it must hold for all $k\not\in\mathcal{A}^{(r)}$ (i.e., all inactive predictors) that
\begin{equation*}
\begin{split}
\left(\frac{-\nabla l(\beta_k^{(r)})}{\bar{L}_k}\right)^2&<\frac{2\lambda^{(r)}\alpha}{\bar{L}_k} \\
\nabla l(\beta_k^{(r)})^2&<2\lambda^{(r)}\alpha\bar{L}_k \\
\frac{\nabla l(\beta_k^{(r)})^2}{2\alpha\bar{L}_k}&<\lambda^{(r)}.
\end{split}
\end{equation*}
Likewise, by the second condition on the inactive predictors in \eqref{eq:threshold}, we have
\begin{equation*}
\begin{split}
\left(\frac{-\nabla l(\beta_k^{(r)})}{\bar{L}_k}\right)^2+\left(\frac{\max(-\nabla l(\gamma_k^{(r)}),0)}{\bar{L}_k}\right)^2&<\frac{2\lambda^{(r)}}{\bar{L}_k} \\
\nabla l(\beta_k^{(r)})^2+\max(-\nabla l(\gamma_k^{(r)}),0)^2&<2\lambda^{(r)}\bar{L}_k \\
\frac{\nabla l(\beta_k^{(r)})^2+\max(-\nabla l(\gamma_k^{(r)}),0)^2}{2\bar{L}_k}&<\lambda^{(r)}.
\end{split}
\end{equation*}
Hence, if we take $\lambda=\lambda^{(r+1)}$, where
\begin{equation*}
\lambda^{(r+1)}<\max_{k\not\in\mathcal{A}^{(r)}}\;\max\left(\frac{\nabla l(\beta_k^{(r)})^2}{2\alpha\bar{L}_k},\frac{\nabla l(\beta_k^{(r)})^2+\max(-\nabla l(\gamma_k^{(r)}),0)^2}{2\bar{L}_k}\right),
\end{equation*}
then one of the above conditions will be violated and the active set must change.

\section{Proof of Theorem~\ref{thrm:bound}}
\label{app:bound}

\subsection{Preliminary Lemmas}

The proof of the theorem requires several technical lemmas, which we now state and prove. The first two lemmas provide probabilistic upper bounds on important quantities. The last two lemmas ensure the negative log-likelihood and its expectation are Lipschitz continuous.

\begin{lemma}
\label{lemma:prob1}
Let $(\boldsymbol{\beta},\boldsymbol{\gamma})\in\mathcal{C}(s)$. Then, for any $t>0$, it holds
\begin{equation*}
\begin{split}
&\Pr\left[\left|\frac{1}{\sigma^2}\sum_{i=1}^m\left[\boldsymbol{\nu}_i^\top\mathbf{V}_i^{-1}(\boldsymbol{\gamma})\boldsymbol{\nu}_i-\sigma^2\operatorname{tr}\left\{\mathbf{V}_i^{-1}(\boldsymbol{\gamma})\mathbf{V}_i(\boldsymbol{\gamma}^0)\right\}\right]\right|\geq t\right] \\
&\hspace{2.5in}\leq2\exp\left\{-c\min\left(\frac{t^2}{\ndot(1+\bar{\gamma}s)^2},\frac{t}{1+\bar{\gamma}s}\right)\right\},
\end{split}
\end{equation*}
where $c>0$ is a universal constant.
\end{lemma}

\begin{proof}
Let $\mathbf{h}_i=\sigma^{-1}\mathbf{V}_i^{-\frac{1}{2}}(\boldsymbol{\gamma}^0)\boldsymbol{\nu}_i$ so that $\mathbf{h}_i\sim\mathrm{N}(\mathbf{0},\mathbf{I})$. Additionally, let
\begin{equation*}
\mathbf{A}_i=\mathbf{V}_i^{\frac{1}{2}}(\boldsymbol{\gamma}^0)\mathbf{V}_i^{-1}(\boldsymbol{\gamma})\mathbf{V}_i^{\frac{1}{2}}(\boldsymbol{\gamma}^0).
\end{equation*}
It follows
\begin{equation*}
\frac{1}{\sigma^2}\boldsymbol{\nu}_i^\top\mathbf{V}_i^{-1}(\boldsymbol{\gamma})\boldsymbol{\nu}_i=\mathbf{h}_i^\top\mathbf{A}_i\mathbf{h}_i
\end{equation*}
and
\begin{equation*}
\operatorname{tr}\left\{\mathbf{V}_i^{-1}(\boldsymbol{\gamma})\mathbf{V}_i(\boldsymbol{\gamma}^0)\right\}=\operatorname{tr}(\mathbf{A}_i).
\end{equation*}
For any $\mathbf{h}\sim\mathrm{N}(\mathbf{0},\mathbf{I})$, matrix $\mathbf{A}$, and $t>0$, the Hanson-Wright inequality \citep{Rudelson2013} gives
\begin{equation*}
\Pr\left[|\mathbf{h}^\top\mathbf{A}\mathbf{h}-\operatorname{E}[\mathbf{h}^\top\mathbf{A}\mathbf{h}]|\geq t\right]\leq 2\exp\left\{-c\min\left(\frac{t^2}{\|\mathbf{A}\|_F^2},\frac{t}{\|\mathbf{A}\|_2}\right)\right\},
\end{equation*}
for some universal constant $c>0$. Hence, using that $\operatorname{E}\left[\mathbf{h}^\top\mathbf{A}\mathbf{h}\right]=\operatorname{tr}(\mathbf{A})$ for any $\mathbf{h}\sim\mathrm{N}(\mathbf{0},\mathbf{I})$ and matrix $\mathbf{A}$, we have
\begin{equation*}
\Pr\left[|\mathbf{h}_i^\top\mathbf{A}_i\mathbf{h}_i-\operatorname{tr}(\mathbf{A}_i)|\geq t\right]\leq2\exp\left\{-c\min\left(\frac{t^2}{\|\mathbf{A}_i\|_F^2},\frac{t}{\|\mathbf{A}_i\|_2}\right)\right\}.
\end{equation*}
Now, to take the sum over $i=1,\dots,m$, we use that
\begin{equation*}
\sum_{i=1}^m\left\{\mathbf{h}_i^\top\mathbf{A}_i\mathbf{h}_i-\operatorname{tr}(\mathbf{A}_i)\right\}=
\begin{pmatrix}
\mathbf{h}_1 \\ 
\vdots \\
\mathbf{h}_m
\end{pmatrix}^\top
\begin{pmatrix}
\mathbf{A}_1 & \hdots & \mathbf{0} \\
\vdots & \ddots & \vdots \\
\mathbf{0} & \hdots & \mathbf{A}_m
\end{pmatrix}
\begin{pmatrix}
\mathbf{h}_1 \\ 
\vdots \\
\mathbf{h}_m
\end{pmatrix}
-\operatorname{tr}\left\{
\begin{pmatrix}
\mathbf{A}_1 & \hdots & \mathbf{0} \\
\vdots & \ddots & \vdots \\
\mathbf{0} & \hdots & \mathbf{A}_m
\end{pmatrix}
\right\}
\end{equation*}
with
\begin{equation*}
\left\|
\begin{pmatrix}
\mathbf{A}_1 & \hdots & \mathbf{0} \\
\vdots & \ddots & \vdots \\
\mathbf{0} & \hdots & \mathbf{A}_m
\end{pmatrix}
\right\|_F^2=\sum_{i=1}^m\|\mathbf{A}_i\|_F^2
\end{equation*}
and
\begin{equation*}
\left\|
\begin{pmatrix}
\mathbf{A}_1 & \hdots & \mathbf{0} \\
\vdots & \ddots & \vdots \\
\mathbf{0} & \hdots & \mathbf{A}_m
\end{pmatrix}
\right\|_2=\max_{i=1,\dots,m}\|\mathbf{A}_i\|_2.
\end{equation*}
It immediately follows
\begin{equation*}
\Pr\left[\left|\sum_{i=1}^m\left\{\mathbf{h}_i^\top\mathbf{A}_i\mathbf{h}_i-\operatorname{tr}(\mathbf{A}_i)\right\}\right|\geq t\right]\leq 2\exp\left\{-c\min\left(\frac{t^2}{\sum_{i=1}^m\|\mathbf{A}_i\|_F^2},\frac{t}{\underset{i=1,\dots,m}{\max}\|\mathbf{A}_i\|_2}\right)\right\}.
\end{equation*}
We now upper bound the right-hand side by upper bounding the terms in the denominator on the right-hand side. First, we have
\begin{equation*}
\|\mathbf{A}_i\|_2=\|\mathbf{V}_i^{-1}(\boldsymbol{\gamma})\mathbf{V}_i(\boldsymbol{\gamma}^0)\|_2\leq\|\mathbf{V}_i^{-1}(\boldsymbol{\gamma})\|_2\|\mathbf{V}_i(\boldsymbol{\gamma}^0)\|_2\leq1+\bar{\gamma}s,
\end{equation*}
where we use that
\begin{equation*}
\|\mathbf{V}_i^{-1}(\boldsymbol{\gamma})\|_2=\frac{1}{\lambda_\mathrm{min}\{\mathbf{V}_i(\boldsymbol{\gamma})\}}\leq\frac{1}{1}=1,
\end{equation*}
which follows from the lower bound on $\lambda_\mathrm{min}\{\mathbf{V}_i(\boldsymbol{\gamma})\}$ in the proof of Lemma~\ref{lemma:bounded}, and
\begin{equation*}
\|\mathbf{V}_i(\boldsymbol{\gamma}^0)\|_2\leq1+\sum_{k:\gamma_k^0\neq0}\gamma_k^0\|\mathbf{x}_{ik}\|_2^2\leq1+s\bar{\gamma},
\end{equation*}
since the columns of $\mathbf{X}$ have unit $\ell_2$ norm and $\boldsymbol{\gamma}^0$ has $s$ nonzero components. Next, we use that $\|\mathbf{A}_i\|_F^2\leq n_i\|\mathbf{A}_i\|_2^2$ in tandem with the above result to get
\begin{equation*}
\sum_{i=1}^m\|\mathbf{A}_i\|_F^2\leq\sum_{i=1}^mn_i\|\mathbf{A}_i\|_2^2\leq\sum_{i=1}^mn_i(1+\bar{\gamma}s)^2=\ndot(1+\bar{\gamma}s)^2.
\end{equation*}
Finally, we conclude
\begin{equation*}
\Pr\left[\left|\sum_{i=1}^m\left\{\mathbf{h}_i^\top\mathbf{A}_i\mathbf{h}_i-\operatorname{tr}(\mathbf{A}_i)\right\}\right|\geq t\right]\leq2\exp\left\{-c\min\left(\frac{t^2}{\ndot(1+\bar{\gamma}s)^2},\frac{t}{1+\bar{\gamma}s}\right)\right\}.
\end{equation*}
\end{proof}

\begin{lemma}
\label{lemma:prob2}
Let $(\boldsymbol{\beta},\boldsymbol{\gamma})\in\mathcal{C}(s)$. Then, for any $t>0$, it holds
\begin{equation*}
\Pr\left[\left|\frac{1}{\sigma^2}\sum_{i=1}^m2(\boldsymbol{\beta}^0-\boldsymbol{\beta})^\top\mathbf{X}_i^\top\mathbf{V}_i^{-1}(\boldsymbol{\gamma})\boldsymbol{\nu}_i\right|\geq t\right]\leq2\exp\left(-\frac{t^2}{128\sigma^{-2}(1+\bar{\gamma}s)s^2\bar{\beta}^2}\right).
\end{equation*}
\end{lemma}

\begin{proof}
Let $\mathbf{h}_i=\mathbf{X}_i^\top\mathbf{V}_i^{-1}(\boldsymbol{\gamma})\boldsymbol{\nu}_i$ so that
\begin{equation*}
\frac{1}{\sigma^2}\sum_{i=1}^m2(\boldsymbol{\beta}^0-\boldsymbol{\beta})^\top\mathbf{X}_i^\top\mathbf{V}_i^{-1}(\boldsymbol{\gamma})\boldsymbol{\nu}_i=\frac{1}{\sigma^2}\sum_{i=1}^m2(\boldsymbol{\beta}^0-\boldsymbol{\beta})^\top\mathbf{h}_i.
\end{equation*}
Since $\mathbf{X}_i$ and $\mathbf{V}_i^{-1}(\boldsymbol{\gamma})$ are fixed and $\boldsymbol{\nu}_i\sim\mathrm{N}(\mathbf{0},\sigma^2\mathbf{V}_i(\boldsymbol{\gamma}^0))$, it holds
\begin{equation*}
\mathbf{h}_i\sim\mathrm{N}(\mathbf{0},\sigma^2\mathbf{X}_i^\top\mathbf{V}_i^{-1}(\boldsymbol{\gamma})\mathbf{V}_i(\boldsymbol{\gamma}^0)\mathbf{V}_i^{-1}(\boldsymbol{\gamma})\mathbf{X}_i).
\end{equation*}
Pre-multiplying $\mathbf{h}_i$ by $2\sigma^{-2}(\boldsymbol{\beta}^0-\boldsymbol{\beta})^\top$, which is also a fixed quantity, we get the scalar random variable
\begin{equation*}
\frac{2}{\sigma^2}(\boldsymbol{\beta}^0-\boldsymbol{\beta})^\top\mathbf{h}_i\sim\mathrm{N}(0,\nu_i^2),
\end{equation*}
where
\begin{equation*}
\nu_i^2=\frac{4}{\sigma^2}(\boldsymbol{\beta}^0-\boldsymbol{\beta})^\top\mathbf{X}_i^\top\mathbf{V}_i^{-1}(\boldsymbol{\gamma})\mathbf{V}_i(\boldsymbol{\gamma}^0)\mathbf{V}_i^{-1}(\boldsymbol{\gamma})\mathbf{X}_i(\boldsymbol{\beta}^0-\boldsymbol{\beta}).
\end{equation*}
Since $\boldsymbol{\nu}_i$ and $\boldsymbol{\nu}_j$ are independent for all $i\neq j$, we can sum over $i=1,\dots,m$ to get
\begin{equation*}
\frac{1}{\sigma^2}\sum_{i=1}^m2(\boldsymbol{\beta}^0-\boldsymbol{\beta})^\top\mathbf{h}_i\sim\mathrm{N}\left(0,\sum_{i=1}^m\nu_i^2\right).
\end{equation*}
We can now invoke the standard Gaussian tail bound $\Pr[|X|>t]\leq 2\exp\{-t^2/(2\nu^2)\}$ for $X\sim\mathrm{N}(0,\nu^2)$ and $t>0$ to get
\begin{equation*}
\Pr\left[\left|\frac{1}{\sigma^2}\sum_{i=1}^m2(\boldsymbol{\beta}^0-\boldsymbol{\beta})^\top\mathbf{h}_i\right|\geq t\right]\leq2\exp\left(-\frac{t^2}{2\sum_{i=1}^m\nu_i^2}\right).
\end{equation*}
Now, using that $\|\mathbf{V}_i^{-1}(\boldsymbol{\gamma})\|_2\leq1$ and $\|\mathbf{V}_i(\boldsymbol{\gamma}^0)\|_2\leq1+\bar{\gamma}s$ as in the proof of Lemma~\ref{lemma:prob1}, we can upper bound the right-hand side by upper bounding the variance term appearing in the denominator as
\begin{equation*}
\begin{split}
\sum_{i=1}^m\nu_i^2&=\sum_{i=1}^m\frac{4}{\sigma^2}(\boldsymbol{\beta}^0-\boldsymbol{\beta})^\top\mathbf{X}_i^\top\mathbf{V}_i^{-1}(\boldsymbol{\gamma})\mathbf{V}_i(\boldsymbol{\gamma}^0)\mathbf{V}_i^{-1}(\boldsymbol{\gamma})\mathbf{X}_i(\boldsymbol{\beta}^0-\boldsymbol{\beta}) \\
&\leq\sum_{i=1}^m\frac{4}{\sigma^2}(1+\bar{\gamma}s)(\boldsymbol{\beta}^0-\boldsymbol{\beta})^\top\mathbf{X}_i^\top\mathbf{X}_i(\boldsymbol{\beta}^0-\boldsymbol{\beta}) \\
&=\sum_{i=1}^m\frac{4}{\sigma^2}(1+\bar{\gamma}s)\|\mathbf{X}_i(\boldsymbol{\beta}^0-\boldsymbol{\beta})\|_2^2 \\
&=\frac{4}{\sigma^2}(1+\bar{\gamma}s)\|\mathbf{X}(\boldsymbol{\beta}^0-\boldsymbol{\beta})\|_2^2.
\end{split}
\end{equation*}
Moreover, for any $\boldsymbol{\Delta}$ satisfying $\|\boldsymbol{\Delta}\|_0\leq2s$ and $\|\boldsymbol{\Delta}\|_\infty\leq2\bar{\Delta}$ for some $\bar{\Delta}>0$, we have
\begin{equation*}
\|\mathbf{X}\boldsymbol{\Delta}\|_2=\left\|\sum_{k:\Delta_k\neq0}\Delta_k\mathbf{x}_k\right\|_2\leq\sum_{k:\Delta_k\neq0}|\Delta_k|\|\mathbf{x}_k\|_2\leq\sum_{k:\Delta_k\neq0}|\Delta_k|\leq4s\bar{\Delta}.
\end{equation*}
Taking $\boldsymbol{\Delta}=\boldsymbol{\beta}^0-\boldsymbol{\beta}$, we have $\|\boldsymbol{\Delta}\|_0\leq2s$ and $\|\boldsymbol{\Delta}\|_\infty\leq2\bar{\beta}$, and hence
\begin{equation*}
\frac{4}{\sigma^2}(1+\bar{\gamma}s)\|\mathbf{X}(\boldsymbol{\beta}^0-\boldsymbol{\beta})\|_2^2\leq\frac{64}{\sigma^2}(1+\bar{\gamma}s)s^2\bar{\beta}^2.
\end{equation*}
Finally, we conclude
\begin{equation*}
\Pr\left[\left|\frac{1}{\sigma^2}\sum_{i=1}^m2(\boldsymbol{\beta}^0-\boldsymbol{\beta})^\top\mathbf{h}_i\right|\geq t\right]\leq2\exp\left(-\frac{t^2}{128\sigma^{-2}(1+\bar{\gamma}s)s^2\bar{\beta}^2}\right).
\end{equation*}
\end{proof}

\begin{lemma}
\label{lemma:lipschitz2}
Let $\mathcal{S}_s$ denote the collection of all sets $\mathcal{A}\subseteq\{1,\dots,p\}$ with cardinality $|\mathcal{A}|\leq s$. Fix any $\mathcal{A}\in\mathcal{S}_s$. Define the constraint set
\begin{equation*}
\mathcal{B}(\mathcal{A}):=\left\{(\boldsymbol{\beta},\boldsymbol{\gamma})\in\mathbb{R}^p\times\mathbb{R}_+^p:\|\boldsymbol{\beta}\|_\infty\leq\bar{\beta},\|\boldsymbol{\gamma}\|_\infty\leq\bar{\gamma},\operatorname{Supp}(\boldsymbol{\gamma})\subseteq\operatorname{Supp}(\boldsymbol{\beta})\subseteq\mathcal{A}\right\}
\end{equation*}
and its convex hull
\begin{equation*}
\tilde{\mathcal{B}}(\mathcal{A}):=\left\{(\boldsymbol{\beta},\boldsymbol{\gamma})\in\mathbb{R}^p\times\mathbb{R}_+^p:\|\boldsymbol{\beta}\|_\infty\leq\bar{\beta},\|\boldsymbol{\gamma}\|_\infty\leq\bar{\gamma},\operatorname{Supp}(\boldsymbol{\gamma})\subseteq\mathcal{A},\operatorname{Supp}(\boldsymbol{\beta})\subseteq\mathcal{A}\right\}.
\end{equation*}
Define the random quantity $L_\mathcal{A}(\mathbf{y})$ as
\begin{equation*}
L_\mathcal{A}(\mathbf{y}):=\sup_{(\boldsymbol{\beta},\boldsymbol{\gamma})\in\tilde{\mathcal{B}}(\mathcal{A})}\;\|\nabla_{\beta_\mathcal{A},\gamma_\mathcal{A}}l(\boldsymbol{\beta},\boldsymbol{\gamma})\|_2.
\end{equation*}
Then almost surely $L_\mathcal{A}(\mathbf{y})<\infty$ and on that event $l(\boldsymbol{\beta},\boldsymbol{\gamma})$ is Lipschitz continuous on $\mathcal{B}(\mathcal{A})$, i.e., it holds
\begin{equation*}
|l(\boldsymbol{\beta},\boldsymbol{\gamma})-l(\boldsymbol{\beta}',\boldsymbol{\gamma}')|\leq L_\mathcal{A}(\mathbf{y})\sqrt{\|\boldsymbol{\beta}-\boldsymbol{\beta}'\|_2^2+\|\boldsymbol{\gamma}-\boldsymbol{\gamma}'\|_2^2}
\end{equation*}
for any $(\boldsymbol{\beta},\boldsymbol{\gamma})\in\mathcal{B}(\mathcal{A})$ and $(\boldsymbol{\beta}',\boldsymbol{\gamma}')\in\mathcal{B}(\mathcal{A})$. Moreover, for any $\delta\in(0,1]$, it holds
\begin{equation*}
\Pr\left[\max_{\mathcal{A}\in\mathcal{S}_s}\;L_\mathcal{A}(\mathbf{y})\leq L(\delta)\right]\geq1-\delta,
\end{equation*}
where $L(\delta)$ is a fixed quantity defined as
\begin{equation*}
L(\delta):=\sqrt{2s}\max\left(\frac{2}{\sigma^2}\{2\bar{\beta}s+g(\delta)\},1+\frac{1}{\sigma^2}\{2\bar{\beta}s+g(\delta)\}^2\right)
\end{equation*}
with
\begin{equation*}
g(\delta):=\sigma\sqrt{1+\bar{\gamma}s}\sqrt{\ndot+2\sqrt{\ndot\log\left(\frac{1}{\delta}\right)}+2\log\left(\frac{1}{\delta}\right)}.
\end{equation*}
\end{lemma}

\begin{proof}
We begin by proving the first claim of the lemma that $l$ is Lipschitz with constant $L_\mathcal{A}(\mathbf{y})$. Towards this end, we evaluate the components of $\nabla_{\beta_\mathcal{A},\gamma_{\mathcal{A}}}l(\boldsymbol{\beta},\boldsymbol{\gamma})$. First, rewrite $l(\boldsymbol{\beta},\boldsymbol{\gamma})=l_1(\boldsymbol{\gamma})+l_2(\boldsymbol{\beta},\boldsymbol{\gamma})$, where
\begin{equation*}
l_1(\boldsymbol{\gamma})=\sum_{i=1}^m\log\det\{\sigma^2\mathbf{V}_i(\boldsymbol{\gamma})\}
\end{equation*}
and
\begin{equation*}
l_2(\boldsymbol{\beta},\boldsymbol{\gamma})=\frac{1}{\sigma^2}\sum_{i=1}^m(\mathbf{y}_i-\mathbf{X}_i\boldsymbol{\beta})^\top\mathbf{V}_i^{-1}(\boldsymbol{\gamma})(\mathbf{y}_i-\mathbf{X}_i\boldsymbol{\beta}).
\end{equation*}
For $l_1$, a routine derivation gives the components of the gradient with respect to $\boldsymbol{\gamma}$ as
\begin{equation*}
\nabla_{\gamma_k}l_1(\boldsymbol{\gamma})=\sum_{i=1}^m\mathbf{x}_{ik}^\top\mathbf{V}_i^{-1}(\boldsymbol{\gamma})\mathbf{x}_{ik}.
\end{equation*}
Likewise, for $l_2$, the components of the gradient with respect to $\boldsymbol{\beta}$ and $\boldsymbol{\gamma}$ are respectively given by
\begin{equation*}
\nabla_{\beta_k}l_2(\boldsymbol{\beta},\boldsymbol{\gamma})=-\frac{1}{\sigma^2}\sum_{i=1}^m2\mathbf{x}_{ik}^\top\mathbf{V}_i^{-1}(\boldsymbol{\gamma})(\mathbf{y}_i-\mathbf{X}_i\boldsymbol{\beta})
\end{equation*}
and
\begin{equation*}
\nabla_{\gamma_k}l_2(\boldsymbol{\beta},\boldsymbol{\gamma})=-\frac{1}{\sigma^2}\sum_{i=1}^m\left\{\mathbf{x}_{ik}^\top\mathbf{V}_i^{-1}(\boldsymbol{\gamma})(\mathbf{y}_i-\mathbf{X}_i\boldsymbol{\beta})\right\}^2.
\end{equation*}
All three of the above gradients are continuous and $\tilde{\mathcal{B}}(\mathcal{A})$ is compact. Recall also that $\mathbf{y}\sim\mathrm{N}(\mathbf{X}\boldsymbol{\beta}^0,\sigma^2\mathbf{V}(\boldsymbol{\gamma}^0))$, so all components of $\mathbf{y}$ are almost surely finite. Hence, the gradient norm $\|\nabla_{\beta_\mathcal{A},\gamma_{\mathcal{A}}}l(\boldsymbol{\beta},\boldsymbol{\gamma})\|_2$ attains a maximum and thus
\begin{equation*}
L_\mathcal{A}(\mathbf{y})=\sup_{(\boldsymbol{\beta},\boldsymbol{\gamma})\in\tilde{\mathcal{B}}(\mathcal{A})}\;\|\nabla_{\beta_\mathcal{A},\gamma_\mathcal{A}}l(\boldsymbol{\beta},\boldsymbol{\gamma})\|_2<\infty
\end{equation*}
almost surely. Moreover, since $\tilde{\mathcal{B}}(\mathcal{A})$ is convex, we can conclude by the mean value theorem that
\begin{equation*}
|l(\boldsymbol{\beta},\boldsymbol{\gamma})-l(\boldsymbol{\beta}',\boldsymbol{\gamma}')|\leq L_\mathcal{A}(\mathbf{y})\sqrt{\|\boldsymbol{\beta}-\boldsymbol{\beta}'\|_2^2+\|\boldsymbol{\gamma}-\boldsymbol{\gamma}'\|_2^2}
\end{equation*}
for any $\boldsymbol{\beta}'\neq\boldsymbol{\beta}$ and any $\boldsymbol{\gamma}'\neq\boldsymbol{\gamma}$ in $\tilde{\mathcal{B}}(\mathcal{A})$ almost surely. Furthermore, because $\mathcal{B}(\mathcal{A})\subseteq\tilde{\mathcal{B}}(\mathcal{A})$, the above inequality also holds for any $\boldsymbol{\beta}'\neq\boldsymbol{\beta}$ and any $\boldsymbol{\gamma}'\neq\boldsymbol{\gamma}$ in $\mathcal{B}(\mathcal{A})$, thereby establishing the first claim of the lemma that $l(\boldsymbol{\beta},\boldsymbol{\gamma})$ is almost surely $L_\mathcal{A}(\mathbf{y})$ Lipschitz on $\mathcal{B}(\mathcal{A})$. To establish the second claim of the lemma, we bound the components of $L_\mathcal{A}(\mathbf{y})$. Observe that 
\begin{equation*}
|\nabla_{\gamma_k}l_1(\boldsymbol{\gamma})|=\left|\sum_{i=1}^m\mathbf{x}_{ik}^\top\mathbf{V}_i^{-1}(\boldsymbol{\gamma})\mathbf{x}_{ik}\right|\leq\sum_{i=1}^m\|\mathbf{x}_{ik}\|_2^2\|\mathbf{V}_i^{-1}(\boldsymbol{\gamma})\|_2\leq\sum_{i=1}^m\|\mathbf{x}_{ik}\|_2^2=1,
\end{equation*}
where we use that $\|\mathbf{V}_i^{-1}(\boldsymbol{\gamma})\|_2\leq1$ as in the proof of Lemma~\ref{lemma:prob1} and that the columns of $\mathbf{X}$ have unit $\ell_2$ norm. Next, see that
\begin{equation*}
\begin{split}
|\nabla_{\beta_k}l_2(\boldsymbol{\beta},\boldsymbol{\gamma})|&=\left|-\frac{1}{\sigma^2}\sum_{i=1}^m2\mathbf{x}_{ik}^\top\mathbf{V}_i^{-1}(\boldsymbol{\gamma})(\mathbf{y}_i-\mathbf{X}_i\boldsymbol{\beta})\right| \\
&\leq\frac{1}{\sigma^2}\sum_{i=1}^m2\|\mathbf{x}_{ik}\|_2\|\mathbf{y}_i-\mathbf{X}_i\boldsymbol{\beta}\|_2 \\
&\leq\frac{2}{\sigma^2}\sqrt{\sum_{i=1}^m\|\mathbf{x}_{ik}\|_2^2}\sqrt{\sum_{i=1}^m\|\mathbf{y}_i-\mathbf{X}_i\boldsymbol{\beta}\|_2^2} \\
&=\frac{2}{\sigma^2}\|\mathbf{y}-\mathbf{X}\boldsymbol{\beta}\|_2 \\
&\leq\frac{2}{\sigma^2}(\|\mathbf{y}\|_2+\|\mathbf{X}\boldsymbol{\beta}\|_2) \\
&\leq\frac{2}{\sigma^2}(\|\mathbf{y}\|_2+\bar{\beta}s),
\end{split}
\end{equation*}
where the second inequality follows from the Cauchy-Schwarz inequality and we use that
\begin{equation*}
\|\mathbf{X}\boldsymbol{\beta}\|_2=\|\sum_{k:\beta_k\neq0}\beta_k\mathbf{x}_k\|_2\leq\sum_{k:\beta_k\neq0}|\beta_k|\|\mathbf{x}_k\|_2\leq\sum_{k:\beta_k\neq0}|\beta_k|\leq\bar{\beta}s.
\end{equation*}
Finally, observe that
\begin{equation*}
\begin{split}
|\nabla_{\gamma_k}l_2(\boldsymbol{\beta},\boldsymbol{\gamma})|&=\left|-\frac{1}{\sigma^2}\sum_{i=1}^m\left[\mathbf{x}_{ik}^\top\mathbf{V}_i^{-1}(\boldsymbol{\gamma})(\mathbf{y}_i-\mathbf{X}_i\boldsymbol{\beta})\right]^2\right| \\
&\leq\frac{1}{\sigma^2}\sum_{i=1}^m\|\mathbf{x}_{ik}\|_2^2\|\mathbf{y}_i-\mathbf{X}_i\boldsymbol{\beta}\|_2^2 \\
&=\frac{1}{\sigma^2}\|\mathbf{y}-\mathbf{X}\boldsymbol{\beta}\|_2^2 \\
&\leq\frac{1}{\sigma^2}(\|\mathbf{y}\|_2+\|\mathbf{X}\boldsymbol{\beta}\|_2)^2 \\
&\leq\frac{1}{\sigma^2}(\|\mathbf{y}\|_2+\bar{\beta}s)^2.
\end{split}
\end{equation*}
Now, to bound $\|\mathbf{y}\|_2$, we use the inequality
\begin{equation*}
\|\mathbf{y}\|_2\leq\|\mathbf{X}\boldsymbol{\beta}^0\|_2+\|\boldsymbol{\nu}\|_2=\|\mathbf{X}\boldsymbol{\beta}^0\|_2+\sigma\|\mathbf{V}^{\frac{1}{2}}(\boldsymbol{\gamma}^0)\mathbf{z}\|_2\leq\bar{\beta}s+\sigma\sqrt{\lambda_{\max}\left\{\mathbf{V}(\boldsymbol{\gamma}^0)\right\}}\|\mathbf{z}\|_2,
\end{equation*}
where $\mathbf{z}\sim\mathrm{N}(\mathbf{0},\mathbf{I})$. As in the proof of Lemma~\ref{lemma:prob1}, it holds
\begin{equation*}
\lambda_{\max}\left\{\mathbf{V}(\boldsymbol{\gamma}^0)\right\}=\max_{i=1,\dots,m}\;\lambda_{\max}\left\{\mathbf{V}_i(\boldsymbol{\gamma}^0)\right\}\leq1+\bar{\gamma}s.
\end{equation*}
Since $\mathbf{z}$ has standard normal components, a Laurent-Massart $\chi^2$ tail bound \citep{Laurent2000} gives
\begin{equation*}
\Pr\left[\|\mathbf{z}\|_2^2\geq\ndot+2\sqrt{\ndot t}+2t\right]\leq\exp(-t).
\end{equation*}
Substituting $t=\log(1/\delta)$ and taking square roots gives
\begin{equation*}
\Pr\left[\|\mathbf{z}\|_2\geq\sqrt{\ndot+2\sqrt{\ndot\log\left(\frac{1}{\delta}\right)}+2\log\left(\frac{1}{\delta}\right)}\right]\leq\delta
\end{equation*}
and hence it holds with probability at least $1-\delta$ that
\begin{equation*}
\|\mathbf{y}\|_2\leq \bar{\beta}s+g(\delta),
\end{equation*}
where
\begin{equation*}
g(\delta)=\sigma\sqrt{1+\bar{\gamma}s}\sqrt{\ndot+2\sqrt{\ndot\log\left(\frac{1}{\delta}\right)}+2\log\left(\frac{1}{\delta}\right)}.
\end{equation*}
Substituting the above result into the gradient bounds and taking the maximum yields
\begin{equation*}
L_\mathcal{A}(\mathbf{y})\leq\sqrt{2s}\max\left(\frac{2}{\sigma^2}\{2\bar{\beta}s+g(\delta)\},1+\frac{1}{\sigma^2}\{2\bar{\beta}s+g(\delta)\}^2\right)=L(\delta)
\end{equation*}
with probability at least $1-\delta$. Finally, since $L(\delta)$ does not depend on $\mathcal{A}$ it holds
\begin{equation*}
\Pr\left[\max_{\mathcal{A}\in\mathcal{S}_s}\;L_\mathcal{A}(\mathbf{y})\leq L(\delta)\right]\geq1-\delta.
\end{equation*}
\end{proof}

\begin{lemma}
\label{lemma:lipschitz3}
Let $\mathcal{S}_s$ denote the collection of all sets $\mathcal{A}\subseteq\{1,\dots,p\}$ with cardinality $|\mathcal{A}|\leq s$. Fix any $\mathcal{A}\in\mathcal{S}_s$. Define the constraint set
\begin{equation*}
\mathcal{B}(\mathcal{A}):=\left\{(\boldsymbol{\beta},\boldsymbol{\gamma})\in\mathbb{R}^p\times\mathbb{R}_+^p:\|\boldsymbol{\beta}\|_\infty\leq\bar{\beta},\|\boldsymbol{\gamma}\|_\infty\leq\bar{\gamma},\operatorname{Supp}(\boldsymbol{\gamma})\subseteq\operatorname{Supp}(\boldsymbol{\beta})\subseteq\mathcal{A}\right\}
\end{equation*}
and its convex hull
\begin{equation*}
\tilde{\mathcal{B}}(\mathcal{A}):=\left\{(\boldsymbol{\beta},\boldsymbol{\gamma})\in\mathbb{R}^p\times\mathbb{R}_+^p:\|\boldsymbol{\beta}\|_\infty\leq\bar{\beta},\|\boldsymbol{\gamma}\|_\infty\leq\bar{\gamma},\operatorname{Supp}(\boldsymbol{\gamma})\subseteq\mathcal{A},\operatorname{Supp}(\boldsymbol{\beta})\subseteq\mathcal{A}\right\}.
\end{equation*}
Define the quantity
\begin{equation*}
L_\mathcal{A}:=\sup_{(\boldsymbol{\beta},\boldsymbol{\gamma})\in\tilde{\mathcal{B}}(\mathcal{A})}\;\|\nabla_{\beta_\mathcal{A},\gamma_\mathcal{A}}\operatorname{E}[l(\boldsymbol{\beta},\boldsymbol{\gamma})]\|_2.
\end{equation*}
Then $L_\mathcal{A}<\infty$ and $\operatorname{E}[l(\boldsymbol{\beta},\boldsymbol{\gamma})]$ is Lipschitz continuous on $\mathcal{B}(\mathcal{A})$, i.e., it holds
\begin{equation*}
|\operatorname{E}[l(\boldsymbol{\beta},\boldsymbol{\gamma})]-\operatorname{E}[l(\boldsymbol{\beta}',\boldsymbol{\gamma}')]|\leq L_\mathcal{A}\sqrt{\|\boldsymbol{\beta}-\boldsymbol{\beta}'\|_2^2+\|\boldsymbol{\gamma}-\boldsymbol{\gamma}'\|_2^2}
\end{equation*}
for any $(\boldsymbol{\beta},\boldsymbol{\gamma})\in\mathcal{B}(\mathcal{A})$ and $(\boldsymbol{\beta}',\boldsymbol{\gamma}')\in\mathcal{B}(\mathcal{A})$. Moreover, it holds
\begin{equation*}
\max_{\mathcal{A}\in\mathcal{S}_s}\;L_\mathcal{A}\leq L:=\sqrt{2s}\max\left(\frac{8}{\sigma^2}\bar{\beta}s,2+\frac{16}{\sigma^2}\bar{\beta}^2s^2+\bar{\gamma}s\right).
\end{equation*}
\end{lemma}

\begin{proof}
We start by establishing the first claim of the lemma that the expectation of $l$ is Lipschitz with constant $L_\mathcal{A}$. Towards that end, we consider the components of $\nabla_{\beta_\mathcal{A},\gamma_\mathcal{A}}\operatorname{E}[l(\boldsymbol{\beta},\boldsymbol{\gamma})]$. First, rewrite $\operatorname{E}[l(\boldsymbol{\beta},\boldsymbol{\gamma})]=\operatorname{E}[l_1(\boldsymbol{\gamma})]+\operatorname{E}[l_2(\boldsymbol{\beta},\boldsymbol{\gamma})]$, where
\begin{equation*}
\operatorname{E}[l_1(\boldsymbol{\gamma})]=\operatorname{E}\left[\sum_{i=1}^m\log\det\{\sigma^2\mathbf{V}_i(\boldsymbol{\gamma})\}\right]=\sum_{i=1}^m\log\det\{\sigma^2\mathbf{V}_i(\boldsymbol{\gamma})\}
\end{equation*}
and
\begin{equation*}
\begin{split}
\operatorname{E}[l_2(\boldsymbol{\beta},\boldsymbol{\gamma})]&=\operatorname{E}\left[\frac{1}{\sigma^2}\sum_{i=1}^m(\mathbf{y}_i-\mathbf{X}_i\boldsymbol{\beta})^\top\mathbf{V}_i^{-1}(\boldsymbol{\gamma})(\mathbf{y}_i-\mathbf{X}_i\boldsymbol{\beta})\right] \\
&=\frac{1}{\sigma^2}\sum_{i=1}^m\left[(\boldsymbol{\beta}^0-\boldsymbol{\beta})^\top\mathbf{X}_i^\top\mathbf{V}_i^{-1}(\boldsymbol{\gamma})\mathbf{X}_i(\boldsymbol{\beta}^0-\boldsymbol{\beta})+\sigma^2\operatorname{tr}\left\{\mathbf{V}_i^{-1}(\boldsymbol{\gamma})\mathbf{V}_i(\boldsymbol{\gamma}^0)\right\}\right].
\end{split}
\end{equation*}
For the expectation of $l_1$, some standard calculations give the components of the gradient with respect to $\boldsymbol{\gamma}$ as
\begin{equation*}
\nabla_{\gamma_k}\operatorname{E}[l_1(\boldsymbol{\gamma})]=\sum_{i=1}^m\mathbf{x}_{ik}^\top\mathbf{V}_i^{-1}(\boldsymbol{\gamma})\mathbf{x}_{ik}.
\end{equation*}
Similarly, for the expectation of $l_2$, the components of the gradient with respect to $\boldsymbol{\beta}$ and $\boldsymbol{\gamma}$ are respectively given as
\begin{equation*}
\nabla_{\beta_k}\operatorname{E}[l_2(\boldsymbol{\beta},\boldsymbol{\gamma})]=-\frac{1}{\sigma^2}\sum_{i=1}^m2\mathbf{x}_{ik}^\top\mathbf{V}_i^{-1}(\boldsymbol{\gamma})\mathbf{X}_i(\boldsymbol{\beta}^0-\boldsymbol{\beta})
\end{equation*}
and
\begin{equation*}
\nabla_{\gamma_k}\operatorname{E}[l_2(\boldsymbol{\beta},\boldsymbol{\gamma})]=-\frac{1}{\sigma^2}\sum_{i=1}^m\left\{\mathbf{x}_{ik}^\top\mathbf{V}_i^{-1}(\boldsymbol{\gamma})\mathbf{X}_i(\boldsymbol{\beta}^0-\boldsymbol{\beta})\right\}^2-\sum_{i=1}^m\mathbf{x}_{ik}^\top\mathbf{V}_i^{-1}(\boldsymbol{\gamma})\mathbf{V}_i(\boldsymbol{\gamma}^0)\mathbf{V}_i^{-1}(\boldsymbol{\gamma})\mathbf{x}_{ik}.
\end{equation*}
The above three gradients are continuous and $\tilde{\mathcal{B}}(\mathcal{A})$ is compact. Thus, the gradient norm $\|\nabla_{\beta_\mathcal{A},\gamma_\mathcal{A}}\operatorname{E}[l(\boldsymbol{\beta},\boldsymbol{\gamma})]\|_2$ attains a maximum and so it holds
\begin{equation*}
L_\mathcal{A}=\sup_{(\boldsymbol{\beta},\boldsymbol{\gamma})\in\tilde{\mathcal{B}}(\mathcal{A})}\;\|\nabla_{\beta_\mathcal{A},\gamma_\mathcal{A}}\operatorname{E}[l(\boldsymbol{\beta},\boldsymbol{\gamma})]\|_2<\infty.
\end{equation*}
Furthermore, since $\tilde{\mathcal{B}}(\mathcal{A})$ is convex, we can apply the mean value theorem to get
\begin{equation*}
|\operatorname{E}[l(\boldsymbol{\beta},\boldsymbol{\gamma})]-\operatorname{E}[l(\boldsymbol{\beta}',\boldsymbol{\gamma}')]|\leq L_\mathcal{A}\sqrt{\|\boldsymbol{\beta}-\boldsymbol{\beta}'\|_2^2+\|\boldsymbol{\gamma}-\boldsymbol{\gamma}'\|_2^2}
\end{equation*}
for any $\boldsymbol{\beta}'\neq\boldsymbol{\beta}$ and any $\boldsymbol{\gamma}'\neq\boldsymbol{\gamma}$ in $\tilde{\mathcal{B}}(\mathcal{A})$. Moreover, because $\mathcal{B}(\mathcal{A})\subseteq\tilde{\mathcal{B}}(\mathcal{A})$, the above inequality also holds for any $\boldsymbol{\beta}'\neq\boldsymbol{\beta}$ and any $\boldsymbol{\gamma}'\neq\boldsymbol{\gamma}$ in $\mathcal{B}(\mathcal{A})$, thereby proving the first claim of the lemma that $\operatorname{E}[l(\boldsymbol{\beta},\boldsymbol{\gamma})]$ is $L_\mathcal{A}$ Lipschitz on $\mathcal{B}(\mathcal{A})$. To prove the second claim of the lemma, we bound the components of $L_\mathcal{A}$. See that 
\begin{equation*}
|\nabla_{\gamma_k}\operatorname{E}[l_1(\boldsymbol{\gamma})]|=\left|\sum_{i=1}^m\mathbf{x}_{ik}^\top\mathbf{V}_i^{-1}(\boldsymbol{\gamma})\mathbf{x}_{ik}\right|\leq\sum_{i=1}^m\|\mathbf{x}_{ik}\|_2^2\|\mathbf{V}_i^{-1}(\boldsymbol{\gamma})\|_2\leq\sum_{i=1}^m\|\mathbf{x}_{ik}\|_2^2=1.
\end{equation*}
Next, observe that
\begin{equation*}
\begin{split}
|\nabla_{\beta_k}\operatorname{E}[l_2(\boldsymbol{\beta},\boldsymbol{\gamma})]|&=\left|-\frac{1}{\sigma^2}\sum_{i=1}^m2\mathbf{x}_{ik}^\top\mathbf{V}_i^{-1}(\boldsymbol{\gamma})\mathbf{X}_i(\boldsymbol{\beta}^0-\boldsymbol{\beta})\right| \\
&\leq\frac{1}{\sigma^2}\sum_{i=1}^m2\|\mathbf{x}_{ik}\|_2\|\mathbf{X}_i(\boldsymbol{\beta}^0-\boldsymbol{\beta})\|_2 \\
&\leq\frac{2}{\sigma^2}\sqrt{\sum_{i=1}^m\|\mathbf{x}_{ik}\|_2^2}\sqrt{\sum_{i=1}^m\|\mathbf{X}_i(\boldsymbol{\beta}^0-\boldsymbol{\beta})\|_2^2} \\
&=\frac{2}{\sigma^2}\|\mathbf{X}(\boldsymbol{\beta}^0-\boldsymbol{\beta})\|_2 \\
&\leq\frac{8}{\sigma^2}\bar{\beta}s,
\end{split}
\end{equation*}
where we employ the same tricks used in the previous lemma. Finally, see that
\begin{equation*}
\begin{split}
|\nabla_{\gamma_k}\operatorname{E}[l_2(\boldsymbol{\beta},\boldsymbol{\gamma})]|&=\left|-\frac{1}{\sigma^2}\sum_{i=1}^m\left\{\mathbf{x}_{ik}^\top\mathbf{V}_i^{-1}(\boldsymbol{\gamma})\mathbf{X}_i(\boldsymbol{\beta}^0-\boldsymbol{\beta})\right\}^2\right. \\
&\hspace{2in}\left.-\sum_{i=1}^m\mathbf{x}_{ik}^\top\mathbf{V}_i^{-1}(\boldsymbol{\gamma})\mathbf{V}_i(\boldsymbol{\gamma}^0)\mathbf{V}_i^{-1}(\boldsymbol{\gamma})\mathbf{x}_{ik}\right| \\
&\leq\frac{1}{\sigma^2}\sum_{i=1}^m\|\mathbf{x}_{ik}\|_2^2\|\mathbf{X}_i(\boldsymbol{\beta}^0-\boldsymbol{\beta})\|_2^2+\sum_{i=1}^m\|\mathbf{x}_{ik}\|_2^2\|\mathbf{V}_i(\boldsymbol{\gamma}^0)\|_2 \\
&\leq\frac{1}{\sigma^2}\|\mathbf{X}(\boldsymbol{\beta}^0-\boldsymbol{\beta})\|_2^2+1+\bar{\gamma}s \\
&\leq\frac{16}{\sigma^2}\bar{\beta}^2s^2+1+\bar{\gamma}s.
\end{split}
\end{equation*}
Taking the maximum over the gradient bounds yields
\begin{equation*}
L_\mathcal{A}\leq\sqrt{2s}\max\left(\frac{8}{\sigma^2}\bar{\beta}s,2+\frac{16}{\sigma^2}\bar{\beta}^2s^2+\bar{\gamma}s\right)=L,
\end{equation*}
and, since $L$ does not depend on $\mathcal{A}$, it holds
\begin{equation*}
\max_{\mathcal{A}\in\mathcal{S}_s}\;L_\mathcal{A}\leq L.
\end{equation*}
\end{proof}

\subsection{Proof of Main Result}

Since $(\hat{\boldsymbol{\beta}},\hat{\boldsymbol{\gamma}})$ minimizes the negative log-likelihood, we necessarily have
\begin{equation*}
l(\hat{\boldsymbol{\beta}},\hat{\boldsymbol{\gamma}})-l(\boldsymbol{\beta}^0,\boldsymbol{\gamma}^0)\leq0.
\end{equation*}
Subtracting and adding $\operatorname{E}[l(\hat{\boldsymbol{\beta}},\hat{\boldsymbol{\gamma}})]$ and $\operatorname{E}[l(\boldsymbol{\beta}^0,\boldsymbol{\gamma}^0)]$ on the left-hand side gives
\begin{equation*}
l(\hat{\boldsymbol{\beta}},\hat{\boldsymbol{\gamma}})-\operatorname{E}[l(\hat{\boldsymbol{\beta}},\hat{\boldsymbol{\gamma}})]+\operatorname{E}[l(\hat{\boldsymbol{\beta}},\hat{\boldsymbol{\gamma}})]-\operatorname{E}[l(\boldsymbol{\beta}^0,\boldsymbol{\gamma}^0)]+\operatorname{E}[l(\boldsymbol{\beta}^0,\boldsymbol{\gamma}^0)]-l(\boldsymbol{\beta}^0,\boldsymbol{\gamma}^0)\leq0.
\end{equation*}
Now, isolating $\operatorname{E}[l(\hat{\boldsymbol{\beta}},\hat{\boldsymbol{\gamma}})]-\operatorname{E}[l(\boldsymbol{\beta}^0,\boldsymbol{\gamma}^0)]$ on the left-hand side gives the upper bound
\begin{equation}
\label{eq:bound1}
\begin{split}
\operatorname{E}[l(\hat{\boldsymbol{\beta}},\hat{\boldsymbol{\gamma}})]-\operatorname{E}[l(\boldsymbol{\beta}^0,\boldsymbol{\gamma}^0)]&\leq\operatorname{E}[l(\hat{\boldsymbol{\beta}},\hat{\boldsymbol{\gamma}})]-l(\hat{\boldsymbol{\beta}},\hat{\boldsymbol{\gamma}})+l(\boldsymbol{\beta}^0,\boldsymbol{\gamma}^0)-\operatorname{E}[l(\boldsymbol{\beta}^0,\boldsymbol{\gamma}^0)] \\
&\leq\left|\operatorname{E}[l(\hat{\boldsymbol{\beta}},\hat{\boldsymbol{\gamma}})]-l(\hat{\boldsymbol{\beta}},\hat{\boldsymbol{\gamma}})\right|+\left|l(\boldsymbol{\beta}^0,\boldsymbol{\gamma}^0)-\operatorname{E}[l(\boldsymbol{\beta}^0,\boldsymbol{\gamma}^0)]\right| \\
&\leq2\sup_{(\boldsymbol{\beta},\boldsymbol{\gamma})\in\mathcal{C}(s)}\;\left|l(\boldsymbol{\beta},\boldsymbol{\gamma})-\operatorname{E}[l(\boldsymbol{\beta},\boldsymbol{\gamma})]\right|.
\end{split}
\end{equation}
We now seek a high-probability upper bound for the supremum on the right-hand side. Towards this end, consider the difference inside the supremum, given by
\begin{equation*}
\begin{split}
&l(\boldsymbol{\beta},\boldsymbol{\gamma})-\operatorname{E}[l(\boldsymbol{\beta},\boldsymbol{\gamma})]=\sum_{i=1}^m\log\det\{\sigma^2\mathbf{V}_i(\boldsymbol{\gamma})\}+\frac{1}{\sigma^2}\sum_{i=1}^m(\mathbf{y}_i-\mathbf{X}_i\boldsymbol{\beta})^\top\mathbf{V}_i^{-1}(\boldsymbol{\gamma})(\mathbf{y}_i-\mathbf{X}_i\boldsymbol{\beta}) \\
&\hspace{1.2in}-\operatorname{E}\left[\sum_{i=1}^m\log\det\{\sigma^2\mathbf{V}_i(\boldsymbol{\gamma})\}+\frac{1}{\sigma^2}\sum_{i=1}^m(\mathbf{y}_i-\mathbf{X}_i\boldsymbol{\beta})^\top\mathbf{V}_i^{-1}(\boldsymbol{\gamma})(\mathbf{y}_i-\mathbf{X}_i\boldsymbol{\beta})\right].
\end{split}
\end{equation*}
Since the log determinant terms do not depend on $\boldsymbol{\nu}_i$, which is the only source of randomness, we can simplify the above expression to
\begin{equation*}
\begin{split}
l(\boldsymbol{\beta},\boldsymbol{\gamma})-\operatorname{E}[l(\boldsymbol{\beta},\boldsymbol{\gamma})]&=\frac{1}{\sigma^2}\sum_{i=1}^m(\mathbf{y}_i-\mathbf{X}_i\boldsymbol{\beta})^\top\mathbf{V}_i^{-1}(\boldsymbol{\gamma})(\mathbf{y}_i-\mathbf{X}_i\boldsymbol{\beta}) \\
&\hspace{1in}-\operatorname{E}\left[\frac{1}{\sigma^2}\sum_{i=1}^m(\mathbf{y}_i-\mathbf{X}_i\boldsymbol{\beta})^\top\mathbf{V}_i^{-1}(\boldsymbol{\gamma})(\mathbf{y}_i-\mathbf{X}_i\boldsymbol{\beta})\right].
\end{split}
\end{equation*}
Expanding the quadratic terms and evaluating the expectations yields
\begin{equation*}
\begin{split}
l(\boldsymbol{\beta},\boldsymbol{\gamma})-\operatorname{E}[l(\boldsymbol{\beta},\boldsymbol{\gamma})]&=\frac{1}{\sigma^2}\sum_{i=1}^m\left[\boldsymbol{\nu}_i^\top\mathbf{V}_i^{-1}(\boldsymbol{\gamma})\boldsymbol{\nu}_i-\sigma^2\operatorname{tr}\left\{\mathbf{V}_i^{-1}(\boldsymbol{\gamma})\mathbf{V}_i(\boldsymbol{\gamma}^0)\right\}\right] \\
&\hspace{1.5in}+\frac{1}{\sigma^2}\sum_{i=1}^m2(\boldsymbol{\beta}^0-\boldsymbol{\beta})^\top\mathbf{X}_i^\top\mathbf{V}_i^{-1}(\boldsymbol{\gamma})\boldsymbol{\nu}_i.
\end{split}
\end{equation*}
For the first term on the right-hand side, Lemma~\ref{lemma:prob1} gives the probabilistic upper bound
\begin{equation*}
\begin{split}
&\Pr\left[\left|\frac{1}{\sigma^2}\sum_{i=1}^m\left[\boldsymbol{\nu}_i^\top\mathbf{V}_i^{-1}(\boldsymbol{\gamma})\boldsymbol{\nu}_i-\sigma^2\operatorname{tr}\left\{\mathbf{V}_i^{-1}(\boldsymbol{\gamma})\mathbf{V}_i(\boldsymbol{\gamma}^0)\right\}\right]\right|\geq t\right] \\
&\hspace{2.5in}\leq2\exp\left\{-c\min\left(\frac{t^2}{\ndot(1+\bar{\gamma}s)^2},\frac{t}{1+\bar{\gamma}s}\right)\right\},
\end{split}
\end{equation*}
where $c>0$ is a universal constant and $t>0$. Similarly, for the second term, Lemma~\ref{lemma:prob2} gives the probabilistic upper bound
\begin{equation*}
\Pr\left[\left|\frac{1}{\sigma^2}\sum_{i=1}^m2(\boldsymbol{\beta}^0-\boldsymbol{\beta})^\top\mathbf{X}_i^\top\mathbf{V}_i^{-1}(\boldsymbol{\gamma})\boldsymbol{\nu}_i\right|\geq t\right]\leq2\exp\left(-\frac{t^2}{128\sigma^{-2}(1+\bar{\gamma}s)s^2\bar{\beta}^2}\right).
\end{equation*}
These bounds hold for a single fixed $(\boldsymbol{\beta},\boldsymbol{\gamma})$. We now extend them to hold uniformly over $\mathcal{C}(s)$ via an $\epsilon$-net argument. Towards this end, define the set $\mathcal{B}(\mathcal{A})$ as
\begin{equation*}
\mathcal{B}(\mathcal{A}):=\left\{(\boldsymbol{\beta},\boldsymbol{\gamma})\in\mathbb{R}^p\times\mathbb{R}_+^p:\|\boldsymbol{\beta}\|_\infty\leq\bar{\beta},\|\boldsymbol{\gamma}\|_\infty\leq\bar{\gamma},\operatorname{Supp}(\boldsymbol{\gamma})\subseteq\operatorname{Supp}(\boldsymbol{\beta})\subseteq\mathcal{A}\right\}.
\end{equation*}
Then, letting $\mathcal{S}_s$ be the set of all possible active sets $\mathcal{A}$ of size at most $s$, we have
\begin{equation*}
\sup_{(\boldsymbol{\beta},\boldsymbol{\gamma})\in\mathcal{C}(s)}\;\left|l(\boldsymbol{\beta},\boldsymbol{\gamma})-\operatorname{E}\left[l(\boldsymbol{\beta},\boldsymbol{\gamma})\right]\right|=\max_{\mathcal{A}\in\mathcal{S}_s}\;\sup_{(\boldsymbol{\beta},\boldsymbol{\gamma})\in\mathcal{B}(\mathcal{A})}\;\left|l(\boldsymbol{\beta},\boldsymbol{\gamma})-\operatorname{E}\left[l(\boldsymbol{\beta},\boldsymbol{\gamma})\right]\right|.
\end{equation*}
Next, we require a useful inequality. For any $\boldsymbol{\beta}'\neq\boldsymbol{\beta}$ and any $\boldsymbol{\gamma}'\neq\boldsymbol{\gamma}$ in $\mathcal{B}(\mathcal{A})$ it holds
\begin{equation*}
\begin{split}
&\left|l(\boldsymbol{\beta},\boldsymbol{\gamma})-\operatorname{E}\left[l(\boldsymbol{\beta},\boldsymbol{\gamma})\right]\right| \\
&\hspace{0.5in}=\left|l(\boldsymbol{\beta},\boldsymbol{\gamma})-l(\boldsymbol{\beta}',\boldsymbol{\gamma}')+l(\boldsymbol{\beta}',\boldsymbol{\gamma}')-\operatorname{E}\left[l(\boldsymbol{\beta}',\boldsymbol{\gamma}')\right]+\operatorname{E}\left[l(\boldsymbol{\beta}',\boldsymbol{\gamma}')\right]-\operatorname{E}\left[l(\boldsymbol{\beta},\boldsymbol{\gamma})\right]\right| \\
&\hspace{0.5in}\leq\left|l(\boldsymbol{\beta}',\boldsymbol{\gamma}')-\operatorname{E}\left[l(\boldsymbol{\beta}',\boldsymbol{\gamma}')\right]\right|+\left|l(\boldsymbol{\beta},\boldsymbol{\gamma})-l(\boldsymbol{\beta}',\boldsymbol{\gamma}')\right|+\left|\operatorname{E}\left[l(\boldsymbol{\beta}',\boldsymbol{\gamma}')\right]-\operatorname{E}\left[l(\boldsymbol{\beta},\boldsymbol{\gamma})\right]\right| \\
&\hspace{0.5in}\leq\left|l(\boldsymbol{\beta}',\boldsymbol{\gamma}')-\operatorname{E}\left[l(\boldsymbol{\beta}',\boldsymbol{\gamma}')\right]\right|+(L_\mathcal{A}(\mathbf{y})+L_\mathcal{A})\sqrt{\|\boldsymbol{\beta}-\boldsymbol{\beta}'\|_2^2+\|\boldsymbol{\gamma}-\boldsymbol{\gamma}'\|_2^2},
\end{split}
\end{equation*}
where the last inequality follows from the first results of Lemmas~\ref{lemma:lipschitz2} and \ref{lemma:lipschitz3}, which give that $l(\boldsymbol{\beta},\boldsymbol{\gamma})$ and $\operatorname{E}[l(\boldsymbol{\beta},\boldsymbol{\gamma})]$ are Lipschitz on $\mathcal{B}(\mathcal{A})$ with constants $L_\mathcal{A}(\mathbf{y})$ (a random quantity) and $L_\mathcal{A}$ (a fixed quantity), respectively. Now, let $\mathcal{E}(\mathcal{A})$ be an $\epsilon$-net of $\mathcal{B}(\mathcal{A})$ such that for any $(\boldsymbol{\beta},\boldsymbol{\gamma})\in\mathcal{B}(\mathcal{A})$ there exists a $(\boldsymbol{\beta}',\boldsymbol{\gamma}')\in\mathcal{E}(\mathcal{A})$ with $\sqrt{\|\boldsymbol{\beta}-\boldsymbol{\beta}'\|_2^2+\|\boldsymbol{\gamma}-\boldsymbol{\gamma}'\|_2^2}\leq\epsilon$. We then have
\begin{equation*}
(L_\mathcal{A}(\mathbf{y})+L_\mathcal{A})\sqrt{\|\boldsymbol{\beta}-\boldsymbol{\beta}'\|_2^2+\|\boldsymbol{\gamma}-\boldsymbol{\gamma}'\|_2^2}\leq(L_\mathcal{A}(\mathbf{y})+L_\mathcal{A})\epsilon\leq2\max(L_\mathcal{A}(\mathbf{y}),L_\mathcal{A})\epsilon,
\end{equation*}
from which it follows
\begin{equation*}
\begin{split}
&\max_{\mathcal{A}\in\mathcal{S}_s}\;\sup_{(\boldsymbol{\beta},\boldsymbol{\gamma})\in\mathcal{B}(\mathcal{A})}\;\left|l(\boldsymbol{\beta},\boldsymbol{\gamma})-\operatorname{E}\left[l(\boldsymbol{\beta},\boldsymbol{\gamma})\right]\right| \\
&\hspace{1in}\leq\max_{\mathcal{A}\in\mathcal{S}_s}\;\max_{(\boldsymbol{\beta}',\boldsymbol{\gamma}')\in\mathcal{E}(\mathcal{A})}\;\left(\left|l(\boldsymbol{\beta}',\boldsymbol{\gamma}')-\operatorname{E}\left[l(\boldsymbol{\beta}',\boldsymbol{\gamma}')\right]\right|+2\max(L_\mathcal{A}(\mathbf{y}),L_\mathcal{A})\epsilon\right).
\end{split}
\end{equation*}
Invoking the second result of Lemma~\ref{lemma:lipschitz2} gives the high-probability bound
\begin{equation}
\label{eq:lipschitzprob}
\Pr\left[\max_{\mathcal{A}\in\mathcal{S}_s}\;L_\mathcal{A}(\mathbf{y})\leq L(\delta_1)\right]\geq1-\delta_1
\end{equation}
for $\delta_1\in(0,1]$. Similarly, by the second result of Lemma~\ref{lemma:lipschitz3}, we have the deterministic bound
\begin{equation*}
\max_{\mathcal{A}\in\mathcal{S}_s}\;L_\mathcal{A}\leq L.
\end{equation*}
Both $L(\delta_1)$ and $L$ are deterministic quantities. Define $L^\star(\delta_1)$ to be
\begin{equation*}
L^\star(\delta_1):=\max(L(\delta_1),L).
\end{equation*}
Then, with probability $1-\delta_1$, it holds
\begin{equation*}
\begin{split}
&\max_{\mathcal{A}\in\mathcal{S}_s}\;\max_{(\boldsymbol{\beta}',\boldsymbol{\gamma}')\in\mathcal{E}(\mathcal{A})}\;\left(\left|l(\boldsymbol{\beta}',\boldsymbol{\gamma}')-\operatorname{E}\left[l(\boldsymbol{\beta}',\boldsymbol{\gamma}')\right]\right|+2\max(L_\mathcal{A}(\mathbf{y}),L_\mathcal{A})\epsilon\right) \\
&\hspace{1in}\leq\max_{\mathcal{A}\in\mathcal{S}_s}\;\max_{(\boldsymbol{\beta}',\boldsymbol{\gamma}')\in\mathcal{E}(\mathcal{A})}\;\left|l(\boldsymbol{\beta}',\boldsymbol{\gamma}')-\operatorname{E}\left[l(\boldsymbol{\beta}',\boldsymbol{\gamma}')\right]\right|+2L^\star(\delta_1)\epsilon.
\end{split}
\end{equation*}
In what follows, every inequality that contains $L^\star(\delta_1)$ is stated on the event in \eqref{eq:lipschitzprob}. We now have
\begin{equation*}
\begin{split}
&\Pr\left[\sup_{(\boldsymbol{\beta},\boldsymbol{\gamma})\in\mathcal{C}(s)}\;\left|l(\boldsymbol{\beta},\boldsymbol{\gamma})-\operatorname{E}\left[l(\boldsymbol{\beta},\boldsymbol{\gamma})\right]\right|\geq t\right] \\
&\hspace{1in}\leq\Pr\left[\max_{\mathcal{A}\in\mathcal{S}_s}\;\max_{(\boldsymbol{\beta}',\boldsymbol{\gamma}')\in\mathcal{E}(\mathcal{A})}\;\left|l(\boldsymbol{\beta}',\boldsymbol{\gamma}')-\operatorname{E}\left[l(\boldsymbol{\beta}',\boldsymbol{\gamma}')\right]\right|+2L^\star(\delta_1)\epsilon\geq t\right].
\end{split}
\end{equation*}
Setting $\epsilon=t/[4L^\star(\delta_1)]$ and applying a union bound on the right-hand side gives
\begin{equation*}
\begin{split}
&\Pr\left[\max_{\mathcal{A}\in\mathcal{S}_s}\;\max_{(\boldsymbol{\beta}',\boldsymbol{\gamma}')\in\mathcal{E}(\mathcal{A})}\;\left|l(\boldsymbol{\beta}',\boldsymbol{\gamma}')-\operatorname{E}\left[l(\boldsymbol{\beta}',\boldsymbol{\gamma}')\right]\right|\geq\frac{1}{2}t\right] \\
&\hspace{2in}\leq\sum_{\mathcal{A}\in\mathcal{S}_s}\sum_{(\boldsymbol{\beta}',\boldsymbol{\gamma}')\in\mathcal{E}(\mathcal{A})}\Pr\left[\left|l(\boldsymbol{\beta}',\boldsymbol{\gamma}')-\operatorname{E}\left[l(\boldsymbol{\beta}',\boldsymbol{\gamma}')\right]\right|\geq\frac{1}{2}t\right].
\end{split}
\end{equation*}
A standard binomial coefficient bound \citep[see, e.g.,][Section 1.2.6, Exercise 67]{Knuth1997} gives
\begin{equation*}
|\mathcal{S}_s|=\binom{p}{s}\leq\left(\frac{ep}{s}\right)^s.
\end{equation*}
Next, since $\|\boldsymbol{\beta}'\|_2\leq\sqrt{s}\bar{\beta}$ and $\|\boldsymbol{\gamma}'\|_2\leq\sqrt{s}\bar{\gamma}$, a standard $\ell_2$ norm $\epsilon$-net cardinality bound \citep[see, e.g.,][Corollary 4.2.13]{Vershynin2018} gives
\begin{equation*}
|\mathcal{E}(\mathcal{A})|\leq\left(1+\frac{2\sqrt{s}\bar{\beta}}{\epsilon}\right)^s\left(1+\frac{2\sqrt{s}\bar{\gamma}}{\epsilon}\right)^s\leq\exp\left(2s\sqrt{s}\frac{\bar{\beta}+\bar{\gamma}}{\epsilon}\right)=\exp\left(8L^\star(\delta_1)s\sqrt{s}\frac{\bar{\beta}+\bar{\gamma}}{t}\right).
\end{equation*}
It now follows that
\begin{equation}
\label{eq:probbound}
\begin{split}
&\sum_{\mathcal{A}\in\mathcal{S}_s}\sum_{(\boldsymbol{\beta}',\boldsymbol{\gamma}')\in\mathcal{E}(\mathcal{A})}\Pr\left[\left|l(\boldsymbol{\beta}',\boldsymbol{\gamma}')-\operatorname{E}\left[l(\boldsymbol{\beta}',\boldsymbol{\gamma}')\right]\right|\geq\frac{1}{2}t\right]\leq\left(\frac{ep}{s}\right)^s\exp\left(8L^\star(\delta_1)s\sqrt{s}\frac{\bar{\beta}+\bar{\gamma}}{t}\right) \\
&\hspace{0.1in}\times2\left[\exp\left\{-c\min\left(\frac{t^2}{4\ndot(1+\bar{\gamma}s)^2},\frac{t}{2(1+\bar{\gamma}s)}\right)\right\}+\exp\left(-\frac{t^2}{512\sigma^{-2}(1+\bar{\gamma}s)s^2\bar{\beta}^2}\right)\right].
\end{split}
\end{equation}
We wish to bound the right-hand side by a user chosen $\delta_2\in(0,1]$. To this end, we solve for a $t$ that simultaneously satisfies
\begin{equation}
\label{eq:t1}
2\left(\frac{ep}{s}\right)^s\exp\left(8L^\star(\delta_1)s\sqrt{s}\frac{\bar{\beta}+\bar{\gamma}}{t}\right)\exp\left\{-c\min\left(\frac{t^2}{4\ndot(1+\bar{\gamma}s)^2},\frac{t}{2(1+\bar{\gamma}s)}\right)\right\}\leq\frac{\delta_2}{2}
\end{equation}
and
\begin{equation}
\label{eq:t2}
2\left(\frac{ep}{s}\right)^s\exp\left(8L^\star(\delta_1)s\sqrt{s}\frac{\bar{\beta}+\bar{\gamma}}{t}\right)\exp\left(-\frac{t^2}{512\sigma^{-2}(1+\bar{\gamma}s)s^2\bar{\beta}^2}\right)\leq\frac{\delta_2}{2}.
\end{equation}
Some routine algebra shows that \eqref{eq:t1} is satisfied whenever
\begin{equation*}
t\geq t_1=(1+\bar{\gamma}s)\max\left(\sqrt{\frac{4\ndot}{c}\Pi},\frac{2}{c}\Pi\right),
\end{equation*}
where
\begin{equation*}
\Pi=s\log\left(\frac{ep}{s}\right)+8L^\star(\delta_1)s\sqrt{s}(\bar{\beta}+\bar{\gamma})+\log\left(\frac{4}{\delta_2}\right).
\end{equation*}
The above holds provided $t\geq1$, which is true for our final choice of $t$. Similarly, some standard derivations give that \eqref{eq:t2} is satisfied if
\begin{equation*}
\begin{split}
t\geq t_2&=\sqrt{512\sigma^{-2}(1+\bar{\gamma}s)s^2\bar{\beta}^2\Pi} \\
&=\frac{s\bar{\beta}}{\sigma}\sqrt{512(1+\bar{\gamma}s)\Pi}.
\end{split}
\end{equation*}
Hence, \eqref{eq:probbound} is guaranteed less than $\delta_2$ provided
\begin{equation*}
t\geq t_0=\max(t_1,t_2).
\end{equation*}
Now, setting $\delta_1=\delta/2$ and $\delta_2=\delta/2$, it follows from the previous inequalities that
\begin{equation}
\label{eq:bound2}
2\sup_{(\boldsymbol{\beta},\boldsymbol{\gamma})\in\mathcal{C}(s)}\;\left|l(\boldsymbol{\beta},\boldsymbol{\gamma})-\operatorname{E}[l(\boldsymbol{\beta},\boldsymbol{\gamma})]\right|\leq 2t_0
\end{equation}
with probability at least $1-\delta$. Finally, by the definition of KL divergence, we have
\begin{equation*}
\begin{split}
2\operatorname{KL}\left[\mathrm{N}\left(\mathbf{X}\boldsymbol{\beta}^0,\sigma^2\mathbf{V}(\boldsymbol{\gamma}^0)\right)\parallel\mathrm{N}\left(\mathbf{X}\hat{\boldsymbol{\beta}},\sigma^2\mathbf{V}(\hat{\boldsymbol{\gamma}})\right)\right]&=\operatorname{E}[-l(\boldsymbol{\beta}^0,\boldsymbol{\gamma}^0)]-\operatorname{E}[-l(\hat{\boldsymbol{\beta}},\hat{\boldsymbol{\gamma}})] \\
&=\operatorname{E}[l(\hat{\boldsymbol{\beta}},\hat{\boldsymbol{\gamma}})]-\operatorname{E}[l(\boldsymbol{\beta}^0,\boldsymbol{\gamma}^0)],
\end{split}
\end{equation*}
from which we can combine \eqref{eq:bound1} and \eqref{eq:bound2} to complete the proof.

\section{Implementation Details}
\label{app:implementation}

The configuration of each method is as follows.
\begin{itemize}
\item \texttt{glmmsel} is configured with a sequence of 100 values of $\lambda$ chosen using Proposition~\ref{prop:lambda}, where the first $\lambda$ sets all coefficients to zero. For each value of $\lambda$, we use a sequence of 10 values of $\alpha$ equispaced between $0.1$ and $1$.
\item \texttt{pysr3} is configured with the $\ell_0$ regularizer and a grid of $11\times 11$ values of $(k_1,k_2)$ equispaced between $0$ and $10$, where $k_1$ and $k_2$ are the maximum allowable number of nonzeros in the fixed effects and random effects, respectively.
\item \texttt{L0Learn} is configured with the $\ell_0$ regularizer and a sequence of 100 values of $\lambda$ chosen using the default method of the package.
\item \texttt{glmnet} is configured with the relaxed $\ell_1$ regularizer and a sequence of 100 values of $\lambda$ chosen using the default method of the package. The relaxation parameter $\gamma=0$.
\item \texttt{ncvreg} is configured with the minimax concave penalty and a sequence of 100 values of $\lambda$ chosen using the default method of the package. The concavity parameter $\gamma=1+10^{-4}$.
\item \texttt{glmmPen} is configured with the minimax concave penalty and a grid of $10\times 10$ values of $(\lambda_1,\lambda_2)$ chosen using the default method of the package, where $\lambda_1$ and $\lambda_2$ are the penalty parameters for the fixed effects and random effects, respectively.\footnote{\texttt{glmmPen} uses only 10 values each of $\lambda_1$ and $\lambda_2$ because its run time is significantly longer than the other methods. However, it only runs in the $p=10$ setting where \texttt{glmmsel} also uses at most 10 values of $\lambda$. Indeed, by Proposition~\ref{prop:lambda}, the number of active predictors changes between consecutive values of $\lambda$, and hence there can be at most 10 values when $p=10$.} The concavity parameter $\gamma=1+10^{-4}$.
\item \texttt{rpql} is configured with the minimax concave penalty and a grid of $30\times 30$ values of $(\lambda_1,\lambda_2)$ chosen using the default method of \texttt{glmmPen}, where $\lambda_1$ and $\lambda_2$ are the penalty parameters for the fixed effects and random effects, respectively. The package uses \texttt{glmmPen}'s grid construction method as it does not provide a method for computing the $\lambda$ sequence. The concavity parameter $\gamma=1+10^{-4}$.
\end{itemize}
The specific choices of $\gamma=0$ in \texttt{glmnet} and $\gamma=1+10^{-4}$ in \texttt{ncvreg}, \texttt{glmmPen}, and \texttt{rpql} are made to eliminate continuous shrinkage effects as far as possible. These parameter choices bring the behavior of these methods as close as possible to an $\ell_0$ regularizer, ensuring that the comparison across methods focuses on sparsity rather than bias.

\section{Small-Scale Experiments}
\label{app:small}

To accommodate a broader set of baselines, namely \texttt{glmmPen} and \texttt{rpql}, we run several small-scale experiments with just $p=10$ predictors. Figures~\ref{fig:gaussian-10-5} and \ref{fig:bernoulli-10-5} report the results for Gaussian and Bernoulli responses, respectively.
\begin{figure}[t]
\centering
\input{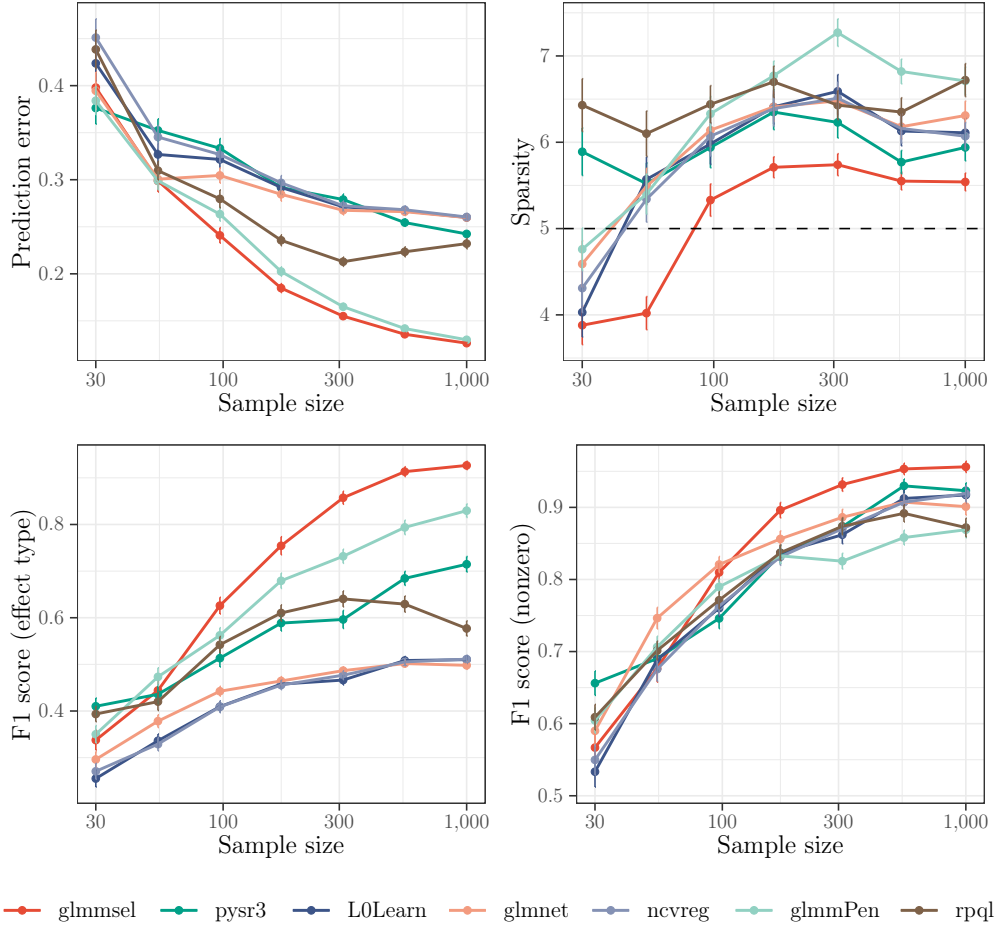}
\caption{Comparisons on synthetic data with Gaussian response. The number of predictors $p=10$ with 5 nonzero fixed effects and 3 nonzero random effects. The correlation level $\rho=0.5$. The averages (points) and standard errors (error bars) are measured over 100 datasets. In the upper right panel, the dashed horizontal line indicates the true number of predictors with nonzero effects.}
\label{fig:gaussian-10-5}
\end{figure}
\begin{figure}[t]
\centering
\input{Figures/bernoulli-10-5.tex}
\caption{Comparisons on synthetic data with Bernoulli response. The number of predictors $p=10$ with 5 nonzero fixed effects and 3 nonzero random effects. The correlation level $\rho=0.5$. The averages (points) and standard errors (error bars) are measured over 100 datasets. In the upper right panel, the dashed horizontal line indicates the true number of predictors with nonzero effects.}
\label{fig:bernoulli-10-5}
\end{figure}
In this low-dimensional setting, \texttt{glmmPen} emerges as the next best performer behind \texttt{glmmsel}, eventually delivering predictions on par with \texttt{glmmsel} when the sample size reaches $\ndot=1,000$. Nonetheless, its ability to accurately distinguish zeros from nonzeros and fixed effects from random effects is typically weaker. The underperformance of \texttt{rpql} is likely due to it modeling the random effect covariance matrix as unstructured, whereas the other methods use a more efficient diagonal representation.

\section{Denser Experiments}
\label{app:denser}

To assess performance beyond the highly sparse regime, we consider denser settings with 15 nonzero fixed effects and 5 nonzero random effects. Figures~\ref{fig:gaussian-1000-15} and \ref{fig:gaussian-10000-15} report the results for $p=1,000$ and $p=10,000$, respectively.
\begin{figure}[t]
\centering
\input{Figures/gaussian-1000-15.tex}
\caption{Comparisons on synthetic data with Gaussian response. The number of predictors $p=1,000$ with 15 nonzero fixed effects and 5 nonzero random effects. The correlation level $\rho=0.5$. The averages (points) and standard errors (error bars) are measured over 100 datasets. In the upper right panel, the dashed horizontal line indicates the true number of predictors with nonzero effects.}
\label{fig:gaussian-1000-15}
\end{figure}
\begin{figure}[t]
\centering
\input{Figures/gaussian-10000-15.tex}
\caption{Comparisons on synthetic data with Gaussian response. The number of predictors $p=10,000$ with 15 nonzero fixed effects and 5 nonzero random effects. The correlation level $\rho=0.5$. The averages (points) and standard errors (error bars) are measured over 100 datasets. In the upper right panel, the dashed horizontal line indicates the true number of predictors with nonzero effects.}
\label{fig:gaussian-10000-15}
\end{figure}
As expected, these settings are more challenging for all methods, especially at smaller sample sizes. Even so, the main conclusions are unchanged: \texttt{glmmsel} continues to perform strongly in prediction, achieves competitive recovery of effect type, particularly at larger sample sizes, and selects models with sparsity close to the truth. Overall, the results suggest that \texttt{glmmsel} remains stable beyond the highly sparse regime, including in higher-dimensional settings.

\section{Sensitivity to \texorpdfstring{$\gamma$}{gamma} Tuning}
\label{app:sensitivity}

To examine the sensitivity of the competing continuous shrinkage methods to their additional tuning parameter $\gamma$, we run additional experiments for \texttt{glmnet} and \texttt{ncvreg} in which $\gamma$ is selected by validation over a 10-point grid. Figure~\ref{fig:gaussian-gamma-tuning-1000-5} reports the results.
\begin{figure}[t]
\centering
\input{Figures/gaussian-gamma-tuning-1000-5.tex}
\caption{Comparisons on synthetic data with Gaussian response when the $\gamma$ parameter is tuned for \texttt{glmnet} and \texttt{ncvreg}. The number of predictors $p=1,000$ with 5 nonzero fixed effects and 3 nonzero random effects. The correlation level $\rho=0.5$. The averages (points) and standard errors (error bars) are measured over 100 datasets. In the upper right panel, the dashed horizontal line indicates the true number of predictors with nonzero effects.}
\label{fig:gaussian-gamma-tuning-1000-5}
\end{figure}
Relative to the main experiments, tuning $\gamma$ improves prediction for \texttt{glmnet} and \texttt{ncvreg}, but substantially worsens selection performance. The fitted models become much denser, and both F1 measures decline. Figure~\ref{fig:gaussian-gamma-tuning-1000-5} also shows that the number of selected predictors tends to increase with sample size. This increase in selected predictors appears to reflect misspecification when random effects are present in the data-generating process, since methods without random effects can improve their fit by selecting more fixed effects to absorb heterogeneity. This pattern is not observed in unreported experiments under a classical data-generating process with fixed effects only.

\section{Mashable Dataset}
\label{app:mashable}

The \texttt{mashable} dataset from \citet{Fernandes2015} contains data on the popularity of news articles posted to the online news platform Mashable. In addition to a binary response variable indicating whether an article is popular (classified as achieving more than the median number of social media shares), the dataset comprises $p=53$ predictor variables that characterize the article (e.g., number of total words, positive words, and images). The task is to identify the article characteristics that are reliable determinants of popularity. Towards that end, the dataset includes observations over $m=27$ days in December 2014, with $n_i\in\{20,\dots,100\}$ articles posted each day, giving a total of $\ndot=1,700$ observations.

We randomly split the dataset into training, validation, and testing sets in 0.70-0.15-0.15 proportions. Table~\ref{tab:mashable} reports the results from 100 of these splits.
\begin{table}[ht]
\centering
\small

\begin{tabular}{lrrrr}
\toprule
 &  & \multicolumn{3}{c}{Sparsity} \\ 
\cmidrule(lr){3-5}
 & Prediction error & Total & Fixed & Random \\ 
\midrule
\texttt{glmmsel} & 0.908 (0.003) & 14.4 (0.9) & 13.4 (0.8) & 1.0 (0.2) \\ 
\texttt{L0Learn} & 0.916 (0.004) & 19.0 (1.5) & 19.0 (1.5) & 0.0 (0.0) \\ 
\texttt{glmnet} & 0.918 (0.004) & 21.9 (1.4) & 21.9 (1.4) & 0.0 (0.0) \\ 
\texttt{ncvreg} & 0.916 (0.004) & 17.3 (1.4) & 17.3 (1.4) & 0.0 (0.0) \\ 
\bottomrule
\end{tabular}

\caption{Comparisons on the \texttt{mashable} dataset. The averages and standard errors (parentheses) are measured over 100 splits of the data. Prediction error is the cross entropy loss of the fitted linear predictor relative to the fixed intercept-only model.}
\label{tab:mashable}
\end{table}
\texttt{glmmsel} again clearly stands out by producing GLMMs that attain the lowest prediction loss and contain the fewest predictors (approximately 14 article characteristics on average). In contrast, the GLMs from \texttt{L0Learn} and \texttt{ncvreg}, which rely solely on fixed effects and thus cannot capture heterogeneity across days, select roughly 20–30\% more article characteristics, resulting in substantially worse prediction performance. \texttt{glmnet} is even less competitive, averaging 50\% more predictors and achieving notably higher prediction loss. Consistent with the results reported on the \texttt{riboflavin} dataset, \texttt{glmmsel} again emerges as the most performant method and the only method capable of successfully fitting sparse GLMMs on the \texttt{mashable} dataset.

\bibliographystyle{agsm}
\bibliography{library}

\end{document}